\documentclass[letterpaper,11pt]{article}
\usepackage[in]{fullpage}

\usepackage{bm}
\usepackage{comment}
\usepackage{microtype}
\usepackage{nicefrac}
\usepackage{enumitem} 
\usepackage{amsmath,amssymb,amsfonts}
\usepackage{accmacros}
\usepackage{cases}
\usepackage{latexsym,caption,listings}
\usepackage{natbib}  
\usepackage{graphicx}
\usepackage{appendix}

\hypersetup{
  breaklinks  = true,
  colorlinks   = true, 
  urlcolor     = black, 
  linkcolor    = red, 
  citecolor   = blue 
}

\long\def\acks#1{\vskip 0.3in\noindent{\large\bf Acknowledgments}
\noindent #1}

\renewenvironment{abstract}
{\centerline{\large\bf Abstract}\vspace{0.7ex}%
  \bgroup\leftskip 20pt\rightskip 20pt\small\noindent\ignorespaces}%
{\par\egroup\vskip 0.25ex}

\newenvironment{keywords}
{\bgroup\leftskip 20pt\rightskip 20pt \small\noindent{\bf Keywords:} }%
{\par\egroup\vskip 0.25ex}

\newcommand{\BlackBox}{\rule{1.5ex}{1.5ex}}  
\newenvironment{proof}{\par\noindent{\bf Proof\ }}{\hfill\BlackBox\\[2mm]}
 
\newtheorem{theorem}{Theorem}
\newtheorem{lemma}[theorem]{Lemma} 
 
\newtheorem{remark}[theorem]{Remark}
\newtheorem{corollary}[theorem]{Corollary}
\newtheorem{definition}[theorem]{Definition}

\title{Accelerated SGD for Non-Strongly-Convex Least Squares }

\author{
  Aditya Varre\\
  TML Lab, EPFL\\
  \href{mailto:aditya.varre@epfl.ch}{aditya.varre@epfl.ch}
  \and
  Nicolas Flammarion \\
  TML Lab, EPFL\\
  \href{mailto:nicolas.flammarion@epfl.ch}{nicolas.flammarion@epfl.ch}
}

\begin{document}
\maketitle
\begin{abstract}%
We consider stochastic approximation for the least squares regression problem in the non-strongly convex setting. 
We present the first practical algorithm that achieves the optimal prediction error rates in terms of dependence on the noise of the problem, as $O(d/t)$ while accelerating the forgetting of the initial conditions to $O(d/t^2)$.
Our new algorithm is based on a simple modification of the accelerated gradient descent. 
We provide convergence results for both the averaged and the last iterate of the algorithm.
In order to describe the tightness of these new bounds, we present a matching lower bound in the noiseless setting  and thus show the optimality of our algorithm.
\end{abstract}






\section{Introduction}

When it comes to large scale machine learning, the stochastic gradient descent (SGD) of \cite{robbinmunro} is the practitioners' algorithm of choice.
Both its practical efficiency and its theoretical performance make it the driving force of modern machine learning~\citep{BotBou08}.
On a practical level, its updates are cheap to compute thanks to stochastic gradients. 
On a theoretical level, it  achieves the optimal rate of convergence with  statistically-optimal asymptotic variance for convex problems. 

However, the recent successes of deep neural networks brought a new paradigm to the classical learning setting~\citep{ma2018power}.
In many applications, the variance of gradient noise is not the limiting factor in the optimization anymore; rather it is the distance separating the initialization of the algorithm and the problem solution. 
Unfortunately, the bias of the stochastic gradient descent, which characterizes how fast the initial conditions are “forgotten”, is suboptimal.
In this respect, fast gradient methods (including momentum~\citep{POLYAK19641} or accelerated methods~\citep{nesterov1983method}) are optimal, but have the drawback of being sensitive to noise~\citep{Asp08,devolder2014first}. 
 
This naturally raises the question of whether we can accelerate the bias convergence while still relying on computationally cheap  gradient estimates.
This question has been partially answered for the elementary problem of least squares regression in a seminal line of research \citep{dieuleveutFB17,jain2018accelerating}. 
Theoretically their methods enjoy the best of both worlds---they converge at the fast rate of accelerated  methods while being robust to noise in the gradient. 
However their investigations are still inconclusive. 
On the one hand, \citet{jain2018accelerating}  assume  the least squares problem to be strongly convex, an assumption which is rarely satisfied in practice but which enables to efficiently stabilise the algorithm. 
%
On the other hand, \citet{dieuleveutFB17} makes a simplifying assumption on the gradient oracle they consider and their results do not apply to the cheaply-computed stochastic gradient used in practice. 
Therefore, even for this simple quadratic problem which is one of the main primitive of machine learning, the question is still open.  
 
In this work, we propose a novel algorithm which accelerates the convergence of the bias term while maintaining the optimal variance for non-strongly convex least squares regression. 
Our algorithm only requires access to the stream of observations and is easily implementable. 
%
It rests on a simple modification of the Nesterov accelerated gradient descent. Following the linear coupling view of \citet{AllenOrecchia2017}, acceleration can be obtained by coupling gradient descent and another update with aggressive stepsize.
Consequently one simply has to scale down the stepsize in the aggressive update to make it robust to the gradient noise. 
With this modification, the average of the iterates converges at rate  $O(\frac{d\| \xt{0} - \xt{*}\|^2}{t^2} + \frac{\sigma^2d}{t})$ after $t$ iterations, where   $\xt{0},\xt{*}\in \R^d$ are the starting point and the problem solution, and $\sigma^2$ is the noise variance of the linear regression model. In practice, the last iterate is often favored. We show for this latter a convergence of $O(\frac{d\| \xt{0} - \xt{*}\|^2}{t^2} + \sigma^2)$ which is relevant in applications where $\sigma$ is small. We also investigate the extra dimensional factor compared to the truly accelerated rate. This slowdown comes from the step-size reduction and is shown to be inevitable. 

\paragraph{Contributions.} In this paper, we make the following contributions:
\begin{itemize}[topsep=0.5em]\setlength\itemsep{0.5em}
    \item In Section~\ref{sec:setup}, we propose a novel stochastic accelerated algorithm AcSGD which rests on a simple modification of the Nesterov accelerated algorithm: scaling down one of its step size makes it provably robust to noise in the gradient. 
    \item In Section~\ref{sec:convav}, we show that the weighted average of the iterates of AcSGD  converges at rate $O(\frac{d}{t^2} + \frac{\sigma^2d}{t})$, thus attaining the optimal rate for the variance and accelerating the bias term.
    \item In Section~\ref{sec:last-iterate}, we show that the final iterate of AcSGD achieves a convergence rate $O(\frac{d}{t^2} + \sigma^2)$. In particular for noiseless problems, the final iterate converges to the solution at the accelerated rate $O(\frac{d}{t^2})$.
    \item In Section~\ref{sec:lower-bound}, we show that the dimension dependency in the accelerated rate  is necessary for certain distributions and therefore the rates we obtain are optimal.
    \item The algorithm is simple to implement and practically efficient as we illustrate with simulations on synthetic examples in Section~\ref{section:exp}.
\end{itemize}

  \subsection{Related Work} 
   Our work lies at the intersection of two classical themes - noise stability of accelerated gradient methods and  stochastic approximation for least squares. 
    
\paragraph{Accelerated methods and their noise stability.}
Fast gradient methods  refer to first order algorithms which converge at a faster rate than the classical gradient descent---the most famous among them being the accelerated gradient descent of \citet{nesterov1983method}.
First initiated by \citet{nemirovskij1983problem},
these methods are inspired by algorithms dedicated to the optimization of quadratic functions, i.e.,  the Heavy ball algorithm~\citep{POLYAK19641} and the conjugate gradient~\citep{conjugate52}.
For smooth convex problems, these algorithms accelerate the  convergence rate of gradient descent from $O(1/t)$ to $O(1/t^2)$, a rate which is optimal among first-order techniques.



These algorithms are however sensitive to noise in the gradients as shown for Heavy ball~\citep{polyakbook}, conjugate gradient~\citep{greenbaum}, accelerated gradient descent~\citep{Asp08, devolder2014first} and momentum  gradient descent~\citep{yuan16}.
%
%
%
%
%
%
Positive results for accelerated gradient descent were nevertheless obtained when the gradients are perturbed with zero-mean finite variance random noise~\citep{lancomposite12,hu2009,xiao2010}.  Convergence rates 
$O(\frac{L\|\xt{0}-\xt{*}\|^2}{t^2}+\frac{\sigma \|\xt{0}-\xt{*}\| }{\sqrt{t}})$
were proved  for $L$-smooth convex functions with minimum $\xt{*}$,  starting point $\xt{0}$ and when the variance of the noisy gradient is bounded by $\sigma^2$. Accelerated rates for strongly convex problems were also derived~\citep{ghadimi2012optimal,ghadimi2013optimal}. For the stochastic Heavy ball, almost sure convergence has been proved~\citep{gadat2018stochastic,sebbouh2021almost} but without improvement over gradient descent. 

%
%
%
%
%

\paragraph{Stochastic Approximation for Least Squares.}
Stochastic approximation dates back to \citet{robbinmunro} and their seminal work on SGD which has then spurred a surge of research.
In the convex regime, a complete complexity theory has been derived, with matching upper and lower bounds on the convergence rates~\citep{NemJudLan08,bach2011non,nemirovskij1983problem,AgaBarRavWai12}.
For smooth problems, averaging techniques~\citep{Rup88,Pol90} which consist in replacing the iterates by their average, have had an important theoretical impact. 
Indeed,
\citet{polyak1992acceleration} observed that averaging the SGD iterates along the optimization path provably reduces the  impact of gradient noise and makes the estimation rates statistically optimal.
The least squares regression problem has been given particular attention~\citep{bach2013non,DieBac15,jain201parallelizing,flammarion2017stochastic,zou2021benign}. 
\citet{bach2013non} showed that averaged SGD achieves the non-asymptotic rate of $O(1/t)$ even in the non-strongly convex case.  
For this problem, the performance of the algorithms can be decomposed as the sum of a bias term, characteristic of the initial-condition forgetting, and a variance term, characteristic of the effect of the noise in the linear statistical model.  
While averaged SGD  obtains the statistically optimal variance term $O(\sigma^2 d/  t)$~\citep{tsybakov2003optimal}, its bias term  converges at a suboptimal rate $O(1/t)$. 
    
\paragraph{Accelerated Stochastic Methods for least squares.}
Acceleration and stochastic approximation have been reconciled in the setting of least-squares regression by \citet{flammarion2015averaging,dieuleveutFB17,jain2018accelerating}.
Assuming an additive bounded-variance noise oracle, \citet{dieuleveutFB17} designed an algorithm simultaneously achieving optimal prediction error rates, both in terms of forgetting of initial conditions and noise dependence. However this oracle requires the knowledge of the covariance of the features and their algorithm is therefore not applicable in practice. 
\citet{jain2018accelerating}, relaxed this latter condition and proposed an algorithm using the regular SGD oracle which obtains an accelerated linear rate for strongly convex objectives. However the strong convexity assumption is often too restrictive for machine learning problems where the variables are in large dimension and highly correlated. Thus the strong convexity constant is often insignificant and bounds derived using this assumption are vacuous. 
We finally note that in the offline setting when multiple passes over the data are possible, accelerated version of variance reduced algorithms have been developed~\citep{frostig15,allen2017katyusha}.
%
In the same setting, \citet{paquette2021dynamics} studied the convergence of stochastic momentum 
algorithm and  derived asymptotic accelerated rates with a dimension dependent scaling of learning rates similar to ours. 
The focus of the offline setting is however different and no generalization results are given.

%

    
\section{Setup: Stochastic Nesterov acceleration for Least squares}
\label{sec:setup}

We consider the classical problem of least squares regression in the finite dimensional Euclidean space $\mathbb{R}^{\dimn}$.  We observe a stream of samples $\left(a_n,b_n\right) \in \left( \mathbb{R}^{\dimn} , \mathbb{R} \right) $, for $n \geq 1$, independent and identically sampled from an unknown distribution $\rho$, such that $\E \|a_n\|^2$ and $\E[b_n^2]$ are finite.   The objective is to  minimize the population risk 
    \begin{align*}
    \mathcal{R}(x) = \frac{1}{2}\E_\rho \left(\langle x,a\rangle - b\right)^2,
   \ \ \ \text{ where } (a,b)\sim\rho.
   \end{align*}



\paragraph{Covariance.} We denote by $\cov \defeq  \Ex{a \otimes a}$, the covariance matrix which is also the Hessian of the function $\mathcal{R}$. Without loss of generality, we assume that $\cov$ is invertible (by reducing $\mathbb{R}^{\dimn}$ to a minimal subspace where all $(a_n)_{n \geq 1}$ lie almost surely). The function $\mathcal{R}$ admits then a unique global minimum, we denote by $\xt{*} $, i.e. $\xt{*} = \argmin_{x \in \mathbb{R}^d } \risk{x}$.
Even if this assumption implies that the eigenvalues of $\cov$ are strictly positive,  they can still be arbitrarily small. In addition, we do not assume any knowledge of lower bound on the smallest eigenvalue.  The smoothness constant of risk $\mathcal{R}$, say  $L$, is the largest eigenvalue of $\cov$.     
    
    
We make the following assumptions on the joint distribution of $(a_n,b_n)$ which are standard in the analysis of stochastic algorithms for the least squares problem. 
    \begin{asshort} [Fourth Moment] \label{asmp:fourth-moment}
        There exists a finite constant $R$ such that 
    \begin{align}
        \E{ \left[ \nor{a}^2  ~ a \otimes a \right] } \preccurlyeq	 R^2 \cov.
    \end{align}
    \end{asshort}
        \begin{asshort} [Noise Level] \label{asmp:noise-covariance}
        There exists a finite constant $\sigma$ such that 
    \begin{align}
        \Ex{ \left(b - \scal{\xt{*}}{a}\right)^2 a \otimes a } \preccurlyeq \sigma^2 \cov.
    \end{align}
    \end{asshort}
    \begin{asshort} [Statistical Condition Number] \label{asmp:stat-condition-number}
    There exists a finite constant $\ki$ such that 
    \begin{align}
    \E{ \left[ \nor{a}^2_{\cov^{-1}} ~ a \otimes a \right] } \preccurlyeq	 \ki \cov.
    \end{align}
    \end{asshort}
    
    \paragraph{Discussion of assumptions.} 
    Assumptions \ref{asmp:fourth-moment} and \ref{asmp:noise-covariance} on the fourth moment and the noise level are classical to the analysis of stochastic gradient methods in least squares setting  \citep{bach2013non,jain201parallelizing}.  Assumption~\ref{asmp:fourth-moment} holds if the features are bounded, i.e., $\nor{a}^2 \leq R^2$, $\rho_{a}$ almost surely. It also holds, more generally, for features with infinite support such as sub-Gaussian features.   Assumption~\ref{asmp:noise-covariance}  states that the covariance of the gradient at optimum $\xt{*}$ is bounded by $ \sigma^2 \cov $. In the case of homoscedastic/well-specified model i.e. $ b = \scal{\xt{*}}{a} + \epsilon $  where $\epsilon$ is independent of $a$, the above assumption holds with $ \sigma^2 = \Ex{\epsilon^2} $.
    
The statistical condition number defined in Assumption~\ref{asmp:stat-condition-number} is  specific to acceleration of SGD. It was introduced by \citet{jain2018accelerating} in the context of acceleration for strongly convex least squares. It was also used by \citet{even2021continuized} for the analysis of continuized Nesterov acceleration on non-strongly convex least squares. The statistical condition number is always larger than the dimension, i.e., $\ki \geq d$. For  sub-Gaussian distribution, $\ki$ is $O(d)$. However, for one-hot basis distribution, i.e., $a = e_i$ with probability $p_i$, it is equal to $\ki = p_{\min}^{-1}$ and thus can be arbitrarily large.


\paragraph{Nesterov Acceleration.}

We consider the following algorithm (AcSGD) which starts with the  initial values $\xt{0} \in \mathbb{R}^{\dimn} , \zt{0} = \xt{0}$  and update for $ t \geq 0 $
\begin{subequations} \label{alg:nesterov}
\eqals{ \yt{t+1} &= \xt{t} -  \beta \nabla_{t}\risk{ \xt{t} } \label{eq:yt}, \\ \zt{t+1} &= \zt{t} - \alpha (t+1) \nabla_{t}\risk{ \xt{t} }  \label{eq:nest_zt} ,\\ ( t+2 )\xt{t+1} &= (t+1)\yt{t+1} + \zt{t+1}, \label{eq:xt} }
\end{subequations}
with step sizes $\alpha,\beta>0$ and where $\nabla_{t}\risk{ \xt{t} }$ is an unbiased estimate of the gradient of $\nabla \mathcal R(\xt{t})$.

This algorithm is similar to the standard three-sequences formulation of the Nesterov accelerated gradient descent~\citep{Nesterov2005} but with two different learning rates $ \alpha$ and $\beta $ in the gradient steps \eqref{eq:yt}, and \eqref{eq:nest_zt}. As noted by \citep{flammarion2015averaging}, this formulation captures various algorithms. With exact gradients, for different $\alpha, \beta$  for e.g. when $\alpha = 0$, we recover averaged gradient descent while with $\beta = 0$ we recover a version of the Heavy ball algorithm.

We especially consider the weighted averages of the iterates defined after  $T$ iterations by 
\eqal{ \label{eq:estimator} \avgxt{T} \defeq  \frac{\sum_{t=0}^{T} (t+1) \xt{t} }{\sum_{t=0}^{T} (t+1) }. }
In contrast to the classical average considered by \citet{polyak1992acceleration}, \eqref{eq:estimator} uses  weighted average which gives more importance to the last iterates and is therefore related to tail-averaging.
%

    

\paragraph{Stochastic Oracles.}

Let $(a_t,b_t) \in \left( \mathbb{R}^d, \mathbb{R} \right)$ be the sample at iteration $t$, we consider the stochastic gradient of $\mathcal R$ at $\xt{t}$ 
\begin{align} \label{eq:SGD-oracle} 
    \nabla_{t}\risk{\xt{t}} =  a_t \left( \scal{a_t}{\xt{t}} - {b_t} \right). 
\end{align}
Note that this is a true stochastic gradient oracle unlike \citet{dieuleveutFB17}, where a simpler oracle which assumes the knowledge of the covariance $\cov$ is considered. As explained in App.~\ref{subsec:prelim}, this oracle combines an additive noise independent of the iterate $\xt{t}$ and a multiplicative noise which scales with $\xt{t}$. Dealing with the multiplicative part of the oracle is the main challenge of our analysis.

\section{Convergence of the Averaged Iterates} \label{sec:convav}

In this section, we present our main result on the decay rate of the excess error of our estimate. We extend the results of \cite{dieuleveutFB17} to the general stochastic gradient oracle in the following theorem. 

\begin{theorem}
\label{thm:main}
    Consider Algorithm~\ref{alg:nesterov} under  Assumptions~\ref{asmp:fourth-moment},~\ref{asmp:noise-covariance},~\ref{asmp:stat-condition-number} and step sizes satisfying $ ( \alpha + 2 \beta ) R^2 \leqslant 1$ and    $\alpha \leqslant  \frac{\beta}{2 \ki}$. In expectation, the excess risk of estimator $\avgxt{T}$ after $T$ iterations  is bounded as
    \begin{align*}
        \Ex{\risk{\avgxt{T}}} - \risk{\xt{*}} &\leqslant \min\left\lbrace \frac{12}{\alpha T^2}, \frac{48}{\beta T} \right\rbrace \nor{\xt{0} - \xt{*}}^2 +  \frac{ 72 ~ \sigma^2 d  }{T}.
    \end{align*} 
\end{theorem}
The constants in the bounds are partially artifacts of the proof technique. The proof can be found in App.~\ref{app:proof:theorem:main}. In order to give a clear picture of how the rate depends on the constants $R^2,\ki$, we give a corollary below for a specific choice of step-sizes.  
\begin{corollary}
\label{cor:main}
Under the same conditions as Theorem~\ref{thm:main} and with the step sizes $\beta = \frac{1}{3R^2}$, and $\alpha = \frac{1}{ 6 \ki R^2}$. In expectation, the excess error of estimator $\avgxt{T}$ after $T$ iterations  is bounded as
    \begin{align*}
        \Ex{\risk{\avgxt{T}}} - \risk{\xt{*}} &\leqslant \min\left\lbrace \frac{72 \ki R^2}{T^2}, \frac{144 R^2}{T} \right\rbrace \nor{\xt{0} - \xt{*}}^2 +  \frac{ 72 ~ \sigma^2 d  }{T}.
    \end{align*} 
\end{corollary}
\noindent We make the following comments on Theorem~\ref{alg:nesterov} and Corollary~\ref{cor:main}
 \paragraph{Optimality of the convergence rate.}
   The convergence rate is composed of two terms: (a) a bias term which describes how fast initial conditions are forgotten and corresponds to the noiseless problem ($\sigma=0$). (b) A variance term which indicates the effect of the noise in the statistical model, independently of the starting point. It corresponds to the problem where the initialization is the solution $\xt{*}$.
   
     The algorithm recovers the fast rate of $O\left( \nicefrac{1}{\left(\alpha T^2\right)} \right) $ of accelerated gradient descent for the bias term. This is the optimal convergence rate for minimizing  quadratic functions with a first-order method.
     However to make the algorithm robust to the stochastic-gradient noise, the learning rate $\alpha$ has to be scaled with regards to the statistical condition number $\ki$.  
     For $ T \leq \ki $, the bias of the algorithm decays as that of averaged SGD, i.e, the second component of the bias governs the rate. 
     However the acceleration comes in for $ T \geq \ki $ and we observe an accelerated rate of $O\left( \nicefrac{\ki}{T^2} \right)$ afterward. 
   This $\ki$-dependence is the consequence of using computationally cheap rank-one stochastic gradients $ a_i  \scal{a_i}{\xt{t}-\xt{*}}$ instead of the full-rank update $\cov (\xt{t}-\xt{*})  $ as in gradient descent.
     The tightness of the rate with respect to $\ki$ and consequently on the dimension $d$ is of particular importance and is discussed in Section~\ref{sec:lower-bound}.

     The algorithm also recovers the optimal rate  $O\left( \nicefrac{\sigma^2 d}{T} \right)$ for the variance term~\citep{tsybakov2003optimal}. Hence it retains the optimal rate of variance error while improving the rate of bias error over standard SGD.  
     
     
\paragraph{Stochastic Gradients and Error Accumulation.}
When true gradients are replaced by stochastic gradients, algorithms accumulate the noisy gradient induced errors as they progress.
In order to still converge, the algorithms need to be modified to adapt accordingly. 
In the case of linear regression, when comparing SGD with GD, the error accumulation due to the multiplicative noise is controlled by scaling the step size from $O(\nicefrac{1}{L})$ to $O(\nicefrac{1}{R^2})$. The error due to the additive noise is controlled by averaging the iterates.
%
%
In the case of accelerated gradient descent, the scaling of the step sizes becomes intuitive if we consider the \emph{linear coupling} interpretation of  \citet{AllenOrecchia2017}.
%
%
In this view of Algorithm~\ref{alg:nesterov}, a  gradient step (on $\yt{t}$) and an \textit{aggressive} gradient step (on $\zt{t}$) are elegantly coupled to achieve acceleration.
The aggressive step is more sensitive to noise since it is of scale $O(t)$. 
Therefore, the step-size $\alpha$ needs to be appropriately scaled down to control the error accumulation of the $\zt{t}$-gradient step. %
Strikingly, this scaling is proportional to the  statistical condition number and therefore to the dimension of the features.

\paragraph{Comparison with \citet{jain2018accelerating}.} Note that as both algorithms have the optimal rate for the variance, we only compare the rate for the bias error here. 
The accelerated stochastic algorithm for strongly convex objectives of \citet{jain2018accelerating} converges with linear rate $O\left(poly(\mu^{-1}) \cdot e^{-{\left({t}/{\sqrt{\mu^{-1} \ki}}\right)}}\right)$ where $\mu$ is the smallest eigenvalue of $\cov$.
We note that (a) this rate is vacuous for finite time horizon  smaller that $\sqrt{\ki/\mu}$ and (b) the algorithm requires the knowledge of the constant $\mu$ which is unknown in practice.
In comparison, our algorithm converges at rate $O(\ki/t^2)$ for any arbitrarily small $\mu$ and therefore is faster for reasonable finite horizon.
Hence, assuming  $\cov$ invertible does not make the problem strongly convex, emphasizing the relevance of the non-strongly convex  setting for least squares problems.
%
%
%
%
The algorithm of \citet{jain2018accelerating} can also be coupled with an appropriate regularization \citep{allen2016optimal} to be directly used on non-strongly convex functions.
The resulting algorithm achieves a target error $\epsilon$ in $O(\sqrt{\ki/\epsilon}\log{\epsilon^{-1}})$ iterations. In comparison, our algorithm requires $O(\sqrt{\ki/\epsilon})$ iterations.
Besides the additional logarithmic factor,  algorithms with the aforementioned regularization are not truly online, since the target accuracy has to be known and the total number of iterations set in advance.
In contrast, our algorithm shows that acceleration can be made robust to stochastic gradients without additional regularization or strong-convexity assumption.

\paragraph{Finite sum minimization of regularized ERM.} 
We investigate here the competitiveness of our method when compared to direct minimization of the regularized ERM objective\footnote{Generalization is not guaranteed without regularization~\citep{GyoKohKrzWal06}.}. 
The ERM problem can be efficiently minimized using variance reduced algorithms~\citep{riesvrg,schmidt2017} and in particular their accelerated variants~\citep{frostig15,allen2017katyusha}.
To achieve a target error of $\epsilon$, these methods required $O\left( \nicefrac{\sigma^2 d}{\epsilon} + \nicefrac{\nor{\xt{*}} \sqrt{L ~ \sigma^2 d}}{\epsilon} \right)$  basic vector computations. 
The number of vector computation of our algorithm is comparatively  $O\left( \nicefrac{\sigma^2 d}{\epsilon} + \nicefrac{\nor{\xt{*}} \sqrt{R^2 d} }{\sqrt{\epsilon}} \right)$, taking $\ki = O(d)$ for simplicity.
Therefore, our method needs fewer computations  for small target errors $ \epsilon \leq \frac{L \sigma^2}{R^2}$.
In addition, accelerated SVRG needs a $O(dn)$ memory, where $n$ is number of samples in ERM while our single pass method only uses a $O(d)$ space.



%

\paragraph{Mini-Batch Accelerated SGD.} 
We also consider the mini-batch stochastic gradient oracle which queries the gradient oracle several times and returns the average of these stochastic gradients given by the observations $( a_{t,i}, b_{t,i} )_{i \leq b}$  :
\begin{align}\label{eq:minibatch}
    \nabla_{t}\risk{\xt{t}} = \frac{1}{b} \sum_{i=1}^{b} a_{t,i} \left( \scal{a_{t,i}}{\xt{t}} - {b_{t,i}} \right). 
\end{align}
Mini-batching enables to reduce the variance of the gradient estimate and to parallelize the computations. 
When we implement Algorithm~\ref{alg:nesterov} with the mini-batch stochastic gradient oracle defined in \eqref{eq:minibatch}, Theorem~\ref{thm:main} becomes valid for learning rates satisfying
 $ ( \alpha + 2 \beta ) R^2 \leqslant b$, $\alpha \leq \frac{b \beta}{2 \ki} $ and
 $ \alpha,\beta \leq \frac{1}{L} $. 
 For batch size $b \lesssim \nicefrac{R^2}{L} $, the rate of convergence becomes $ O\left( \min\left\lbrace\frac{\ki}{b^2 t^2}, \frac{1}{b t} \right\rbrace + \frac{\sigma^2 d}{b t} \right) $.
 Even if it does not improve the overall sample complexity,  using mini-batch is interesting from a practical point of view:
the algorithm can be used with larger step size ($\alpha$ scales with $b^2$), which speeds up the accelerated phase. Indeed the algorithm is  accelerated only after $ \nicefrac{\ki}{b}$ iterations.
For larger batch sizes $b \geq \nicefrac{R^2}{L}$, $\beta$ cannot be scaled with $b$ due to  the condition $\beta \leq 1/L$. The learning rate $\alpha$ can nevertheless be scaled linearly with $b$, if $b \leq \ki$ .
Thus, increasing the batch size still leads to fast rate for Algorithm~\ref{alg:nesterov}, in accordance with the findings of \citet{cotter11} for accelerated gradient methods.
This behavior is in contrast to SGD---where the linear speedup is lost for batch size larger than a certain threshold~\citep{jain201parallelizing}. 
Finally, we note that when the batch size is $O(\ki)$, the performance of the algorithm matches the one of Nesterov accelerated gradient descent. 
This fact is consistent with the observation of \citet{hsu2012random} that the empirical covariance of $\ki$ samples is spectrally close to $\cov$. 


\section{Last Iterate}
\label{sec:last-iterate}

In this section, we study the dynamics of the last iterate of Algorithm~\ref{alg:nesterov}. 
The latter is often preferred to the averaged iterate in practice. 
%
%
In general, the noise in the gradient prevents the last iterate to converge. 
When used with constant step sizes,  only a convergence in a $O(\sigma^2)$-neighborhood of the solution can be obtained. 
Therefore variance reduction techniques (including averaging and decaying step sizes) are required.
However in the case of noiseless model, i.e., $b=\langle a, \xt{*}\rangle$ $\rho$-almost surely, last iterate convergence is possible. 
In such cases, the algorithms are inherently robust to the noise in the stochastic gradients.
This setting is particularly relevant to the training of over-parameterized models in the interpolation setting~\citep{varre2021iterate}.
%
%
%
%
%

When studying the behavior of the last iterate, we need to make an additional 4-th order assumption on the distribution of the features.
\begin{asshort} [Uniform Kurtosis] \label{asmp:uniform-kurtosis}
     There exists a finite constant $\kurt$ such that for any positive semidefinite  matrix $M$
\begin{align}
    \Ex{ \scal{a}{Ma} ~ a \otimes a } \preccurlyeq \kurt \tr{\left(M \cov \right)} \cov.
\end{align}
\end{asshort}
The above assumption holds for the Gaussian distribution with $\kappa = 3$ and is also satisfied when $\cov^{-\frac{1}{2}}a$ has sub-Gaussian tails~\citep{zou2021benign}.
Therefore Assumption~\ref{asmp:uniform-kurtosis} is not too restrictive and  is often made when analysing SGD for least squares~\citep{dieuleveutFB17,flammarion2017stochastic}.
It is nevertheless stronger than Assumption~\ref{asmp:fourth-moment}. 
For the one-hot-basis distribution, it only holds for $\kappa = 1/p_{\min}$ which can be arbitrarily large.  
It also directly implies Assumption~\ref{asmp:stat-condition-number} with a statistical condition number satisfying $\ki\leq \kurt d$. 
Yet, the previous inequality is not tight as the example of the one-hot-basis distribution shows. 

Under this assumption, we extend the previous results of \citet{flammarion2015averaging} to the general stochastic gradient oracle.
\begin{theorem} \label{thm:last-iterate}
Consider Algorithm~\ref{alg:nesterov} under Assumptions ~\ref{asmp:noise-covariance}, \ref{asmp:uniform-kurtosis} and  step sizes satisfying $ \kurt  ( \alpha + 2 \beta ) \tr{\cov} \leq 1 ,  \alpha \leq  \frac{\beta}{2 \kurt d } $. 
    In expectation, the excess risk of the last iterate $\xt{t}$ after $t$ iterations is bounded as  
    \begin{align*}
          \excessRisk{\xt{t}} &\leqslant\min\left\lbrace \frac{3}{\alpha t^2} , \frac{24 }{\beta t} \right\rbrace \nor{\xt{0} - \xt{*}}^2 + 2 \left( \left(\alpha + 2 \beta \right) \tr{H} + \frac{2 \alpha d }{\beta}  \right) \sigma^2 
    \end{align*}
\end{theorem}
Let us make some comments on the convergence of last iterate. The proof can be found in App.~\ref{app:proof:theorem:lastiterate}. 
 \begin{itemize}[leftmargin=*,topsep=0.5em]\setlength\itemsep{0.5em}
 \item  When the step-sizes are set to $\beta = \nicefrac{1}{(3 \kurt \tr{\cov})}$ and $\alpha = \nicefrac{1}{(6 d\kurt^2 \tr{\cov}})$ the upper bound  on the excess risk becomes  $ \min\left\lbrace \frac{18 \kurt^2 d \tr{H}}{t^2} , \frac{144 \kurt \tr{\cov} }{ t} \right\rbrace \nor{\xt{0} - \xt{*}}^2 + \frac{4}{\kappa} \sigma^2$.
 \item  For constant step sizes, the excess error of the last iterate does not go to zero in the presence of noise in the model. At infinity, it converges to a neighbourhood of $O(\sigma^2)$  and the constant scales with the learning rate. This neighbourhood shrinks as the step size decreases, as long as the step size of the aggressive step $ \alpha $ should decrease at a faster rate compared to $\beta$. In comparison, Nesterov accelerated gradient descent ($\alpha=\beta$) is diverging. 
\item For noiseless least squares where $\sigma=0$, we get an accelerated rate $O(\kurt d/t^2)$, which has to be compared to the $O(1/t)$-rate of SGD.  
\item 
Following \citet{berthier2020tight}, a similar result can be obtained on the minimum of the excess risk $\min\limits_{0 \leq k \leq t}\excessRisk{\xt{k}}$ by only assuming the less stringent Assumption~\ref{asmp:fourth-moment}. 
\end{itemize}

\section{Lowerbound and open questions}\label{sec:lower-bound}
In this section, we address the tightness of our result with respect to the statistical condition number $\ki$. 
In particular we study the impact of the distribution generating the stream of inputs. 
We start by defining the class of stochastic first-order algorithms for least-squares we consider. 
\begin{definition}[Stochastic First Order Algorithm for Least Squares]
\label{def:sfo}
    Given an initial point $\xt{0}$, and a distribution $\rho$, a stochastic first order algorithm generates a sequence of iterates $\xt{k}$ such that 
    \begin{align}\label{eq:classalgo}
    \xt{k} \in \xt{0} + \mathit{span}\left\lbrace  \nabla_{0} f(\xt{0}), \nabla_{1} f(\xt{1}), \cdots, \nabla_{k-1} f(\xt{k-1})  \right\rbrace \ \ \text{ for } k\geq 1,
    \end{align}
    where $\nabla_{i} f$ are the stochastic gradients at the iteration $i$ defined in \eqref{eq:SGD-oracle}. 
\end{definition}
This definition extends the definition of first order algorithms considered by \citet{Nes04} to the stochastic setting. 
This class of algorithm defined is fairly general and includes SGD and  Algorithm.~\ref{alg:nesterov}. 
By definition of the stochastic oracle, the condition~\ref{eq:classalgo} is equivalent to $ \xt{k} - \xt{0}$ belonging to the linear span of the features $\{a_1,\cdots, a_k\}$.
It is therefore not possible to control the excess error for iterations $t = O(d)$ since the optimum is then likely to be in the span of more than $d$ features.
%
However it is still possible to lowerbound the excess error in the initial stage of the process. 
This is the object of the following lemma which provides a lower bound for noiseless problems.
\begin{lemma}
\label{lem:lower-bound}
   For all starting point $\xt{0}$, there exists a distribution  $\rho$ over $ \mathbb{R}^{d}\times  \mathbb{R} $ satisfying Assumption~\ref{asmp:fourth-moment} with  $R^2 = 1$, Assumption~\ref{asmp:noise-covariance} with $\sigma=0$,  Assumption~\ref{asmp:stat-condition-number} with $\ki=d$ and an optimum $\xtp{*}$ verifying $\nor{\xt{*}^{\prime} - \xt{0}}^2 = 1$, such that  the expected excess risk of any stochastic first order algorithm is lower bounded as 
       \begin{align*}
        \excessRisk{\xt{\lfloor d/2 \rfloor}}  = \Omega\left(\frac{1}{d} \right).
    \end{align*} 
\end{lemma}
Check App.~\ref{app:proof:lemma:lower} for the proof of the lemma. The excess risk  cannot be decreased by more than a factor $d$ in less than $d$ iterations.  Fully accelerated rates $O(R^2 \nor{\xt{*}^{\prime} - \xt{0}}^2/t^2)$ are thus proscribed. Indeed, they correspond to a decrease $O(1/d^2)$ for the above problem, contradicting the lower bound.
Hence, accelerated rates should be scaled with a factor of dimension $d$.
The rate $O(d R^2 \nor{\xt{*}^{\prime} - \xt{0}}^2/t^2)$ of Theorem~\ref{thm:main} is therefore optimal at the beginning of the optimization process. 
On the other side, the SGD algorithm achieves a rate of $O(R^2 \nor{\xt{*}^{\prime} - \xt{0}}^2/t)$  on noiseless linear regression.
For the regression problem described in Lemma~\ref{lem:lower-bound}, this rate  is $O(1/d)$ and also optimal. 
%
%
%
%

The proof of the lemma follow the lines of \citet{jain2018accelerating} and considers the one-hot basis distribution. It is worth noting that the covariance matrix of the worst-case distribution can be fixed beforehand, 
i.e., for any covariance matrix, there exists a matching distribution such that direct acceleration is impossible (see details in Lemma~\ref{lem:no-acceleration}).
Therefore the lower bound does not rely on the construction of a particular Hessian, in contrast to the proof of \citet{Nes04} for the deterministic setting.
%
%
However, the proof strongly leverages the orthogonality of the features output by the oracle.  
It is still an open question to study similar complexity result for more general, e.g., Gaussian, features.



 
A different approach is to consider constraints on the computational resources used by the algorithm. \citet{dagan19b,sharan19} investigate this question from the angle of memory constraint and derive memory/samples tradeoffs for the problem of regression with Gaussian design.  
%
Although their results do no have direct implications on the convergence rate of gradient based methods, we observe some interesting phenomena when increasing the memory resource of the algorithms. 
The stochastic gradient oracle $\left(\scal{a_i}{\xt{t}} - b_i  \right) a_i$ uses a memory $O(d)$. 
If we increase the available memory to $O(d^2)$ and consider instead the running average 
$\frac{1}{t+1}\sum_{i=0}^{t} \left( \scal{a_i}{\xt{t}} - b_i  \right) a_i$
as the gradient estimate, 
 $\alpha$ no longer needs to be scaled with $d$ and we
 empirically observe  $O(1/t^2)$ convergence (see Figure~\ref{fig:space-vs-rate} in App.~\ref{app:proof:lemma:lower}).
 This empirical finding suggests that algorithms using a subquadratic amount of memory may provably converge slower than algorithms without memory constraints. Investigating such speed/memory tradeoff is outside of the scope of this paper, but is a promising direction for further research.




\section{Proof technique}

For the least squares problem, the analysis of stochastic gradient methods is well studied and techniques have been thoroughly refined. Our analysis follows the common underlying scheme. First, the iterates are rescaled to obtain a time invariant linear system. Second, the estimation error is decomposed as the sum of a bias and variance error term which are studied separately. Finally, the rate is obtained using the bias-variance decomposition. 
However there are  significant gaps yet to be filled for this particular problem. The first of many is that the \textit{existing Lyapunov techniques} for either strongly convex functions or classical SGD are not applicable (see App.~\ref{app:our-technique-for-bias}, for more details). The study of the variance error comes with a different set of challenges.
%


\paragraph{Time Rescaling.} Using the approach of \cite{flammarion2015averaging}, we first  reformulate the algorithm using the following scaled iterates \eqals{ \ut{t} := (t+1) ( \xt{t}- \xt{*} )  \qquad \vt{t} := t ( \yt{t}- \xt{*} )  \qquad \wt{t} := \zt{t} - \xt{*} \label{eq:ut_vt_zt}.}  
Using such time rescaling, we can write  Algorithm~\ref{alg:nesterov} with stochastic gradient oracle as a time-independent linear recursion (with random coefficients depending only on the observations)
\begin{align} 
    \thet{t+1} =\J_{t}  \thet{t} + \epsi{t+1},  
\end{align}
where $\thet{t} \defeq \matvz{t}$, $\J_{t} \defeq 
\begin{bmatrix} \I - \beta  a_t a_t^\top & \I - \beta  a_t a_t^\top \\ - \alpha a_t a_t^\top & \I - \alpha a_t a_t^\top \end{bmatrix} $ and $\epsi{t+1} \defeq \left(t+1\right) (b_t - \scal{\xt{*}}{a_t}) \begin{bmatrix}  \beta a_t \\ \alpha a_t    \end{bmatrix} $.  
%
%
The expected excess risk of the averaged iterate $\avgxt{T}$ can then be simply written as 
\begin{align*}
   \excessRisk{\avgxt{T}} - \excessRisk{\xt{*}}  = \frac{1}{2}  \left( \textstyle \sum \limits_{t=1}^{T+1} t \right)^{-2}  \scal{\HZero}{\Ex{\athet{T} \otimes \athet{T}}}, 
\end{align*}
where we define $\athet{T} = \sum_{t=0}^{T} \thet{t}$ the sum of the rescaled iterates.
All that remains to do is to upper-bound the covariance  $\Ex{\athet{T} \otimes \athet{T}}$. 
It now becomes clear why we consider the averaging scheme in \eqref{eq:estimator} instead of the classical average of \citet{polyak1992acceleration}: it integrates  well with our time re-scaling and makes the analysis simpler.

\paragraph{Bias-Variance Decomposition.}
To upper bound the covariance of our estimator $\athet{t}$ we form two independent sub-problems: 
%
\begin{itemize}[leftmargin=*]
        \item \textbf{Bias recursion:}  the least squares problem is assumed to be noiseless, i.e, $\epsi{t} = 0$ for all $t\geq0$. It amounts to the studying the following recursion
        \eqal{ \thetb{t+1} =  \J_t ~ \thetb{t} \text{ started from } \thetb{0} =  \thet{0}. \label{eq:rec-bias}  }
        \item \textbf{Variance recursion:}  the recursion starts at the optimum ($\xt{*}$) and the noise $\epsi{t}$ drive the dynamics. It is equivalent to the following recursion
        \eqal{ \thetv{t+1} =  \J_{t} \thetv{t} + \epsi{t+1} \text{ started from } \thetv{0} =  \bm{0}. \label{eq:rec-var}}
    \end{itemize}
The bias-variance decomposition (see Lemma~\ref{app:lemma:bias:variance} in App.~\ref{app:operator}) consists of upperbounding the covariance of the iterates as 
 \begin{align} \label{eq:bias-variance}
            \E\left[ \athet{T} \otimes \athet{T} \right] \preccurlyeq 2 \left( \E\left[ \athetb{T} \otimes \athetb{T} \right] + \E\left[ \athetv{T} \otimes \athetv{T} \right] \right),
        \end{align}
where $\athetb{T} \defeq \sum_{t=0}^{T} \thetb{T}$ and $\athetv{T} \defeq \sum_{t=0}^{T} \thetv{T}$. The bias error and the variance error can then be studied separately. 

The bias error is directly given by the following lemma which controls the finite sum of the excess bias risk. In the proof of the lemma,  we relate the sum of the expected covariances of the iterates  Algorithm~\ref{alg:nesterov} with stochastic gradients  to  the sum of the covariance of iterates of Algorithm~\ref{alg:nesterov} with exact gradients.  For detailed proof, see Lemma~\ref{lem:bias_sum_covariance}.  
\begin{lemma}[Potential for Bias]
Under Assumptions  \ref{asmp:fourth-moment},\ref{asmp:stat-condition-number} and the step-sizes satisfying the conditions of Theorem~\ref{thm:main}. For $T \geq 0$,
    \begin{align*}
        \sum_{t=0}^{T} \scal{\HZero}{\Ex{ \thetb{t} \otimes \thetb{t} }}  \leq \min\left\lbrace \frac{3(T+1)}{\alpha} , \frac{12(T+1)(T+2)}{\beta} \right\rbrace \nor{\xt{0} - \xt{*}}^2.
    \end{align*}
\end{lemma}
%
%
In order to bound the variance error, we first carefully expand the covariance  of $\athetv{t}$ and relate it to the  covariances of the $\thetv{t}$ (see Lemma~\ref{lem:var:cov-athet} in App.~\ref{app:variance}). We then control each of these covariances using the following lemma which shows that they are of order $O(t^2)$. See Lemma~\ref{lem:covariance-t} in App.~\ref{app:variance}, Lemma~\ref{eq:lem:inv:opt} in App.~\ref{app:inverse-of-operators} for proof. 
\begin{lemma} \label{lem:var-t-main}
   For any $t \geq 0$ and step-sizes satisfying condition of Theorem~\ref{thm:main}, the covariance is characterized by
    \begin{align*}
        \Ex{ \thetv{t} \otimes \thetv{t} } \preccurlyeq t^2 \sigma^2  \begin{bmatrix}
            2 \alpha  ( \beta \cov )^{-1}  +  (2\beta -3 \alpha) \I &  \alpha \beta^{-1}  (2\beta  - \alpha) \I   \\
            \alpha \beta^{-1}  (2\beta  - \alpha) \I  &  2  \alpha^2 \beta^{-1} \I 
          \end{bmatrix}.
    \end{align*}
\end{lemma}
The lemma is proved by studying  $\Ex{ \thetv{t} \otimes \thetv{t} }/t^2$ in the limit of $t \to \infty$.

\paragraph{Last iterate convergence.}
The proof for the last iterate follows the same lines and still uses the bias variance decomposition. 
The main challenge is to bound the bias error. Following~\citet{varre2021iterate}, we show a closed recursion where the  excess risk at time $T$ can be related to the excess risk of the previous iterates through a discrete Volterra integral as stated in the following lemma. 
\begin{lemma} [Final Iterate Risk] \label{lem:final:iterate:main:paper}
Under Assumption \ref{asmp:uniform-kurtosis} and the step-sizes satisfying $\alpha\leq \beta\leq 1/L $. For $T \geq 0$, the last iterate excess error can be determined by the following discrete Volterra integral
    \begin{align*}
    \ft{\thetb{T}} &\leqslant \min\left\lbrace \frac{1}{\alpha} , \frac{8 (t+1)}{\beta} \right\rbrace \nor{\xt{0} - \xt{*}}^2 + \sum_{t=0}^{T-1} \sum_{i=1}^{d} \gt{\cov}{t-k} \ft{\thetb{t}}, 
    \end{align*}
where $\ft{\thetb{t}} \defeq \scal{\HZero}{\Ex{\thetb{t} \otimes \thetb{t}}}$ and the  kernel
$
        \gt{\cov}{t} 
$
is defined in \eqref{eq:kern} in App.~\ref{sec:biaslastiterate}.
\end{lemma}
We recognize here a new bias-variance decomposition. The decrease of the function value $\ft{\thetb{t}}$ is controlled by the sum of a term characterizing how fast the initial conditions are forgotten, and a term characterizing how the gradient noise reverberates through the iterates. The final result is then obtained by a simple induction. For the proof, check Lemma~\ref{lem:final:iterate:app} in App.~\ref{sec:biaslastiterate} .

\section{Experiments} \label{section:exp}

In this section, we illustrate our theoretical findings of Theorems~\ref{thm:main},~\ref{thm:last-iterate} on synthetic data.
For $d=50$, we consider Gaussian distributed inputs $a_n$ with a random covariance $\cov$ whose eigenvalues scales as $1/i^4$, for $1\leq i \leq d$ and optimum $\xt{*}$ which projects equally on the eigenvectors of the covariance. 
The outputs $b_t$ are generated through $b_t = \langle a_t, \xt{*}\rangle +\varepsilon_t$, where $\varepsilon_t\sim\mathcal N(0,\sigma^2)$.
 The step-sizes are chosen as $\beta = \nicefrac{1}{3\tr{\cov}} , \alpha = \nicefrac{1}{(3d~\tr{\cov})}$  for our algorithm; and  $\gamma = \nicefrac{1}{3 \tr{\cov}} $ for SGD. The parameters of ASGD are chosen following \citet{jain2018accelerating}.
All results are averaged over 10 repetitions.
\paragraph{Last Iterate.} 
The left plot in Figure~\ref{fig:synthetic} corresponds to the convergence of the excess risk of the last iterate on a synthetic noiseless regression, i.e., $\sigma = 0$. 
We compare Algorithm~\ref{alg:nesterov} (AcSGD) with the algorithm of \citet{jain2018accelerating} (ASGD) and SGD. 
Note that for the first $O(d)$ iterations, our algorithm matches the performance of SGD. For $ t > O(d)$, the acceleration starts and we observe a rate $O(d/t^2)$.
Finally,  strong convexity takes effect only after a large number of iterations. ASGD decays with a linear rate thereafter. 

\paragraph{Averaging.} The right plot in Figure~\ref{fig:synthetic} corresponds to the performance of the averaged iterate on a noisy least squares problem with $\sigma = 0.02$.
%
%
We compare our AcSGD with averaging defined in \eqref{eq:estimator} (AvAcSGD), ASGD with tail averaging (tail-ASGD) and SGD (AvgSGD) with Polyak-Ruppert averaging.
For $O(d)$ iterations, our algorithm matches the rate of SGD with averaging, then exhibits an accelerated rate of $O(d/t^2)$. Finally, it decays with the optimal asymptotic rate $\sigma^2 d/t$.
\begin{figure}[ht]
\centering
\begin{minipage}[c]{.40\linewidth}
\includegraphics[width=\linewidth]{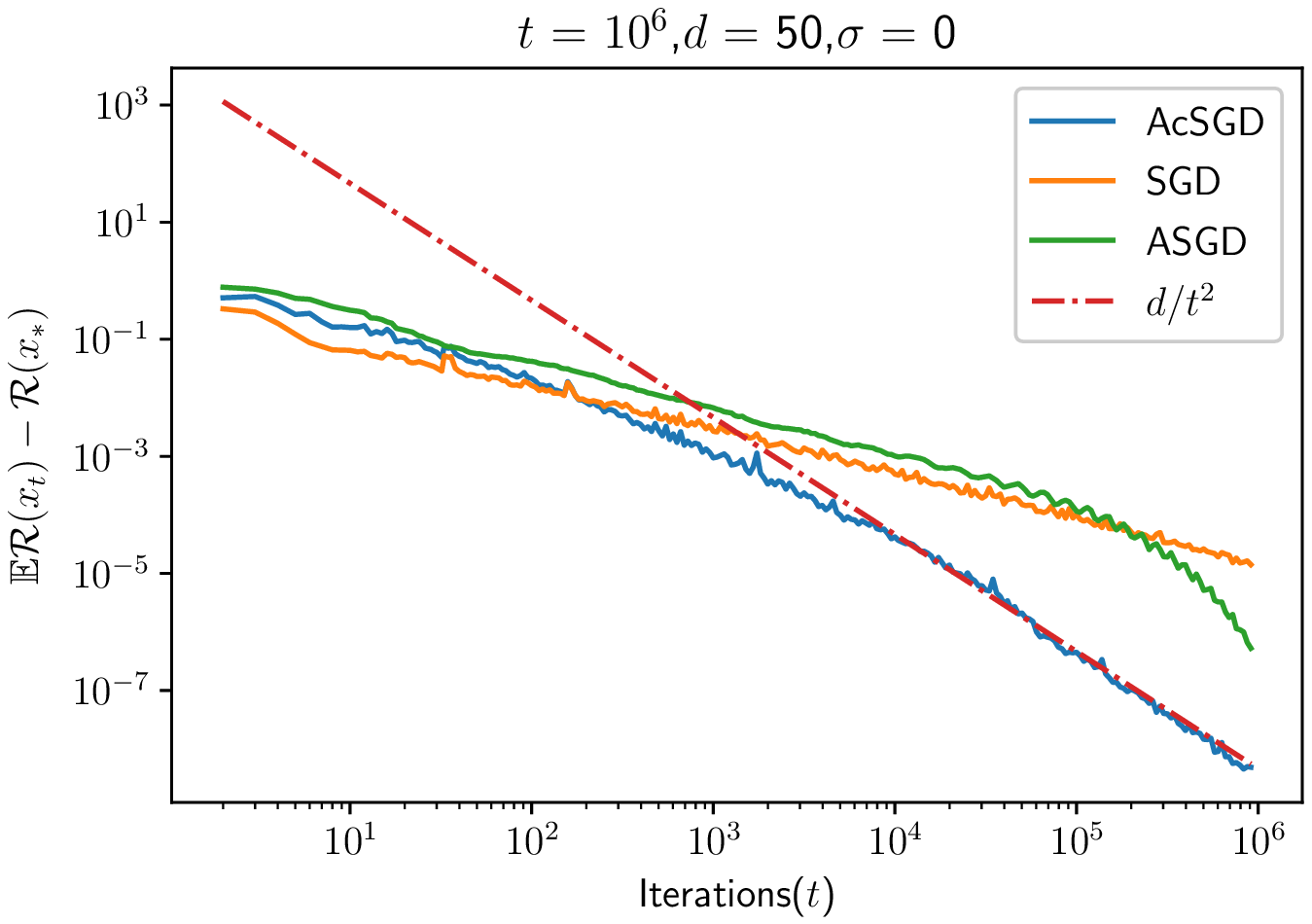}
   \end{minipage}
   \begin{minipage}[c]{.40\linewidth}
\includegraphics[width=\linewidth]{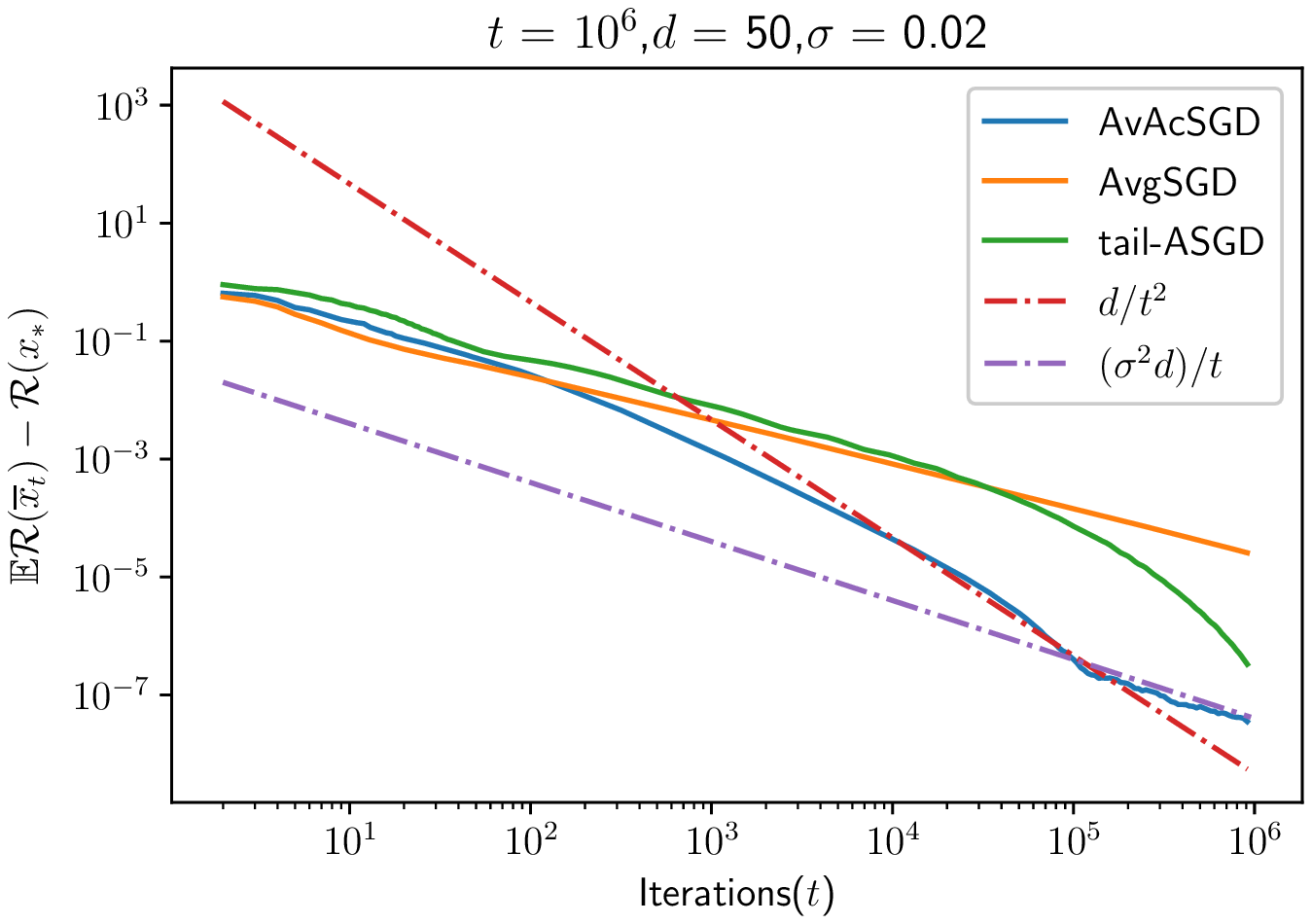}
   \end{minipage}
  \caption{Least-squares regression. Left: Last iterate convergence on a noiseless problem.  The plot exhibits a rate $O(d/t^2)$ given by Thm.~\ref{thm:last-iterate}. Right: Averaged iterate convergence on a noisy problem. The plot first exhibits a rate $O(d/t^2)$ and then the optimal rate $O(\sigma^2 d/t)$ as predicted by Thm.~\ref{thm:main}.}
     \label{fig:synthetic}
\end{figure}
\section{Conclusion}

In this paper, we show that stochastic accelerated gradient descent can be made robust to gradient noise in the case of least-squares regression. 
Our new algorithm, based on a simple step-size modification of the celebrated Nesterov accelerated gradient is the first stochastic algorithm which provably accelerates the convergence of the bias while maintaining the optimal convergence of the variance for non-strongly-convex least-squares.
There are a number of further direction worth pursuing. 
Our current analysis is limited to quadratic functions defined in a Euclidean space. 
An extension of our analysis to all smooth or self-concordant like functions would broaden the applicability of our algorithm.
Finally an extension to Hilbert spaces and 
kernel-based least-squares regression with estimation rates under the usual non-parametric capacity and source conditions would be an interesting development of this work. 
\acks{ The authors thank Loucas Pillaud-Vivien and  Keivan Rezaei for valuable discussions.}

\bibliographystyle{plainnat}
\bibliography{acceleration}

\appendix

\section{Further Setup and Preliminaries}
\label{sec:app:setup}

\paragraph{Organization.}The appendix is organized as follows,
\begin{itemize}
    \item In Section~\ref{sec:app:setup}, we extend the setup of the problem in App.~\ref{subsec:prelim} and introduce operators in App.~\ref{app:operator}  to study the recursions of covariance of the estimator. In the later subsection, we give the proof for bias variance decomposition.
    \item In Section \ref{app:main-proofs}, we give the proofs for Theorem~\ref{thm:main}, Theorem~\ref{thm:last-iterate} and Lemma~\ref{lem:lower-bound}.
    \item In Section \ref{app:bias-variance-recursions}, we study the  recursions on expected covariance of the bias and the variance processes. 
    \item In Section \ref{app:inverse-of-operators}, we investigate the properties of the operators. In particular, we are interested in inverting few operators. 
    \item In Section \ref{app:technical-lemmas}, we study the summations of geometric series of a particular $2 \times 2$ matrix by considering its eigendecomposition. 
\end{itemize}

\subsection{Preliminaries} \label{subsec:prelim}

\paragraph{Notations.} We denote the stream of i.i.d samples by $\left( a_i, b_i \right)_{i \geq 1}$.  We use $\otimes$ to denote the tensor product and $\otimes_{k}$ to denote the Kronecker product.  Let $\filt{t}$ denote the filtration generated by the samples $ \left \lbrace \left(a_i,b_i \right)_{i=1}^{t} \right \rbrace $.   

\paragraph{Additive and Multiplicative Noise.}

Define  for $ t \geq 1 $,  \eqal{\eta_t = b_t - \scal{\xt{*}}{a_t} \label{eq:eta},} since $\xt{*}$ is the optimum, from the first order optimality of $\xt{*}$, 
\begin{align} \label{eq:eta_a}
    \Ex{ \eta_{t} a_t } = 0.
\end{align}
In context of least squares, SGD oracle can be written as follows. Let $(a_t,b_t)$ be the sample at iteration $t$,  the gradient at $\xt{t}$ is 
\begin{align*}
    \nabla_{t} \risk{\xt{t}}  &= a_t \left( \scal{a_t}{\xt{t}} - {b_t} \right), \\
     &= a_t \left( \scal{a_t}{\xt{t}} - \left( \scal{a_t}{\xt{*}} + \eta_t \right) \right) = a_t a_t^{\top} \left(\xt{t} -\xt{*}\right) -  \eta_t a_t.
\end{align*} 
From this, the stochastic gradient can be written as 
\begin{align} \label{eq:SGD-oracle-eta}
    \nabla_{t} \risk{\xt{t}} = a_t a_t^{\top} \left(\xt{t} -\xt{*}\right) -  \eta_t a_t.
\end{align}
As the exact gradient will be $\cov\left( \xt{t} - \xt{*} \right)$ the noise in the oracle is
\begin{align*}
    \left( \cov\left( \xt{t} - \xt{*} \right) \right) - \left( a_t a_t^{\top} \left(\xt{t} -\xt{*}\right) -  \eta_t a_t \right) = ( \cov - a_t a_t^\top ) \left( \xt{t} - \xt{*} \right) + \eta_t a_t.
\end{align*}

Note that the zero mean noise $ ( \cov - a_t a_t^\top ) $ is multiplicative in nature and hence called multiplicative noise. The zero mean noise $\eta_t a_t$ is called additive noise. In the work of \citet{dieuleveutFB17}, stochastic gradients of form $ \cov\left( \xt{t} - \xt{*} \right) + \epsilon $ for some bounded-variance random variable $\epsilon$ are considered. Hence, the results holds only in the case of stochastic oracles with additive noise.

\subsection{Recursion after Rescaling} \label{app:rescaling}

Using \eqref{eq:SGD-oracle-eta}, we can write Algorithm~\ref{alg:nesterov} as follows, \begin{align}
     \yt{t+1} &= \xt{t} - \beta a_t a_t^{\top} ( \xt{t} - \xt{*} ) + \beta \eta_t a_t, \label{eq:yt:asgd} \\ 
    \zt{t+1} &= \zt{t} - \alpha ( t+1) a_t a_t^{\top} (\xt{t} -\xt{*} ) + \alpha ( t+1)  \eta_t a_t, \label{eq:zt:asgd} \\ 
    (t+2)\xt{t+1} &= (t+1)\yt{t+1} + \zt{t+1}. \label{eq:xt:asgd} 
\end{align} 
Recalling the time rescaling of the iterates 
\eqals{ \ut{t} = (t+1) ( \xt{t}- \xt{*} ), \label{eq:ut} \\  \vt{t} = t ( \yt{t}- \xt{*} ), \label{eq:vt} \\ \wt{t} = \zt{t} - \xt{*} \label{eq:zt}.  }  
Now we rewrite the recursion using these rescaled iterates.
Multiplying \eqref{eq:yt:asgd} by $t+1$, and using \eqref{eq:ut} and \eqref{eq:vt}, we get, 
\begin{align*}
\vt{t+1} &= \ut{t} - \beta a_t a_t^{\top} \ut{t} + \beta \eta_t a_t (t+1),  \\
\textrm{ using \eqref{eq:zt:asgd} and \eqref{eq:ut}, } \qquad
 \wt{t+1} &= \wt{t} - \alpha a_t a_t^{\top} (\ut{t}) + \alpha ( t+1)  \eta_t a_t,  \\
  \textrm{from \eqref{eq:xt:asgd},} \qquad \ut{t} &= \vt{t} + \wt{t}, \\ 
  \vt{t+1} &= ( \I - \beta a_t a_t^{\top} ) \left( \vt{t} + \wt{t}   \right) + \beta \eta_t a_t (t+1),  \\
  \wt{t+1} &= \wt{t} - \alpha a_t a_t^{\top} (\vt{t} + \wt{t}) + \alpha ( t+1)  \eta_t a_t.
\end{align*}
Writing these updates compactly in form of a matrix recursion gives, 
\begin{align*}
    \matvz{t+1} = \begin{bmatrix} \I - \beta  a_t a_t^\top & \I - \beta  a_t a_t^\top \\ - \alpha a_t a_t^\top & \I - \alpha a_t a_t^\top \end{bmatrix}
        \matvz{t} + \left(t+1\right) \eta_t  \begin{bmatrix} \beta a_t \\ \alpha a_t    \end{bmatrix}. 
\end{align*}
The above recursion can be written as follows, \begin{align} \label{eq:app-matrix-rec}
    \thet{t+1} = \J_{t} \thet{t} + \epsi{t+1},
\end{align}
where we defined  $\thet{t} \defeq \matvz{t}$, the random matrix  $\J_{t} \defeq 
\begin{bmatrix} \I - \beta  a_t a_t^\top & \I - \beta  a_t a_t^\top \\ - \alpha a_t a_t^\top & \I - \alpha a_t a_t^\top \end{bmatrix} $ and the random noise vector $\epsi{t+1} \defeq \left(t+1\right) \eta_t \begin{bmatrix}  \beta a_t \\ \alpha a_t    \end{bmatrix} $.

\paragraph{Excess Risk of the estimator.} The excess risk of any estimate $\bm{x}$ can be written as 
\begin{align}
    \label{eq:excessRisk}
    \risk{\bm{x}} - \risk{\bm{\xt{*}}} = \frac{1}{2} \scal{\bm{x} - \xt{*}}{\cov \left( \bm{x} - \xt{*} \right) }.
\end{align}
Our estimator is defined in \eqref{eq:estimator} as a time-weighted averaged. 
Recall,
\begin{align*}
    \avgxt{T} &= \frac{\sum_{t=0}^{T} (t+1) \xt{t} }{\sum_{t=0}^{T} (t+1) }, \\
    \avgxt{T} - \xt{*} &= \frac{\sum_{t=0}^{T} (t+1) \left( \xt{t} - \xt{*} \right) }{\sum_{t=0}^{T} (t+1) } = \frac{\sum_{t=0}^{T} \ut{t} }{\sum_{t=0}^{T} (t+1) }. 
\end{align*}
Using the above formulation of excess risk, we have 
\begin{align*}
    \left( \sum_{t=0}^{T} (t+1) \right)^{2} \cdot \left(\risk{\avgxt{T}} - \risk{\bm{\xt{*}}} \right) = \frac{1}{2} \scal{\sum_{t=0}^{T} \ut{t}}{\cov \left( \sum_{t=0}^{T} \ut{t} \right) }.
\end{align*}
We relate this to the covariance of $\athet{T}$ in the following way, 
\begin{align*}
    \athet{T} &= \sum_{t=0}^{T} \thet{t} = \sum_{t=0}^{T} \matvz{t}  = \begin{bmatrix}
        \sum\limits_{t=0}^{T} \vt{t} \\
        \sum\limits_{t=0}^{T} \wt{t}
    \end{bmatrix}.
\end{align*}
From \eqref{eq:xt:asgd} we have the fact that $ \ut{t} = \vt{t} + \wt{t} $, for $ t \geq 1$. Using this, 
\begin{align*}
    \sum\limits_{t=0}^{T} \ut{t} = \sum\limits_{t=0}^{T} \vt{t} + \sum\limits_{t=0}^{T} \wt{t}.
\end{align*} 
From the above formulations, with some simple algebra we get, 
\begin{align*}
    \scal{\sum_{t=0}^{T} \ut{t}}{\cov \left( \sum_{t=0}^{T} \ut{t} \right) } = \scal{\HZero}{\athet{T} \otimes \athet{T}}.
\end{align*}
Hence, by taking expectation, the excess risk can be related to the covariance of  $\athet{T}$ as, 
\begin{align} \label{eq:excessrisk_in_thet}
    \Ex{\risk{\avgxt{T}}} - \risk{\bm{\xt{*}}} = \frac{1}{2}  \left( \sum_{t=0}^{T} (t+1) \right)^{-2} \scal{\HZero}{\Ex{\athet{T} \otimes \athet{T}}}.
\end{align}

\paragraph{Step sizes.} We use the following conditions on the step sizes $\alpha,\beta$
\begin{gather} \label{eq:asgd-step-size}
        ( \alpha + 2 \beta ) R^2 \leqslant 1 ,\quad \alpha \leqslant  \frac{\beta}{2 \ki}.
    \end{gather}
These conditions are a direct result of our analysis. 


\paragraph{Eigen Decomposition of $\cov$.} Since the covariance is positive definite, the eigendecomposition of $\cov$ is given as follows, 
\begin{align} \label{eq:eigen-def-covariance}
    \cov \defeq \sum_{i=1}^{d} \lambda_i e_i e_i^{\top},
\end{align}
where $\lambda_i > 0$'s are the eigenvalues and $e_i$'s are orthonormal eigenvectors.

\subsection{Operators} \label{app:operator}
As seen above,  the excess risk in expectation can be related to the expected covariance of the  $\athet{T}$ ,i.e., $\E\left[ \athet{T} \otimes \athet{T} \right]$. In order to aid the analysis of the covariance, we introduce different operators. 
The expected value of $\J_{t}$, denoted by  $\A = \Ex{\J_{t}}$ 
is given by \eqal{ \A &\defeq  \begin{bmatrix} \I - \beta \cov & \I - \beta  \cov \\ - \alpha \cov & \I - \alpha \cov \end{bmatrix}.  } 
If the feature $a$ is sampled according to the marginal distribution of $\rho$ ,i.e., $(a,b) \sim \rho$, define the random matrices $\mathcal{J}$ as follows  
\eqals{ 
    \mathcal{J} &\defeq \begin{bmatrix} \I - \beta  a a^\top & \I - \beta  a a^\top \\ - \alpha a a^\top & \I - \alpha a a^\top \end{bmatrix}.
    }
Note that $\J_{t}$ from~\eqref{eq:app-matrix-rec} and $\J$ are identically distributed. 
\begin{definition} \label{def:operator}
    For any PSD matrix $\Theta$, the operators $ \opT , \opTil , \opM $ are defined as follows \begin{enumerate}[label=(\alph*)]
        \item $ \begin{aligned}[t]
            \opT\mycirc\Theta &\defeq \E\left[ \J \Theta \J^\top \right]   \label{op:T}
        \end{aligned} $
        \item $ \begin{aligned}[t]
            \opTil \mycirc\Theta &\defeq \A \Theta \A^{\top} \label{op:Tilde}
        \end{aligned} $
        \item $ \begin{aligned}[t]
            \opM \mycirc\Theta &\defeq \E\left[  \left(\J - \A\right) \Theta \left(\J - \A\right)^\top \right]  \label{op:M}
        \end{aligned} $
    \end{enumerate}
\end{definition}
We proceed to show a few properties of these operators.
\begin{lemma} \label{lem:operator}
    For the operators  $ \opT , \opTil , \opM $, the following properties holds.
    \begin{enumerate}[label=(\alph*)]
        \item $ \opT , \opTil , \opM $ are symmetric and positive 
        \item $ \opT  =  \opTil  +  \opM $
    \end{enumerate}
    The operator $\mathcal{O}$ is defined as positive if for any PSD matrix $\Theta$, $\mathcal{O}\mycirc\Theta$ is also PSD.
\end{lemma}
\begin{proof}
    For any vector $\nu$, consider the following scalar product, 
    \begin{align*}
        \scal{\nu}{\J \Theta \J^\top \nu} = \scal{\left(\J^\top \nu\right)~}{~\Theta~\left(\J^\top \nu\right)}.
    \end{align*} 
    This quantity is non-negative as $\Theta$ is a PSD.  Hence $\opT$ is positive. Similarly the other two operators are also positive. For the second statement,
    \begin{align*}
        \E\left[  \left(\J - \A\right) \Theta \left(\J - \A\right)^\top \right] = \Ex{ \J \Theta \J^{\top}  } - \Ex{ \J \Theta \A^\top   } - \Ex{ \A \Theta \J^{\top} } + \Ex{ \A \Theta \A^{\top} }.
    \end{align*}
    Using $ \A = \Ex{\J} $, 
    \begin{align*}
        \E\left[  \left(\J - \A\right) \Theta \left(\J - \A\right)^\top \right] &= \Ex{ \J \Theta \J^{\top}  } - \Ex{ \A \Theta \A^{\top} }.
    \end{align*} 
This completes the proof of the lemma. 
\end{proof} 
\begin{remark} \label{eq:transpose_operator}
    For any PSD matrix $\Theta$ and any operators $\opT, \opM, \opTil$, the transpose is defined as following,
      \begin{enumerate}[label=(\alph*)]
          \item $ \begin{aligned}[t]
              \opT^{\top} \mycirc\Theta &\defeq \E\left[ \J^\top \Theta \J \right].   \label{op:T:top}
          \end{aligned} $
          \item $ \begin{aligned}[t]
              \opTil^{\top} \mycirc\Theta &\defeq \A^{\top} \Theta \A. \label{op:Tilde:top}
          \end{aligned} $
          \item $ \begin{aligned}[t]
              \opM^{\top} \mycirc\Theta &\defeq \E\left[  \left(\J - \A\right)^\top \Theta \left(\J - \A\right) \right].  \label{op:M:top}
          \end{aligned} $
      \end{enumerate}
  \end{remark} 
Having introduced operators, we present a lemma which is central to the analysis, we give a almost eigenvector and eigenvalue of the operators. We call it an almost eigenvector as only an upperbound holds in this case.   
    \begin{lemma} \label{lemma:alnost_eigen}
        For step sizes satisfying Condition~\ref{eq:asgd-step-size}, the following properties hold on the inverse and eigen values of operators $ \opT , \opTil , \opM $
\begin{enumerate}[label = (\alph*)]
    \item For stepsizes $0 < \alpha , \beta < \frac{1}{L}$, $( 1 - \opTil)^{-1}$ exists. 
    \item $\noise$ is an almost eigen vector of   $ \opM \mycirc ( 1 - \opTil)^{-1}  $ with an eigen value less than 1 and $ ( 1 - \opT)^{-1} \mycirc \noise $ exists, 
    \begin{align*} 
        \opM \mycirc ( 1 - \opTil)^{-1} \mycirc \noise  &\preccurlyeq \frac{2}{3} \noise.
    \end{align*}
    \item $\coff$ is an almost eigen vector of   $ \left( \opM^{\top} \mycirc  \left( \Id - \opTil^{\top} \right)^{-1} \right)  $ with an eigen value less than 1,
        \begin{align*}
            \left( \opM^{\top} \mycirc  \left( \Id - \opTil^{\top} \right)^{-1} \right) \mycirc \coff \preccurlyeq \frac{2}{3} \coff.
        \end{align*} 
\end{enumerate}
    where 
    \begin{align}\label{eq:noise_and_coeff}
        \coff \defeq \HZero \qquad \noise \defeq \begin{bmatrix} \beta^2 \cov & \alpha\beta \cov \\ \alpha \beta \cov & \alpha^2 \cov \end{bmatrix}.
    \end{align}
    \end{lemma}

\begin{proof}
    From the diagonalization of the covariance $\cov$ from~\eqref{eq:eigen-def-covariance}, note that $\A$ can be diagonalized as follows 
    \begin{align*}
        \A = \sum_{i=1}^{d} \begin{bmatrix}
            1 - \beta \lambda_i & 1 - \beta \lambda_i \\
            - \alpha \lambda_i & 1 - \alpha \lambda_i 
        \end{bmatrix} \otimes_{k} e_ie_i^{\top}.
    \end{align*}
    From Property~\ref{pro:eigen-decomposition-A}, $0 < \alpha , \beta < \frac{1}{L}$ the absolute value of $\A$ eigen values will be less than 1 and for any PSD matrix $\Theta$ the inverse can be defined by the sum of geometric series as follows, 
    \begin{align*}
        ( 1 - \opTil)^{-1} \mycirc \Theta = \sum_{t \geq 0} \A^{t} \Theta \left(\A^{\top}\right)^{t}.
    \end{align*}    
    
To compute $( 1 - \opTil)^{-1} \mycirc \noise$, although the calculations are a bit extensive, the underlying scheme remains the same. After formulating the inverse as a sum of geometric series, we use the diagonalization of the $\A$ and $\noise$ to compute the geometric series.  In the last part to compute $\opM \mycirc ( 1 - \opTil)^{-1}$, we use Property~\ref{prop:operator-M-upperbound}  and  Assumptions~\ref{asmp:fourth-moment},~\ref{asmp:stat-condition-number} on distribution $\rho$ along with the conditions on the step-sizes~\eqref{eq:asgd-step-size} to get the final bounds.
   The remaining parts can be proven using  Lemmas~\ref{eq:lem:inv:opt},~\ref{lem:opm-optil-inv},~\ref{lem:opM:opTil:transpose}. 
\end{proof}
\paragraph{Bias-Variance Decomposition.} Recall the bias recursion \eqref{eq:rec-bias}, the variance recursion \eqref{eq:rec-var}. 
\begin{align*}
    \thetb{t+1} &=  \J_t \thetb{t} \text{ started from } \thetb{0} =  \thet{0}, \\
    \thetv{t+1} &=  \J_{t} \thetv{t} + \epsi{t+1} \text{ started from } \thetv{0} =  \bm{0}.
\end{align*}
Now we  prove a bias-variance decomposition lemma. Similar lemmas have been derived in the works of \citet{bach2013non,jain2018accelerating}. Following the proof in these works, we re-derive it here for the sake of completeness. 
\begin{lemma} \label{app:lemma:bias:variance}
    For $T \geq 0$, the expected covariance of $\athet{T}$ can be bounded as follows, 
    \begin{align} \label{app:eq:bias-variance}
        \E\left[ \athet{T} \otimes \athet{T} \right] \preccurlyeq 2 \left( \E\left[ \athetb{T} \otimes \athetb{T} \right] + \E\left[ \athetv{T} \otimes \athetv{T} \right] \right).
    \end{align}
\end{lemma}

\begin{proof}
    In the first part, using induction we prove that  $ \thet{t} = \thetb{t} + \thetv{t} $. Note that the hypothesis holds at $k=0$ because $ \thetb{0} =  \thet{0} , \thetv{0} =  \bm{0} $. Assume that $ \thet{t} = \thetb{t} + \thetv{t} $ holds for $ 0 \leq t \leq k-1$. We prove that hypothesis also holds for $k$. From the recursion on $\thet{t}$ in \eqref{eq:app-matrix-rec}, we get, 
    \begin{align*}
        \thet{k} &= \J_{k-1}  \thet{k-1} + \epsi{k},  \\
        \thet{k} &= \J_{k-1} \left( \thetb{k-1} + \thetv{k-1} \right) + \epsi{k},\quad \text{from induction hypothesis,} \\
        &= \J_{k-1} \thetb{k-1} + \J_{k-1} \thetv{k-1} + \epsi{k}.
    \end{align*}
    Form the recursion of bias and variance \eqref{eq:rec-bias}, \eqref{eq:rec-var}. We show that $\thet{k} = \thetb{k} + \thetv{k}$. From induction, this is true for all $k \geq 0$. Summing these equalities from $k = 0, \ldots ,T$, we get, 
    \begin{align*}
        \athet{T} = \athetb{T} + \athetv{T}.
    \end{align*}
    Using the Cauchy Schwarz inequality and then taking expectation, we get the statement of the lemma. 
\end{proof}

\paragraph{Recursions on Covariance.} In the following lemma, we show how the recursions on the expected covariance of the bias and variance processes are governed by the operators defined above. 
\begin{lemma} \label{lem:covariance:recursion} For $t \geq 0$, the recursion on the covariance satisfies 
    \begin{align*}
        \E \left[ \thetb{t+1} \otimes \thetb{t+1} \right] &= \opT \mycirc \E \left[ \thetb{t} \otimes \thetb{t} \right], \\
    \E \left[ \thetv{t+1} \otimes \thetv{t+1} \right] &= \opT \mycirc \E \left[ \thetv{t} \otimes \thetv{t} \right] + \Ex{ \epsi{t+1} \otimes \epsi{t+1} }.
    \end{align*}    
\end{lemma}
\begin{proof}
From the recursion of the bias process~\eqref{eq:rec-bias}, 
\begin{align*}
    \thetb{t+1} =  \J_{t} \thetb{t}. 
\end{align*}
Now the expectation of covariance is 
\begin{align*}
    \Ex{\thetb{t+1} \otimes \thetb{t+1}} = \Ex{ \J_{t} \left[ \thetb{t} \otimes \thetb{t} \right] \J_{t}^{\top} }.
\end{align*}
Note that $\J_{t}$ is independent of $\thetb{t}$. Hence using the definition of operator $\opT$ completes the proof of the first part. Now, from the recursion of the variance process~\eqref{eq:rec-var}, 
\begin{align*}
    \thetv{t+1} =  \J_{t} \thetv{t} + \epsi{t+1}.
\end{align*}
As we know that $\thetv{0} = \bm{0}$ and for $t \geq 1$, $\Ex{\epsi{t}} = 0$ from \eqref{eq:eta_a}. As $\J_{t}$ is independent of $\thetv{t}$ , we get $\Ex{\thetv{t+1}} = \A \Ex{\thetv{t}}$. Combining these we have for $t \geq 0$, $ \Ex{\thetv{t}} = \bm{0}$. Now, the expectation of the covariance is 
\begin{align*}
    \thetv{t+1} \otimes \thetv{t+1} &=  \left(\J_{t}\thetv{t} + \epsi{t+1} \right) \otimes \left( \J_{t}\thetv{t} + \epsi{t+1} \right), \\
    \Ex{\thetv{t+1} \otimes \thetv{t+1}} &=  \Ex{ \left( \J_{t}\thetv{t} + \epsi{t+1} \right) \otimes \left( \J_{t}\thetv{t} + \epsi{t+1} \right) }.
\end{align*}
Using the fact that $\J_{t}, \epsi{t+1}$ are independent of $\thetv{t}$ and $\Ex{\thetv{t}} = 0$,
\begin{align*}
    \E \left[ \thetv{t+1} \otimes \thetv{t+1} \right] &=  \E \left[ \J_{t} ~ \left[ \thetv{t}  \otimes \thetv{t} \right] ~ \J_{t}^{\top} \right]   + \Ex{ \epsi{t+1} \otimes \epsi{t+1} }.
\end{align*}
Note that $\J_{t}$ is independent of $\thetv{t}$. Hence using the definition of operator $\opT$, 
\begin{align} \label{eq:var:covariance-recursion}
    \E \left[ \thetv{t+1} \otimes \thetv{t+1} \right] &= \opT \mycirc \E \left[ \thetv{t} \otimes \thetv{t} \right] + \Ex{ \epsi{t+1} \otimes \epsi{t+1} }.
\end{align}
\end{proof} 

\subsection{Mini-Batching}
In this subsection, we discuss how we can use the same proof techniques for the mini-batch stochastic gradient oracles. Recall the mini-batch oracle  for some batch size $b$ with samples $(a_{t,i}, b_{t,i}) \sim \rho$, for $1 \leq i \leq b$,
\begin{align}
    \nabla_{t}\risk{\xt{t}} = \frac{1}{b} \sum_{i=1}^{b} a_{t,i} \left( \scal{a_{t,i}}{\xt{t}} - {b_{t,i}} \right). 
\end{align}
Following the approach in \ref{app:rescaling}, we get the time rescaled recursion with 
\begin{align*}
    \thet{t+1} =  \J_{t}^{mb} \thet{t} + \epsi{t+1}^{mb}.
\end{align*}
where \begin{align*}
    \J_{t}^{mb} &= \frac{1}{b} \sum_{i=1}^{b} \J_{t,i}  \quad , \quad  \J_{t,i} = \begin{bmatrix} \I - \beta  a_{t,i} a_{t,i}^\top & \I - \beta a_{t,i} a_{t,i}^\top \\ - \alpha a_{t,i} a_{t,i}^\top & \I - \alpha a_{t,i} a_{t,i}^\top \end{bmatrix},  \\
    \epsi{t+1}^{mb} &= \frac{1}{b} \sum_{i=1}^{b} \epsi{t,i}  \quad , \quad  \epsi{t,i} = ( b_{t,i}- \scal{a_{t,i}}{\xt{*}} ) \begin{bmatrix}  \beta a_{t,i} \\ \alpha a_{t,i}    \end{bmatrix}.
\end{align*}  
Note that $\J_{t,i}$'s are independent and identically distributed to $\J$ with $\Ex{\J_{t,i}} = \A$. Hence, by linearity of expectation $\Ex{\J_t^{mb}} = \A $.  Now we can define the operators specific to mini-batch oracles. Note that $\opTil$ stays the same. 
\begin{align}
    \opTmb \mycirc \Theta = \E\left[ \J_{t}^{mb} \Theta (\J_{t}^{mb})^\top \right] \qquad \opMmb \mycirc \Theta = \Ex{ \left( \J_{t}^{mb} - \A\right) \Theta \left(\J_{t}^{mb} - \A\right)^\top }. 
\end{align}
Using the fact that $\J_{t}^{mb} - \A = \frac{1}{b} \sum_{i=1}^{b} \left( \J_{t,i} - \A \right)$ and $\J_{t,i} - \A$'s are zero mean i.i.d random matrices, it is evident that  $\opMmb \mycirc \Theta = \frac{1}{b} \opM \mycirc \Theta $. Hence, 
\begin{align*}
    \opMmb = \frac{1}{b} \opM.
\end{align*}
Using this fact we give a version of Lemma~\ref{lemma:alnost_eigen} for mini-batch with different step size constraints. Define $\opMmb^\top$ along the same line as $\opM^\top$.

    \begin{lemma} \label{lemma:alnost_eigen_mini_batch}
        For step sizes satisfying $0 < \alpha , \beta < \frac{1}{L}$ and  $( \alpha + 2 \beta )R^2 \leq b, \alpha \leq \frac{\beta b}{2\ki} $ , the following properties hold on the inverse and eigen values of operators $ \opTil , \opMmb $
\begin{enumerate}[label = \alph*]
    \item $\noise$ is an almost eigen vector of   $ \opMmb \mycirc ( 1 - \opTil)^{-1}  $ with an eigen value less than 1,
    \begin{align*} 
        \opMmb \mycirc ( 1 - \opTil)^{-1} \mycirc \noise  &\preccurlyeq \frac{2}{3} \noise.
    \end{align*}
    \item $\coff$ is an almost eigen vector of   $ \left( \opMmb^{\top} \mycirc  \left( \Id - \opTil^{\top} \right)^{-1} \right)  $ with an eigen value less than 1,
        \begin{align*}
            \left( \opMmb^{\top} \mycirc  \left( \Id - \opTil^{\top} \right)^{-1} \right) \mycirc \coff \preccurlyeq \frac{2}{3} \coff.
        \end{align*} 
\end{enumerate}
    \end{lemma}
    \begin{proof}
    
    Note that $0 < \alpha , \beta < \frac{1}{L}$ is necessary for $( 1 - \opTil)^{-1}$ to exist. The rescaling of the other condition on step-size is due to the fact that $\opMmb = \frac{1}{b} \opM$. Following the Lemmas~\ref{lem:opm-optil-inv},~\ref{lem:opM:opTil:transpose} with this new operator will give the required condition on the step-sizes.
    \end{proof}
    
    For Theorem~\ref{thm:main} with mini-batch oracles, we can follow the original proof of Theorem~\ref{thm:main} with stochastic oracle with this new Lemma~\ref{lemma:alnost_eigen_mini_batch} for the operator $\opMmb$. 

\section{Proof of the main results} \label{app:main-proofs}
\subsection{Proof of Theorem~\ref{thm:main}} \label{app:proof:theorem:main}
The proof involves three parts. In the first part, we consider the bias recursion and bound the excess risk in the bias process. In the second, we bound the excess error in the variance process. In the last part, we use the bias-variance decomposition and the relation between covariance of $\thet{T}$ and excess error of $\avgxt{T}$ from \eqref{eq:excessrisk_in_thet}.  

\paragraph{Bias Error.}
For the bias part, we show a relation between the finite sum of covariance of the iterates in case of bias process~\eqref{eq:rec-bias} with stochastic gradients and the finite sum of covariance of iterates  of bias process with exact gradients in Lemma~\ref{lem:bias-opT-opTil}. Using the fact that $\coff$ is almost a eigenvector of $\opM^{\top} \mycirc  \left( \Id - \opTil^{\top} \right)$ (see Lemma~\ref{lemma:alnost_eigen}) and using the Nesterov Lyapunov techniques to control the sum of the covariance of the iterates of bias process with exact gradients  (Lemma~\ref{lem:nesterov-potential-H}),  we get the sum of excess risk of the bias iterates  
(see Lemma~\ref{lem:bias_sum_covariance}). From here we have, 
    \begin{align*}
        \sum_{t=0}^{T} \scal{\HZero}{\Ex{ \thetb{t} \otimes \thetb{t} }}  \leq \min\left\lbrace \frac{3(T+1)}{\alpha} , \frac{12(T+1)(T+2)}{\beta} \right\rbrace \nor{\xt{0} - \xt{*}}^2.
    \end{align*}
     We use the property that $\scal{\HZero}{ \bm{\theta} \otimes  \bm{\theta} }$ is convex in $\bm{\theta}$. As  $\athetb{T} \defeq \sum_{t=0}^{T} \thetb{T}$, applying Jensens inequality,
    \begin{align*}
         \scal{\HZero}{ \athetb{T} \otimes  \athetb{T} } &\leq \left( T+1 \right)  \sum_{t=1}^{T} \scal{\HZero}{ \thetb{t} \otimes  \thetb{t} }, \\
        &\leq \min\left\lbrace \frac{3(T+1)^2}{\alpha} , \frac{12(T+1)^2(T+2)}{\beta} \right\rbrace \nor{\xt{0} - \xt{*}}^2.
    \end{align*}
    \paragraph{Variance Error} We expand the expected covariance of $\athetv{T}$ in Lemma~\ref{lem:var:cov-athet} such that the coefficients of $\thetv{t} \otimes \thetv{t}$ in the formulation are positive and any upper bound on the covariance of iterates $\thetv{t}$, for $t \leq T$  would give an upper bound on the expected covariance of $\athetv{T}$. Then we bound the limiting covariance of the iterates, i.e., $\Ex{\thetv{t} \otimes \thetv{t}}/t^2$ in Lemma~\ref{lem:covariance-t}. The fact that $\noise$ is almost eigen vector of $ \opM  \mycirc \left( 1 - \opTil \right)^{-1} $ is used here. Using this upper bound in the above formulation of covariance of $\athetv{T}$ to give the bound in Lemma~\ref{lem:var-main}. From here, we have
    \begin{align*}
        \scal{\HZero}{\E \left[ \athetv{T} \otimes \athetv{T} \right]} &\leq 18 \left( \sigma^2 d \right) T^3.
    \end{align*}
    Now using the bias-variance decomposition, we get, 
    \begin{align*}
        \scal{\HZero}{\Ex{\athet{T} \otimes \athet{T}}} \leq 2 \scal{\HZero}{\Ex{\athetb{T} \otimes \athetb {T}}} + 2 \scal{\HZero}{\Ex{\athetv{T} \otimes \athetv{T}}}.
    \end{align*}
     and using the formulation of excess risk of $\avgxt{T}$ with  covariance of $\athet{T}$ from \eqref{eq:excessrisk_in_thet}, 
     \begin{align*}
    \Ex{\risk{\avgxt{T}}} - \risk{\bm{\xt{*}}} = \frac{1}{2}  \left( \sum_{t=0}^{T} (t+1) \right)^{-2} \scal{\HZero}{\Ex{\athet{T} \otimes \athet{T}}}.
     \end{align*}
     Combining these will prove Theorem~\ref{thm:main}.

\subsection{Proof of Theorem~\ref{thm:last-iterate}}\label{app:proof:theorem:lastiterate}
For the last iterate too, we employ the bias-variance decomposition. First the variance part, we use Lemma~\ref{lem:var-t-main}. Note that if Assumption~\ref{asmp:uniform-kurtosis} holds with constant $\kurt$ then Assumption~\ref{asmp:fourth-moment} holds with $R^2 = \kurt \tr{\cov}$. and Assumption~\ref{asmp:stat-condition-number} holds with $\ki = \kurt d$. Hence this satisfies the condition on step size required for  Lemma~\ref{lem:var-t-main}.

\begin{align*}
    \Ex{ \thetv{t} \otimes \thetv{t} } \preccurlyeq t^2 \sigma^2  \begin{bmatrix}
        2 \alpha  ( \beta \cov )^{-1}  +  (2\beta -3 \alpha) \I &  \alpha \beta^{-1}  (2\beta  - \alpha) \I   \\
        \alpha \beta^{-1}  (2\beta  - \alpha) \I  &  2  \alpha^2 \beta^{-1} \I 
      \end{bmatrix}, \\
     \scal{\HZero}{\Ex{ \thetv{t} \otimes \thetv{t} }} \leq 2 t^2 \sigma^2 \left( (\alpha + 2 \beta ) \tr{\cov} + \frac{2\alpha \dimn}{\beta} \right).
\end{align*}
For the bias, we require uniform kurtosis ,i.e., Assumption~\ref{asmp:uniform-kurtosis}. Under this assumption, one can related the variance due to stochastic oracle to the excess risk of the iterate (see Lemma~\ref{lem:kurtosis-ut}). Using this we give a closed recursion for the excess risk of the last iterate as a discrete Volterra integral of  risk of the previous iterates in Lemma~\ref{lem:last-iterate=main}. Using simple induction to bound this (note that scaling on step-sizes will be used here) will give, 
    \begin{align*}
       \scal{\HZero}{\thetb{t} \otimes \thetb{t}} &\leqslant \min\left\lbrace \frac{3}{\alpha} , \frac{24 (t+1)}{\beta} \right\rbrace \nor{\xt{0} - \xt{*}}^2
    \end{align*}
Using the bias-variance decomposition along with the fact that 
\begin{align*}
    \scal{\HZero}{\Ex{\thetb{t} \otimes \thetb{t}}} = 2 (t+1)^2 \cdot \left( \excessRisk{\xt{t}} - \risk{\xt{*}} \right)
\end{align*}
This proves Theorem~\ref{thm:last-iterate}.

\subsection{Lower Bound} \label{app:proof:lemma:lower}


\begin{lemma}
   For all starting point $\xt{0}$, there exists a distribution  $\rho$ over $ \mathbb{R}^{d}\times  \mathbb{R} $ satisfying Assumption~\ref{asmp:fourth-moment} with  $R^2 = 1$, Assumption~\ref{asmp:noise-covariance} with $\sigma=0$,  Assumption~\ref{asmp:stat-condition-number} with $\ki=d$ and optimum $\xtp{*}$ verifying $\nor{\xt{*}^{\prime} - \xt{0}}^2 = 1$, such that  the expected excess risk of any stochastic first order algorithm is lower bounded as 
       \begin{align*}
        \excessRisk{\xt{\lfloor d/2 \rfloor}}  = \Omega\left(\frac{1}{d} \right).
    \end{align*} 
\end{lemma}

\begin{proof}
Let $\left( e_i \right)_{i=1}^{d}$ be a set of orthonormal basis. Define the following 
\begin{itemize}[leftmargin=*]
    \item The optimum \begin{align*}
    \xtp{*} \defeq \xt{0} + \frac{1}{\sqrt{d}} \sum_{i=1}^{d} e_i. 
\end{align*}
It can be easily verified that $\nor{\xt{*}^{\prime} - \xt{0}}^2 = 1$.
    \item  The feature distribution $\rho$ where each $e_i$ is sampled with a probability $1/d$. In this case the Hessian $\cov^{\prime} = \left( d \I \right)^{-1} $. Note that for this distribution $R^2 = 1$. The excess risk at any $x$ is as follows
    \begin{align*}
        \risk{x} &= \frac{1}{2} \left( \bm{x} - \xtp{*} \right)^\top \cov^{\prime} \left( \bm{x} - \xtp{*} \right), \\
        &= \frac{1}{2d} \sum_{i=1}^{d} \left( \scal{\bm{x} - \xtp{*} }{e_i} \right)^2 =   \frac{1}{2d} \sum_{i=1}^{d} \left( \scal{\bm{x} - \xt{0} }{e_i} + \scal{ \xt{0} - \xtp{*} }{e_i} \right)^2, \\
        &=  \frac{1}{2d} \sum_{i=1}^{d} \left( \scal{\bm{x} - \xt{0} }{e_i} + \scal{ \xt{0} - \xtp{*} }{e_i} \right)^2, \\
        &= \frac{1}{2d} \sum_{i=1}^{d} \left( \scal{\bm{x} - \xt{0} }{e_i} - \frac{1}{\sqrt{d}}\right)^2, \qquad \text{using construction of } \xtp{*}. 
    \end{align*}
    \item For $n \geq 1$, $b_n = \scal{a_n}{\xt{*}}$ where $a_n \sim \rho$. Hence Assumption~\ref{asmp:noise-covariance} holds with $\sigma = 0$. From the construction it can be seen that  $\ki = d$.
\end{itemize}
Consider any stochastic first order algorithm $\mathcal{S}$ for $t$ iterations. Lets say $a_1 = e_{i_1}, \cdots, a_t = e_{i_t} $ be the inputs from the stream till time $t$.  From Definition~\ref{def:sfo}, the estimator $\xt{t}$ after $t$ iterations 
    \begin{align*}
    \xt{t} \in \xt{0} + \mathit{span}\left\lbrace  \nabla_{0} f(\xt{0}), \nabla_{1} f(\xt{1}), \cdots, \nabla_{k-1} f(\xt{k-1})   \right\rbrace. 
    \end{align*}
    Note that for above defined noiseless linear regression the stochastic gradient at time $k$ is $\nabla_{k} f(\xt{k}) = e_{i_k} \scal{e_{i_k}}{\xt{k} - \xtp{*}} $. Using the fact that it is always along the direction of $ e_{i_k} $. 
    \begin{align*}
    e \defeq \xt{d/2} - \xt{0} \in \mathit{span}\left\lbrace  e_{i_1}, \cdots, e_{i_{d/2}}   \right\rbrace. 
    \end{align*}
    Plugging this in the above expression for excess risk, we get, 
    \begin{align*}
        \E \risk{\xt{t}} &=  \frac{1}{2d} \sum_{i=1}^{d}  \E \left( \scal{e}{e_i} - \frac{1}{\sqrt{d}} \right)^2. 
    \end{align*}
    From the construction of $\rho^{\prime}$, $e$ is in the span of $d/2$ orthogonal features. Hence, the remaining $d/2$ directions contribute to the excess error. In technical terms, let $\mathcal{P}$ be the set $\left\lbrace  e_{i_1}, \ldots, e_{i_{d/2}}   \right\rbrace$. Note that $|\mathcal{P}| = d/2$. Then, 
    \begin{align*}
        \E \risk{\xt{t}} &=  \frac{1}{2d} \sum_{i=1}^{d}  \E \left( \scal{e}{e_i} - \frac{1}{\sqrt{d}} \right)^2 , \\
        &\geq \frac{1}{2d} \sum_{e \not\in \mathcal{P}} \frac{1}{d} = \frac{d - |\mathcal{P}|}{2d^2} = \frac{1}{4d} .
    \end{align*}
\end{proof}

\begin{lemma}
\label{lem:no-acceleration}
    For any initial point $\xt{0}$ and Hessian $ \cov^{\prime} $  with $\tr{\left(\cov^{\prime}\right)}=1$, there exists a distribution $\rho^{\prime}$ which prevent acceleration.
\end{lemma}

\begin{proof} Let $$ \cov^{\prime} = \sum_{i=1}^{d} p_i e_i e_i^{\top}. $$ The excess risk on any noise less problem with $\xtp{*}$ as optimum and $\cov^{\prime}$ as Hessian can be written as,
  \begin{align*}
        \risk{\bm{x}} &= \frac{1}{2} \left( \bm{x} - \xtp{*} \right)^\top \cov^{\prime} \left( \bm{x} - \xtp{*} \right), \\
        &= \frac{1}{2} \sum_{i=1}^{d} p_i \left( \scal{\bm{x} - \xtp{*} }{e_i} \right)^2 =   \frac{1}{2d} \sum_{i=1}^{d} \left( \scal{\bm{x} - \xt{0} }{e_i} + \scal{ \xt{0} - \xtp{*} }{e_i} \right)^2, \\
        &=  \frac{1}{2} \sum_{i=1}^{d} p_i \left( \scal{\bm{x} - \xt{0} }{e_i} + \scal{ \xt{0} - \xtp{*} }{e_i} \right)^2. \\
    \end{align*}
Let $\rho^{\prime}$ be the one hot basis distribution where $e_i$ is sampled with probability $ p_i $. Consider any stochastic first order algorithm $\mathcal{S}$ for $t$ iterations. Lets say $e_{i_1}, \cdots, e_{i_t} $ be the inputs from the stream till time $t$.  From Definition~\ref{def:sfo}, the estimator $\xt{t}$ after $t$ iterations 
    \begin{align*}
    \xt{t} \in \xt{0} + \mathit{span}\left\lbrace  \nabla_{0} f(\xt{0}), \nabla_{1} f(\xt{1}), \cdots, \nabla_{k-1} f(\xt{k-1})   \right\rbrace. 
    \end{align*}
    Note that for noiseless regression the stochastic gradient at time $k$ is $\nabla_{k} f(\xt{k}) = e_{i_k} ( \scal{e_{i_k}}{\xt{k} - \xt{*}}  $. Using the fact that it is always along the direction of $ e_{i_k} $. 
    \begin{align*}
    e \defeq \xt{t} - \xt{0} \in \mathit{span}\left\lbrace  e_{i_1}, \cdots, e_{i_t}   \right\rbrace. 
    \end{align*}
    Plugging this in the above expression for excess risk, we get, 
      \begin{align*}
        \E \risk{\xt{t}} &=  \frac{1}{2} \sum_{i=1}^{d} p_i \E \left( \scal{e}{e_i} + \scal{ \xt{0} - \xtp{*} }{e_i} \right)^2. 
    \end{align*}
If none of the $e_{i_k}$'s , for $k \leq t$ are $e_i$ then  $\scal{e}{e_i} = 0$. This event occurs with a probability $\left( 1 - p_i\right)^{t}$. Hence, with probability $\left( 1 - p_i\right)^{t}$, $\scal{e}{e_i} = 0$. Taking this into consideration, 
 \begin{align*}
     \E \left( \scal{e}{e_i} + \scal{ \xt{0} - \xtp{*} }{e_i} \right)^2  \geq \left( 1 - p_i\right)^{t} \scal{ \xt{0} - \xtp{*} }{e_i}^2
 \end{align*}
 Hence,
 \begin{align*}
     \E \risk{\xt{t}} &\geq  \frac{1}{2} \sum_{i=1}^{d} p_i \left( 1 - p_i\right)^{t} \scal{ \xt{0} - \xtp{*} }{e_i}^2
 \end{align*}
 Noting that the right hand side corresponds to the performance of gradient descent with step size $1$ after $t/2$ iterations. In conclusion, performance of gradient descent is better than any stochastic algorithm. Hence, direct acceleration with this oracle defined by $\rho^{\prime}$ is not possible. 
 \end{proof}

 \paragraph{Space Complexity.} 
 \begin{figure}[tbh]
    \centering
    \begin{minipage}[c]{.5\linewidth}
    \includegraphics[width=\linewidth]{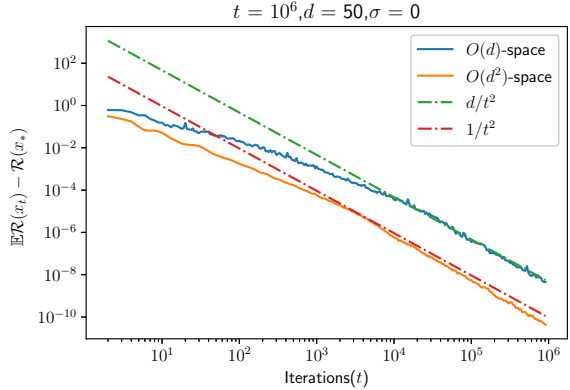}
       \end{minipage}
      \caption{Least-squares regression. Comparison of Alg.~\ref{alg:nesterov} with space constrains. Note that the version with $O(d^2)$-space decrease at rate $1/t^2$ where Alg.~\ref{alg:nesterov} with $O(d)$ decays at rate $d/t^2$. }
         \label{fig:space-vs-rate}
\end{figure}

In Figure~\ref{fig:space-vs-rate}, we demonstrate that the with additional space the speed of the decay can be improved. Note that the version with $O(d^2)$-space decrease at rate $1/t^2$ where Algorithm~\ref{alg:nesterov} with $O(d)$ decays at rate $d/t^2$. The set up for this experiment is same as the setup of the plot on Last Iterate described in Section~\ref{section:exp}. The $O(d)-curve$ corresponds to the Algorithm~\ref{alg:nesterov} with SGD oracle \eqref{eq:SGD-oracle} with step sizes $\alpha = \nicefrac{1}{3~d\tr{\cov}}, \beta = \nicefrac{1}{3\tr{\cov}}$ where $\cov$ is the covariance of Gaussian data. The $O(d^2)-curve$ corresponds to the Algorithm~\ref{alg:nesterov} with running average SGD oracle in Section~\ref{sec:lower-bound} with step sizes $\alpha = \nicefrac{1}{3\tr{\cov}}, \beta = \nicefrac{1}{3\tr{\cov}}$.

\section{Bias and Variance} \label{app:bias-variance-recursions}

In this section, we investigate the recursions of expected covariance of the bias and variance process. 
\subsection{Our technique for Bias} \label{app:our-technique-for-bias}
The work of \citet{jain2018accelerating} introduces a novel Lyapunov function $ c_1 \Ex{\nor{\yt{t} - \xt{*}}^{2}} + c_2 \Ex{\nor{\zt{t} - \xt{*}}^2_{\cov^{-1}}} $, for some constants $c_1,c_2$ for the analysis of bias error to show accelerated SGD rates for strongly convex least squares. Using similar Lyapunov function on the non-strongly convex version of Nesterov acceleration algorithm, in \citet{even2021continuized},  a rate of convergence for $\Ex{\nor{\xt{t} - \xt{*}}^2}$ ,i.e, $$ \Ex{\nor{\xt{t} - \xt{*}}^2}  \lesssim \frac{\nor{\xt{0} - \xt{*}}^2_{\cov^{-1}}}{t^2}. $$ is shown. Even with this result, it is still unclear how to relate excess error $\Ex{\nor{\xt{t} - \xt{*}}^2_{\cov}}$ and distance of initialization $\nor{\xt{0} - \xt{*}}^2$. Note that for non-strongly convex functions $ \nor{\xt{0} - \xt{*}}^2_{\cov^{-1}} $ can be arbitrarily large in comparison to $\nor{\xt{0} - \xt{*}}^2$. As there is an absence of direct Lyapunov techniques for bias error, it is needed to introduce a new method.

In recent times, many works \citep{zou2021benign,varre2021iterate} have studied the sharp characterization of bias process in SGD to understand the performance of SGD for over-parameterized least squares. In \cite{zou2021benign}, it is shown that  sum of covariance i.e. $\sum_{i=0}^{t} \Ex{ \thetb{i} \otimes \thetb{i} }$ of SGD at the limit $t \to \infty$ is used to give sharp bounds for bias excess risk. Even this approach cannot be used to in our case for two reasons (a) the limit in the case of our accelerated methods still depends on $ \nor{\xt{0} - \xt{*}}^2_{\cov^{-1}} $ which can be arbitrarily large (b) this requires more restricting uniformly bounded kurtosis assumption. In our approach we give sharp estimates for finite sum of covariance and relate them to the sum of covariance for the  Algorithm~\ref{alg:nesterov} with exact gradients (see Lemma~\ref{lem:bias-opT-opTil}). This method gives us the  bounds on the excess risk of bias part. Also, note that our approach does not require the assumption of bounded uniform kurtosis and works with standard fourth moment Assumption~\ref{asmp:fourth-moment}. 

\subsection{Bias}\label{app:bias}
Recalling the recursion \eqref{eq:rec-bias}, we have \eqal{ \thetb{t+1} =  \J_{t} ~ \thetb{t} .  } For all $t \geq 0$, we use the following notation for the sake of brevity  $$  \C{t} \defeq \E\left[ \thetb{t} \otimes \thetb {t} \right]. $$ 

\begin{lemma} \label{lem:rec:last:iterate} For $ t \geq 0 $, the covariance of the bias iterates is determined by the recursion,
    \begin{align*}
        \C{t+1} &= \opTil^{~\left(t+1\right)} \mycirc \C{0} + \sum_{k = 0}^{t} \opTil^{k} \mycirc \opM \mycirc \C{t-k}
    \end{align*}
    
\end{lemma}

\begin{proof}
From the recursion on the bias covariance of Lemma~\ref{lem:covariance:recursion}, we have \begin{align*}
        \Ex{ \thetb{t+1} \otimes \thetb{t+1} } &= \opT \mycirc \Ex{ \thetb{t} \otimes \thetb{t} }, \\
        \C{t+1} &= \opT \mycirc \C{t} = \left( \opTil + \opM \right) \mycirc \C{t}, \quad \text{from Lemma~\ref{lem:operator}}, \\
        &= \opTil \mycirc \C{t} + \opM \mycirc \C{t}, \\
        &= \opTil \mycirc \left( \opTil \mycirc \C{t-1} + \opM \mycirc \C{t-1} \right) + \opM \mycirc \C{t} = \opTil^2 \mycirc \C{t-1} + \opTil \mycirc \opM \mycirc \C{t-1} + \opM \mycirc \C{t}.
    \end{align*}
Expanding this recursively we get the following expression
\begin{align*}
    \C{t+1} &= \opTil^{~\left(t+1\right)} \mycirc \C{0} + \sum_{k = 0}^{t} \opTil^{k} \mycirc \opM \mycirc \C{t-k}.
\end{align*}
\end{proof}

\begin{lemma} \label{lem:bias-opT-opTil}
  For $ T \geq 0$, 
    \begin{align}
        \left( \Id - \left( \Id - \opTil \right)^{-1} \mycirc \opM \right) \mycirc \sum_{t=0}^{T} \C{t} &\preccurlyeq \left( \sum_{t=0}^{T} \opTil^{~\left(t\right)} \right) \mycirc \C{0}.
    \end{align}
\end{lemma}

\begin{proof} 
From Lemma~\ref{lem:rec:last:iterate}, 
\begin{align*}
    \C{t} &= \opTil^{~\left(t\right)} \mycirc \C{0} + \sum_{k = 0}^{t-1} \opTil^{t-k-1} \mycirc \opM \mycirc \C{k}.
\end{align*}
Consider the following summation, 
\begin{align*}
    \sum_{t=0}^{T} \C{t} &= \sum_{t} \left[ \opTil^{~\left(t\right)} \mycirc \C{0} + \sum_{k = 0}^{t-1} \opTil^{t-1-k} \mycirc \opM \mycirc \C{k} \right], \\
     &= \left( \sum_{t=0}^{T} \opTil^{~\left(t\right)} \right)  \mycirc \C{0} +   \sum_{t=0}^{T} \sum_{k = 0}^{t-1} \opTil^{t-1-k} \mycirc \opM \mycirc \C{k}.
\end{align*}
Exchanging the summations for the second part,
\begin{align*}
\sum_{t=0}^{T} \C{t} &= \left( \sum_{t=0}^{T} \opTil^{~\left(t\right)} \right)  \mycirc \C{0} + \sum_{k=0}^{T-1} \sum_{t=k+1}^{T} \opTil^{t-1-k} \mycirc \opM \mycirc \C{k},   \\
&=  \left( \sum_{t=0}^{T} \opTil^{~\left(t\right)} \right)  \mycirc \C{0} + \sum_{k=0}^{T-1} \left[ \sum_{t=0}^{T-k-1} \opTil^{t}  \right] \mycirc \opM \mycirc \C{k}.
\end{align*}
Note that $\C{k}$ is PSD, for $k \geq 0$ and $\opM$ is positive. Hence, $\opM \mycirc \C{k}$ is PSD.  Since $\opTil \succcurlyeq 0$ (to be precise $ \opTil \mycirc  \opM \mycirc \C{k} \succcurlyeq 0 $, but we drop this for simplicity of writing), we can say the following things, 
\begin{align*}
    \forall t \geq 0, ~~ \opTil^{~t} &\succcurlyeq 0, \\
    \textrm{for any } t{'} \geq 0, ~~ \sum_{t>t{'}} \opTil^{~t} &\succcurlyeq 0, \\
    \textrm{Hence, for any } t{'} \geq 0, ~~ \sum_{t \geq 0} \opTil^{~t} &\succcurlyeq \sum_{t=0}^{t{'}} \opTil^{~t}.
\end{align*}
\begin{align*}
    \sum_{t=0}^{T} \C{t} &\preccurlyeq  \left( \sum_{t=0}^{T} \opTil^{~\left(t\right)} \right)  \mycirc \C{0} + \sum_{k=0}^{T-1} \left[ \sum_{t \geq 0} \opTil^{~t} \right] \mycirc \opM \mycirc \C{k}.
\end{align*}
Here we use the fact that $$ \sum_{t \geq 0} \opTil^{~t} = \left( \Id - \opTil \right)^{-1}. \ $$
\begin{align*}
    \sum_{t=0}^{T} \C{t} &\preccurlyeq  \left( \sum_{t=0}^{T} \opTil^{~\left(t\right)} \right)  \mycirc \C{0} + \sum_{k=0}^{T-1} \left( \Id - \opTil \right)^{-1} \mycirc \opM \mycirc \C{k}, \\
    &\preccurlyeq \left( \sum_{t=0}^{T} \opTil^{~\left(t\right)} \right)  \mycirc \C{0} +  \left( \Id - \opTil \right)^{-1} \mycirc \opM \mycirc \sum_{k=0}^{T-1} \C{k}.
\end{align*}
Using the fact that $ \C{T} \succcurlyeq 0 $, we have $$ \sum_{k=0}^{T-1} \C{k} \leq \sum_{k=0}^{T} \C{k}   $$
Hence,
\begin{align*}
    \sum_{t=0}^{T} \C{t} &\preccurlyeq \left( \sum_{t=0}^{T} \opTil^{~\left(t\right)} \right)  \mycirc \C{0} +  \left( \Id - \opTil \right)^{-1} \mycirc \opM \mycirc \sum_{k=0}^{T} \C{k}. \\
\end{align*}
From this we can prove the lemma 
\begin{align*}
    \left( \Id - \left( \Id - \opTil \right)^{-1} \mycirc \opM \right) \mycirc \sum_{t=0}^{T} \C{t} &\preccurlyeq \left( \sum_{t=0}^{T} \opTil^{~\left(t\right)} \right) \mycirc \C{0}. \\
\end{align*}
\end{proof}

\begin{lemma} \label{lem:bias_sum_covariance}
    With the stepsizes satisfying $( \alpha + 2 \beta ) R^2 \leq 1 ,  \alpha \leq  \frac{\beta}{2 \ki}$ the sum of covariance can be bounded by  
    \begin{align*}
        \sum_{t=0}^{T} \scal{\HZero}{\Ex{ \thetb{t} \otimes \thetb{t} }}  \leq \min\left\lbrace \frac{3(T+1)}{\alpha} , \frac{12(T+1)(T+2)}{\beta} \right\rbrace \nor{\xt{0} - \xt{*}}^2.
    \end{align*}
\end{lemma}
\begin{proof}
From Lemma~\ref{lem:bias-opT-opTil},
    \begin{align*}
        \left( \Id - \left( \Id - \opTil \right)^{-1} \mycirc \opM \right) \mycirc \sum_{t=0}^{T} \C{t} &\preccurlyeq \left( \sum_{t=0}^{T} \opTil^{~\left(t\right)} \right) \mycirc \C{0}, \\
        \scal{\HZero}{ \left( \Id - \left( \Id - \opTil \right)^{-1} \mycirc \opM \right) \mycirc \sum_{t=0}^{T} \C{t}} &\leq \scal{\HZero}{\left( \sum_{t=0}^{T} \opTil^{~\left(t\right)} \right) \mycirc \C{0}}.
    \end{align*}
Using the definition of transpose of operators in Remark~\ref{eq:transpose_operator}, we get
\begin{align}
    \scal{\HZero}{ \left( \Id - \left( \Id - \opTil \right)^{-1} \mycirc \opM \right) \mycirc \sum_{t=0}^{T} \C{t}} = \scal{ \left(\Id -   \opM^{\top} \mycirc  \left( \Id - \opTil^{\top} \right)^{-1} \right) \mycirc \coff}{\sum_{t=0}^{T} \C{t}}
\end{align}
As the condition on the stepsize is satisfied, we can use the fact that $\coff$ is almost eigen vector of $ \opM^{\top} \mycirc  \left( \Id - \opTil^{\top}  \right)^{-1}  $ from Lemmas~\ref{lemma:alnost_eigen_mini_batch},~\ref{lem:opM:opTil:transpose},
\begin{align*}
    \opM^{\top} \mycirc  \left( \Id - \opTil^{\top} \right)^{-1} \mycirc \HZero &\preccurlyeq \frac{2}{3} \HZero \\
    \left( \Id - \opM^{\top} \mycirc  \left( \Id - \opTil^{\top} \right)^{-1}  \right)\mycirc \HZero &\succcurlyeq \frac{1}{3} \HZero.
\end{align*}
Combining them we get,
\begin{align*}
    \frac{1}{3} \scal{\HZero}{\sum_{t=0}^{T} \C{t}}  &\leq  \scal{\HZero}{\left( \sum_{t=0}^{T} \opTil^{~\left(t\right)} \right) \mycirc \C{0}}, \\
    \sum_{t=0}^{T} \scal{\HZero}{ \C{t}}  &\leq 3  \sum_{t=0}^{T} \scal{\HZero}{\left( \opTil^{~\left(t\right)} \right) \mycirc \C{0}}. 
\end{align*} 
Note that $\thetb{0} = \thet{0}$ gives $\C{0} = \thet{0} \otimes \thet{0} $ . From Lemma~\ref{lem:nesterov-potential-H}, for $ 0 \leqslant t \leqslant T$,  
    \begin{align*}
        \scal{\HZero}{\opTil^{~\left(t\right)} \mycirc \ \left[ \thet{0} \otimes \thet{0} \right] } \leq \min\left\lbrace \frac{1}{\alpha} , \frac{8 (t+1)}{\beta} \right\rbrace \nor{\xt{0} - \xt{*}}^2.
    \end{align*}
    Summing this for $0 \leq t \leq T $ proves the lemma.
\end{proof}

\subsection{Bias Last Iterate}\label{sec:biaslastiterate}

\begin{lemma} [Final Iterate Risk] \label{lem:final:iterate:app}
Under Assumptions \ref{asmp:uniform-kurtosis} and the step-sizes satisfying $\alpha\leq \beta\leq 1/L $. For $T \geq 0$, the last iterate excess error can be determined by the following discrete Volterra integral
    \begin{align*}
    \ft{\thetb{T}} &\leqslant \min\left\lbrace \frac{1}{\alpha} , \frac{8 (T+1)}{\beta} \right\rbrace \nor{\xt{0} - \xt{*}}^2 + \sum_{t=0}^{T-1}  \gt{\cov}{t-k-1} \ft{\thetb{t}}, 
    \end{align*}
where $\ft{\thetb{t}} \defeq \scal{\HZero}{\Ex{\thetb{t} \otimes \thetb{t}}}$ and the  kernel
\begin{align} \label{eq:kern}
            \gt{\cov}{t} \defeq  \kappa \scal{\coff}{ \opTil^{t} \mycirc \noise}.
\end{align}
\end{lemma}
\begin{proof}
            Invoking Lemma~\ref{lem:rec:last:iterate} and $\thetb{0} = \thet{0}$ gives \begin{align*}
            \Ex{ \thetb{t} \otimes \thetb{t} } &= \opTil^{~\left(t\right)} \mycirc \left[ \thet{0} \otimes \thet{0} \right] + \sum_{k = 0}^{t-1} \opTil^{t-k-1} \mycirc \opM \mycirc \Ex{ \thetb{k} \otimes \thetb{k} }.
        \end{align*}
    Using Lemma~\ref{lem:kurtosis-ut}, note $\coff$ is defined at \eqref{eq:noise_and_coeff}
        \begin{align*}
            \opM \mycirc \Ex{ \thetb{k} \otimes \thetb{k} } &\preccurlyeq \kurt \scal{\coff}{\Ex{\thetb{k} \otimes \thetb{k}}} \noise.
        \end{align*}
    Using this and fact that $\opT$ is positive and $\noise$ is PSD, 
        \begin{align*}
            \Ex{ \thetb{t} \otimes \thetb{t} } &\preccurlyeq \opTil^{~\left(t\right)} \mycirc \left[ \thet{0} \otimes \thet{0} \right] + \sum_{k = 0}^{t-1} \opTil^{t-k-1} \mycirc \kurt \scal{\coff}{\Ex{\thetb{k} \otimes \thetb{k}}} \noise.
        \end{align*}
    Taking the scalar product with $\coff$ on both sides, gives us
           \begin{align*}
            \scal{\coff}{\Ex{ \thetb{t} \otimes \thetb{t} }} \leq \scal{\coff}{\opTil^{~\left(t\right)} \mycirc \left[ \thet{0} \otimes \thet{0} \right]} + \sum_{k = 0}^{t-1} \scal{\coff}{ \opTil^{t-k-1} \mycirc \noise} \kurt \scal{\coff}{\Ex{\thetb{k} \otimes \thetb{k}}}.
        \end{align*}
    From Lemma~\ref{lem:nesterov-potential-H}, we get,
        \begin{align*}
            \scal{\coff}{\Ex{ \thetb{t} \otimes \thetb{t} }} \leq \min\left\lbrace \frac{1}{\alpha} , \frac{8 (t+1)}{\beta} \right\rbrace \nor{\xt{0} - \xt{*}}^2   + \sum_{k = 0}^{t-1} \scal{\coff}{ \opTil^{t-k-1} \mycirc \noise} \kurt \scal{\coff}{\Ex{\thetb{k} \otimes \thetb{k}}}.
        \end{align*}
    The definition of $\ft{\thetb{t}}$ proves the lemma.
\end{proof}


\begin{lemma} \label{lem:last-iterate=main}    With
        $( \alpha + 2 \beta )  \leq \frac{1}{\kurt \tr{\cov} } ,  \alpha \leq  \frac{\beta}{2 \kurt d }$, 
     after $t$ iterations of Algorithm \ref{alg:nesterov} the bias excess error, 
    \begin{align*}
       \scal{\HZero}{\thetb{t} \otimes \thetb{t}} &\leqslant \min\left\lbrace \frac{3}{\alpha} , \frac{24 (t+1)}{\beta} \right\rbrace \nor{\xt{0} - \xt{*}}^2.
    \end{align*}
\end{lemma}

\begin{proof}
   From Lemma \ref{lem:final:iterate:app},
           \begin{align*}
            \scal{\coff}{\Ex{ \thetb{t} \otimes \thetb{t} }} \leq \min\left\lbrace \frac{1}{\alpha} , \frac{8 (t+1)}{\beta} \right\rbrace \nor{\xt{0} - \xt{*}}^2   + \sum_{k = 0}^{t-1} \scal{\coff}{ \opTil^{t-k-1} \mycirc \noise} \kurt \scal{\coff}{\Ex{\thetb{k} \otimes \thetb{k}}}.
        \end{align*}
    Now we will use induction to show that $ \scal{\coff}{\Ex{\thetb{k} \otimes \thetb{k}}}  $ is bounded. 
    \paragraph{Induction Hypothesis} There exists a constant $C$, for all $ 0 \leqslant k \leqslant t-1 $ , $ \scal{\coff}{\Ex{\thetb{k} \otimes \thetb{k}}} \leq C $. Using this, \begin{align*}
           \scal{\coff}{\Ex{ \thetb{t} \otimes \thetb{t} }} &\leq \min\left\lbrace \frac{1}{\alpha} , \frac{8 (t+1)}{\beta} \right\rbrace \nor{\xt{0} - \xt{*}}^2   + \sum_{k = 0}^{t-1} \scal{\coff}{ \opTil^{t-k-1} \mycirc \noise} \kurt C, \\
            &=  \min\left\lbrace \frac{1}{\alpha} , \frac{8 (t+1)}{\beta} \right\rbrace \nor{\xt{0} - \xt{*}}^2   +  \scal{\coff}{ \sum_{k = 0}^{t-1} \opTil^{t-k-1} \mycirc \noise} \kurt C. 
        \end{align*}
        As $\opT$ is positive and $ \coff, \noise $ is PSD, 
        \begin{align*}
            \left( \sum_{k = 0}^{t-1} \opTil^{k} \right) \mycirc \noise &\preccurlyeq \left( \sum_{k = 0}^{\infty} \opTil^{k} \right) \mycirc \noise = \left( 1 - \opTil \right)^{\scriptscriptstyle -1} \mycirc \noise
        \end{align*}
        Using this upperbound, 
         \begin{align} \label{eq:last:iterate:p1}
            \scal{\coff}{\Ex{ \thetb{t} \otimes \thetb{t} }} &\leqslant \min\left\lbrace \frac{1}{\alpha} , \frac{8 (t+1)}{\beta} \right\rbrace \nor{\xt{0} - \xt{*}}^2 + \kurt C \scal{ \coff }{ \left( 1 - \opTil \right)^{\scriptscriptstyle -1} \mycirc \noise  }.
        \end{align}
        As the step sizes are chosen accordingly, using Lemma~\ref{eq:lem:optil},
        \begin{align*}
            \kurt C \scal{\coff}{ \left( 1 - \opTil \right)^{\scriptscriptstyle -1} \mycirc \noise  } \leq \frac{2C}{3}.
        \end{align*}
        Substituting these back in \eqref{eq:last:iterate:p1}, 
        \begin{align*} 
            \scal{\coff}{\Ex{ \thetb{t} \otimes \thetb{t} }}  &\leqslant \min\left\lbrace \frac{1}{\alpha} , \frac{8 (t+1)}{\beta} \right\rbrace \nor{\xt{0} - \xt{*}}^2 + \frac{2C}{3}.
        \end{align*}
    If we choose $C$ such that,     
        \begin{align*}
            C &= \min\left\lbrace \frac{3}{\alpha} , \frac{24 (t+1)}{\beta} \right\rbrace \nor{\xt{0} - \xt{*}}^2,  \\
            \frac{2C}{3} &= \min\left\lbrace \frac{2}{\alpha} , \frac{16 (t+1)}{\beta} \right\rbrace \nor{\xt{0} - \xt{*}}^2, \\
            \scal{\coff}{\Ex{ \thetb{t} \otimes \thetb{t} }} &\leqslant \min\left\lbrace \frac{3}{\alpha} , \frac{24 (t+1)}{\beta} \right\rbrace \nor{\xt{0} - \xt{*}}^2  = C .
        \end{align*}
    Hence we have shown that $\scal{\coff}{\Ex{ \thetb{t} \otimes \thetb{t} }} \leq C$. From induction we can say that for all $T > 0$, 
        \begin{align*}
            \scal{\coff}{\Ex{ \thetb{t} \otimes \thetb{t} }} &\leqslant C \quad
         \text{where}  \quad 
            C =\min\left\lbrace \frac{3}{\alpha} , \frac{24 (t+1)}{\beta} \right\rbrace \nor{\xt{0} - \xt{*}}^2 \nor{\xt{0} - \xt{*}}^2. 
        \end{align*}    
    \end{proof}

\subsection{Variance}\label{app:variance}

We start by extending the the definition of the random matrix $\J_{t}$, 
\begin{definition} \label{eq:def-jt}
For every $ 0 \leq i \leq j $, define the random linear operator $ \mathcal{J}\left( i , j \right) $ as follows 
\begin{align}
    \Jt{j}{i}= \prod_{k=i}^{j-1} \J_{k} \quad and \quad \Jt{i}{i} = \Id.
\end{align}
\end{definition}
Recalling the variance subproblem~\eqref{eq:rec-var} and using the above definition, 
\begin{align} \label{eq:variance_recursion} \thetv{0} = \begin{bmatrix}
    \bm{0} \\ \bm{0}
\end{bmatrix},  \qquad \thetv{t} &=  \Jt{t}{t-1} \thetv{t-1} + \epsi{t}.  \end{align}
Using this recursion for $t-1$ and expanding it, we will get the following
\begin{align*} \thetv{t} &=  \Jt{t}{t-1} \left(\Jt{t-1}{t-2} \thetv{t-2} + \epsi{t-1}  \right) + \epsi{t}.  \end{align*}
Using the definition of $\Jt{i}{j}$ Def.~\ref{eq:def-jt},
\begin{align*} \thetv{t} &=  \Jt{t}{t-2} \thetv{t-2} + \Jt{t}{t-1} \epsi{t-1} + \epsi{t}.  \end{align*}
Expanding it further for any $ 0 \leq i \leq t$, we have the following expression
\begin{align} \label{eq:theta-i} \thetv{t} &=  \Jt{t}{i} \thetv{i} +  \sum_{k=i+1}^{t} \Jt{t}{k} \epsi{k}.  \end{align}

\begin{lemma} \label{lem:var:cov-athet}
    With the recursion defined by~\eqref{eq:variance_recursion} and the expected covariance of the $\athetv{t}$ 
    \begin{align*}
        \E \left[ \athetv{T} \otimes \athetv{T} \right] &= \sum_{i=1}^{T} \left( \sum_{j \geq i}^{T} \A^{j-i} \right) \left\lbrace  \opM  \mycirc \Ex{ \thetv{i-1} \otimes \thetv{i-1} } + \Ex{ \epsi{i} \otimes \epsi{i} }  \right\rbrace \left( \sum_{j \geq i}^{T} \A^{j-i} \right)^{\top} .
    \end{align*}
    
\end{lemma}

\begin{proof} Recall that
\begin{align*}
    \athetv{T} = \sum_{t=0}^{T} \thetv{t}. 
\end{align*}
Considering the covariance of $\athetv{T}$, 
\begin{align*}
    \athetv{T} \otimes \athetv{T} &= \left(\sum_{i=0}^{T} \thetv{i}\right) \otimes \left(\sum_{j=0}^{T} \thetv{j}\right), \\
    &=  \sum_i \left( \thetv{i} \otimes \thetv{i} + \sum_{j > i} \left( \thetv{j} \otimes \thetv{i} + \thetv{i} \otimes \thetv{j} \right) \right).
\end{align*}
Taking expectation and using the linearity of expectation we get,
\begin{align*}
    \E \left[ \athetv{T} \otimes \athetv{T} \right] &= \sum_{i=1}^{T} \left[ \E \left[ \thetv{i} \otimes \thetv{i} \right] + \sum_{j > i}^{T} \left( \Ex{\thetv{j} \otimes \thetv{i}} + \Ex{ \thetv{i} \otimes \thetv{j}  } \right) \right].
\end{align*} 
Note that from~\eqref{eq:theta-i}, we can write $\thetv{j} \otimes \thetv{i}$ for $ j > i $ as follows
\begin{align*}
    \thetv{j} \otimes \thetv{i} =  \Jt{t}{i} \thetv{i} \otimes \thetv{i} +  \sum_{k=i+1}^{j} \Jt{t}{k} \epsi{k} \otimes \thetv{i}.
\end{align*}
Now taking the expectation,
\begin{align*}
    \Ex{\thetv{j} \otimes \thetv{i} } =  \Ex{ \Jt{t}{i} \thetv{i} \otimes \thetv{i} } +  \sum_{k=i+1}^{j} \Ex{ \Jt{t}{k} \epsi{k} \otimes \thetv{i} }. 
\end{align*}
For all $k > i$, $\Jt{t}{k}, \epsi{k}$ is independent of $\thetv{i}$, $\Jt{t}{k},\epsi{k}$ are also independent from their definition. We  have $\Ex{\Jt{t}{k}} = \A^{t-k}$, and $\Ex{\epsi{k}} = 0$. Using these,
\begin{align*}
    \Ex{\thetv{j} \otimes \thetv{i} } = \A^{j-i} \Ex{ \thetv{i} \otimes \thetv{i} }.
\end{align*}
With the same reasoning, for $ j \geq i$,
\begin{align*}
    \Ex{\thetv{i} \otimes \thetv{j} } =  \Ex{ \thetv{i} \otimes \thetv{i} } \left(\A^{\top}\right)^{j-i}.
\end{align*}
Substituting the above here gives,
\begin{align*}
  \E \left[ \athetv{T} \otimes \athetv{T} \right] &= \sum_{i=1}^{T} \left\lbrace \E \left[ \thetv{i} \otimes \thetv{i} \right] + \sum_{j > i}^{T} \left( \Ex{\thetv{j} \otimes \thetv{i}} + \Ex{ \thetv{i} \otimes \thetv{j}  } \right) \right\rbrace, \\
    &= \sum_{i=1}^{T} \left\lbrace \E \left[ \thetv{i} \otimes \thetv{i} \right] + \left( \sum_{j > i}^{T} \A^{j-i} \right) \Ex{ \thetv{i} \otimes \thetv{i} } + \Ex{ \thetv{i} \otimes \thetv{i} } \left( \sum_{j > i}^{T} \A^{j-i} \right)^{\top} \right \rbrace .
\end{align*}
Note that from here, a upper bound on $\Ex{ \thetv{i} \otimes \thetv{i} }$ doesnot translate to an upperbound on the $\E \left[ \athetv{T} \otimes \athetv{T} \right]$ as the matrix $\A$ is not positive unlike the case of SGD.  Using the following identity for any two matrices $S$ and a vector $\phi$, $$ \phi \otimes \phi + S \cdot \phi \otimes \phi + \phi \otimes \phi \cdot S^\top = \left( I + S\right) \cdot \phi \otimes \phi \cdot \left( I + S\right)^{\top} - S \cdot \phi \otimes \phi \cdot S^{\top}. $$

\begin{align*}
    \E \left[ \athetv{T} \otimes \athetv{T} \right] &= \sum_{i=1}^{T} \left( \Id  + \sum_{j > i}^{T} \A^{j-i} \right) \E \left[ \thetv{i} \otimes \thetv{i} \right] \left( \Id  + \sum_{j > i}^{T} \A^{j-i} \right)^{\top} \\ & \hspace*{4.5cm} - \sum_{i=1}^{T} \left( \sum_{j > i}^{T} \A^{j-i} \right) \E \left[ \thetv{i} \otimes \thetv{i} \right] \left( \sum_{j > i}^{T} \A^{j-i} \right)^{\top}. 
\end{align*}
Now the first term can be written as follows, 
\begin{align*}
    \sum_{i=1}^{T} \left( \Id  + \sum_{j > i}^{T} \A^{j-i} \right) \E \left[ \thetv{i} \otimes \thetv{i} \right] \left( \Id  + \sum_{j > i}^{T} \A^{j-i} \right)^{\top} &= \sum_{i=1}^{T} \left( \sum_{j \geq i}^{T} \A^{j-i} \right) \E \left[ \thetv{i} \otimes \thetv{i} \right] \left( \sum_{j > i}^{T} \A^{j-i} \right)^{\top}.
\end{align*}
For the second term note that at $i=T$ the summation will be $0$. So we directly consider the summation till $T-1$. 
\begin{align*}
    \sum_{i=1}^{T} \left( \sum_{j > i}^{T} \A^{j-i} \right) \E \left[ \thetv{i} \otimes \thetv{i} \right] \left( \sum_{j > i}^{T} \A^{j-i} \right)^{\top} &= \sum_{i=1}^{T-1} \left( \sum_{j > i}^{T} \A^{j-i} \right) \E \left[ \thetv{i} \otimes \thetv{i} \right] \left( \sum_{j > i}^{T} \A^{j-i} \right)^{\top}, \\
    &= \sum_{i=1}^{T-1} \left( \sum_{j \geq i+1 }^{T} \A^{j-i-1} \right) \A \E \left[ \thetv{i} \otimes \thetv{i} \right] \A^{\top} \left( \sum_{j \geq i+1 }^{T} \A^{j-i-1} \right)^{\top} .
\end{align*}
By definition of $\opTil$ and change of variable '$ i+1 \rightarrow i $' gives
\begin{align*}
    \sum_{i=1}^{T} \left( \sum_{j > i}^{T} \A^{j-i} \right) \E \left[ \thetv{i} \otimes \thetv{i} \right] \left( \sum_{j > i}^{T} \A^{j-i} \right)^{\top} &= \sum_{i=2}^{T} \left( \sum_{j \geq i }^{T} \A^{j-i} \right) \opTil \mycirc \E \left[ \thetv{i-1} \otimes \thetv{i-1} \right]  \left( \sum_{j \geq i }^{T} \A^{j-i} \right)^{\top}.
\end{align*}
Combining  both parts and noting that $\thetv{0} = 0$ we get, 
\begin{align*}
    \E \left[ \athetv{T} \otimes \athetv{T} \right] = \sum_{i=1}^{T} \left( \sum_{j \geq i}^{T} \A^{j-i} \right) \left\lbrace  \E \left[ \thetv{i} \otimes \thetv{i} \right] - \opTil \mycirc \E \left[ \thetv{i-1} \otimes \thetv{i-1} \right] \right\rbrace \left( \sum_{j \geq i}^{T} \A^{j-i} \right)^{\top}.
\end{align*}
From Lemma~\ref{lem:covariance:recursion},
\begin{align*}
    \Ex{\thetv{i} \otimes \thetv{i}}   &= \opT \mycirc \Ex{ \thetv{i-1} \otimes \thetv{i-1} } + \Ex{ \epsi{i} \otimes \epsi{i} }, \\
    \Ex{\thetv{i} \otimes \thetv{i}} &= \left( \opTil + \opM \right) \mycirc \Ex{ \thetv{i-1} \otimes \thetv{i-1} } + \Ex{ \epsi{i} \otimes \epsi{i} }, \\
    \Ex{\thetv{i} \otimes \thetv{i}} - \opTil  \mycirc \Ex{ \thetv{i-1} \otimes \thetv{i-1} } &=  \opM  \mycirc \Ex{ \thetv{i-1} \otimes \thetv{i-1} } + \Ex{ \epsi{i} \otimes \epsi{i} }. 
\end{align*}
\begin{align*}
    \E \left[ \athetv{T} \otimes \athetv{T} \right] &= \sum_{i=1}^{T} \left( \sum_{j \geq i}^{T} \A^{j-i} \right) \left\lbrace  \E \left[ \thetv{i} \otimes \thetv{i} \right] - \opTil \mycirc \E \left[ \thetv{i-1} \otimes \thetv{i-1} \right] \right\rbrace \left( \sum_{j \geq i}^{T} \A^{j-i} \right)^{\top}, \\
    &= \sum_{i=1}^{T} \left( \sum_{j \geq i}^{T} \A^{j-i} \right) \left\lbrace  \opM  \mycirc \Ex{ \thetv{i-1} \otimes \thetv{i-1} } + \Ex{ \epsi{i} \otimes \epsi{i} }  \right\rbrace \left( \sum_{j \geq i}^{T} \A^{j-i} \right)^{\top}. 
\end{align*}
This proves the lemma. 
\end{proof}

\begin{lemma} \label{lem:covariance-t}
    With the recursion defined by~\eqref{eq:variance_recursion} and  step sizes satisfying Condition~\ref{eq:asgd-step-size}, for  $ t \geq 0 $,  
    \begin{align*}
        \Ex{ \thetv{t} \otimes \thetv{t} } \preccurlyeq t^2 \sigma^2 \left( \Id - \opT \right)^{-1} \mycirc \noise.
    \end{align*}
\end{lemma}

\begin{proof} From Lemma~\ref{lem:covariance:recursion}, we have 
\begin{align*}
    \Ex{\thetv{t} \otimes \thetv{t}}   &= \opT \mycirc \Ex{ \thetv{t-1} \otimes \thetv{t-1} } + \Ex{ \epsi{t} \otimes \epsi{t} }, \\
    &= \opT^{2} \mycirc \Ex{ \thetv{t-2} \otimes \thetv{t-2} } + \opT \mycirc \Ex{ \epsi{t-1} \otimes \epsi{t-1} }  + \Ex{ \epsi{t} \otimes \epsi{t} } , \\
    &= \sum_{k=0}^{t-1} \opT^{k} \mycirc \Ex{ \epsi{t-k} \otimes \epsi{t-k} }.
\end{align*}
Recalling from the definition of $\epsi{k}$ and its covariance, 
\begin{align*}
   \epsi{k} &=  k \eta_k  \begin{bmatrix}  \beta a_t \\ \alpha a_t    \end{bmatrix}, \\
   \epsi{k} \otimes \epsi{k} &= k^2 \eta_k^2 \begin{bmatrix} \beta a_t \\ \alpha a_t    \end{bmatrix} \otimes \begin{bmatrix}  \beta a_t \\ \alpha a_t    \end{bmatrix}, \\
   &= k^2 \matalbe \otimes_{k} \left[ \eta_k^2 ~~ a_k \otimes a_k \right].
\end{align*}
where $ \otimes_{k} $ is the kronecker product. Taking the expectation, we have 
\begin{align*}
    \Ex{ \epsi{k} \otimes \epsi{k} } &= k^2 \matalbe \otimes_{k} \Ex{ \eta_k^2 ~~ a_k \otimes a_k }.
\end{align*}
From the Assumption~\ref{asmp:noise-covariance}, we have 
\begin{align*}
    \Ex{ \eta_k^2 ~~ a_k \otimes a_k } =  \Ex{ \left(b_k - \scal{\xt{*}}{a_k}\right)^2 a_k \otimes a_k } \preccurlyeq \sigma^2 \cov.
\end{align*}
Using the fact that kronecker product of two PSD matrices is a PSD and recalling $\noise$ from \eqref{eq:noise_and_coeff}, we get 
\begin{align*}
    \matalbe \otimes_{k} \left( \sigma^2 \cov  -  \Ex{ \eta_k^2 ~~ a_k \otimes a_k } \right) &\succcurlyeq 0, \\
    \matalbe \otimes_{k} \left( \Ex{ \eta_k^2 ~~ a_k \otimes a_k } \right) &\preccurlyeq
    \matalbe \otimes_{k} \left( \sigma^2 \cov \right), \\
&\preccurlyeq \sigma^2 {\begin{bmatrix}
    \beta^2 \cov & \alpha  \beta \cov  \\  \alpha  \beta \cov  &  \alpha^2 \cov
\end{bmatrix}} =  \sigma^2 \noise.
\end{align*}
Combining these we get the following, 
\begin{align*}
    \Ex{ \epsi{k} \otimes \epsi{k} } \preccurlyeq \sigma^2 k^2 \cdot \noise.
\end{align*}
Using this upper bound in the expansion of $ \Ex{ \thetv{t} \otimes \thetv{t} } $, 
\begin{align*}
    \Ex{ \thetv{t} \otimes \thetv{t} }  &= \sum_{k=0}^{t-1} \opT^{k} \mycirc \Ex{ \epsi{t-k} \otimes \epsi{t-k} }, \\
    &\preccurlyeq  \sigma^2 \sum_{k=0}^{t-1} \opT^{k} \mycirc (t-k)^2 ~ \noise.
\end{align*}
For $ 0 \leqslant k \leqslant T $, we have $(t-k)^2 \leqslant t^2$ and using the fact that $\opT, \noise $ are positive, 
\begin{align*}
    \sum_{k=0}^{t-1} \opT^{k} \mycirc (t-k)^2 ~ \noise &\preccurlyeq t^2  \sum_{k=0}^{t-1} \opT^{k} \mycirc \noise, \\
    &\preccurlyeq t^2  \sum_{k=0}^{\infty} \opT^{k} \mycirc \noise ~ = ~ t^2 \left( \Id - \opT \right)^{-1} \mycirc \noise . 
\end{align*}
Hence, we have 
\begin{align*}
    \Ex{ \thetv{t} \otimes \thetv{t} }  \preccurlyeq \sigma^2 t^2 \left( \Id - \opT \right)^{-1} \mycirc \noise  
\end{align*}
This completes the proof of the lemma. 
\end{proof}

\begin{lemma} \label{lem:var-main}
    With $\alpha$ and $\beta$ satisfying Condition~\ref{eq:asgd-step-size}, the excess error after $T$ iterations of the variance process, 
    \begin{align*}
        \scal{\HZero}{\E \left[ \athetv{T} \otimes \athetv{T} \right]} &\leq 18 \left( \sigma^2 d \right) T^3
    \end{align*}
\end{lemma}

\begin{proof}
    From Lemma~\ref{lem:var:cov-athet},
    \begin{align*}
        \E \left[ \athetv{T} \otimes \athetv{T} \right] &= \sum_{i=1}^{T} \left( \sum_{j \geq i}^{T} \A^{j-i} \right) \left\lbrace  \opM  \mycirc \Ex{ \thetv{i-1} \otimes \thetv{i-1} } + \Ex{ \epsi{i} \otimes \epsi{i} }  \right\rbrace \left( \sum_{j \geq i}^{T} \A^{j-i} \right)^{\top}. 
    \end{align*} First lets upperbound $ \opM  \mycirc \Ex{ \thetv{i-1} \otimes \thetv{i-1} } + \Ex{ \epsi{i} \otimes \epsi{i} } $. We have the following 
    \begin{itemize}[leftmargin=*]
        \item Invoking Lemma~\ref{lem:covariance-t},  $$ \Ex{ \thetv{i-1} \otimes \thetv{i-1} } \preccurlyeq  (i-1)^2 \sigma^2 \left( \Id - \opT \right)^{-1} \mycirc \noise. $$ 
        \item For the choice of stepsizes from Lemma~\ref{eq:lem:inv:opt}, $$ \left( \Id - \opT \right)^{-1} \mycirc \noise \preccurlyeq 3 \left( \Id - \opTil \right)^{-1} \mycirc \noise . $$ 
        \item Combining these to get  $$  \opM \mycirc \Ex{ \thetv{i-1} \otimes \thetv{i-1} } \preccurlyeq  3 (i-1)^2 \sigma^2 \opM \mycirc \left( \Id - \opTil \right)^{-1} \mycirc \noise. $$
        the step sizes chosen allows us to invoke Lemma~\ref{lem:opm-optil-inv}. Hence, 
        \begin{align*}
             \opM \mycirc \left( \Id - \opTil \right)^{-1} \mycirc \noise &\preccurlyeq \frac{2}{3} \noise, \\
             3 \sigma^2 (i-1)^2  \opM \mycirc \left( \Id - \opTil \right)^{-1} \mycirc \noise &\preccurlyeq 2 \sigma^2 (i-1)^2 \noise.
        \end{align*}
        \item The remaining $ \Ex{ \epsi{i} \otimes \epsi{i} } $ can be upperbounded by $\sigma^2 i^2 \noise $.
    \end{itemize}
Combining the above gives 
    \begin{align*}
        \opM  \mycirc \Ex{ \thetv{i-1} \otimes \thetv{i-1} } + \Ex{ \epsi{i} \otimes \epsi{i} }  &\preccurlyeq 2 \sigma^2 (i-1)^2 \noise + \sigma^2 i^2 \noise. \end{align*}
    For $0 \leq i \leq T$ this can be bounded as follows.
    \begin{align*}
        \opM  \mycirc \Ex{ \thetv{i-1} \otimes \thetv{i-1} } + \Ex{ \epsi{i} \otimes \epsi{i} }  &\preccurlyeq  3 \sigma^2 T^2 \noise.
    \end{align*}
 Note that this can be used in Lemma~\ref{lem:var:cov-athet}  to bound $  \E \left[ \athetv{T} \otimes \athetv{T} \right] $ because for any matrix P, $ P(.)P^{\top} $ is a positive operator. Hence   
\begin{align*}
        \left( \sum_{j \geq i}^{T} \A^{j-i} \right) \left\lbrace  \opM  \mycirc \Ex{ \thetv{i-1} \otimes \thetv{i-1} } + \Ex{ \epsi{i} \otimes \epsi{i} }  \right\rbrace \left( \sum_{j \geq i}^{T} \A^{j-i} \right)^{\top} &\preccurlyeq 
        3 \sigma^2 T^2 \left( \sum_{j \geq i}^{T} \A^{j-i} \right) \cdot  \noise  \cdot  \left( \sum_{j \geq i}^{T} \A^{j-i} \right)^{\top}.
    \end{align*}
    Adding this and using Lemma~\ref{lem:var:cov-athet},
    \begin{align*}
        \E \left[ \athetv{T} \otimes \athetv{T} \right] &\preccurlyeq 3 \sigma^2 T^2 \cdot  \sum_{i=1}^{T} \left( \sum_{j \geq i}^{T} \A^{j-i} \right) \noise  \left( \sum_{j \geq i}^{T} \A^{j-i} \right)^{\top}. 
    \end{align*}
    Using the following identity, 
    \begin{align*}
        \left( \sum_{j \geq i}^{T} \A^{j-i} \right)  &= \left( \Id - \A \right)^{-1} \left( \Id - \A^{\left(T-i+1\right)} \right),\\
        \E \left[ \athetv{T} \otimes \athetv{T} \right] &\preccurlyeq 3 \sigma^2 T^2 \cdot \left( \Id - \A \right)^{-1} \left[ \sum_{i=1}^{T} \left( \Id - \A^{\left(T-i+1\right)} \right) \noise \left( \Id - \A^{\left(T-i+1\right)} \right)^{\top} \right]    \left( \Id - \A^\top \right)^{-1}.
    \end{align*}
    Note that we are interested in $ \scal{\HZero}{\E \left[ \athetv{T} \otimes \athetv{T} \right]} .$
    \begin{align*}
        \scal{\HZero}{\E \left[ \athetv{T} \otimes \athetv{T} \right]} &\\  
        &\hspace{-2.5cm}\leq 3 \sigma^2 T^2 \scal{\HZero}{\left( \Id - \A \right)^{-1} \left[ \sum_{i=1}^{T} \left( \Id - \A^{\left(T-i+1\right)} \right) \noise \left( \Id - \A^{\left(T-i+1\right)} \right)^{\top} \right]    \left( \Id - \A^\top \right)^{-1}}, \\
        &\hspace{-2.5cm}= 3 \sigma^2 T^2 \scal{ \left( \Id - \A^\top \right)^{-1} \HZero \left( \Id - \A \right)^{-1} }{\sum_{i=1}^{T} \left( \Id - \A^{\left(T-i+1\right)} \right) \noise \left( \Id - \A^{\left(T-i+1\right)} \right)^{\top}}.
    \end{align*}
Note that \begin{align*}
        \left( \Id - \A \right)^{-1} &= \begin{bmatrix}
            \I & \left(\alpha \cov\right)^{-1} \left(\I - \beta \cov\right) \\ -\I & \left(\alpha \cov\right)^{-1} \left(\beta \cov\right) \\
        \end{bmatrix} ,\\
        \HZero \left( \Id - \A \right)^{-1} &= \HZero \begin{bmatrix}
            \I & \left(\alpha \cov\right)^{-1} \left(\I - \beta \cov\right) \\ -\I & \left(\alpha \cov\right)^{-1} \left(\beta \cov\right) \\
        \end{bmatrix} ,\\
        &= \begin{bmatrix}
            0 & \alpha^{\scriptscriptstyle -1} \I \\
            0 & \alpha^{\scriptscriptstyle -1} \I 
        \end{bmatrix}, \\
        \left( \Id - \A^\top \right)^{-1} \HZero \left( \Id - \A \right)^{-1} &= \begin{bmatrix}
            \I & -\I \\  \left(\alpha \cov\right)^{-1} \left(\I - \beta \cov\right) & \left(\alpha \cov\right)^{-1} \left(\beta \cov\right) \\
        \end{bmatrix}  \begin{bmatrix}
            0 & \alpha^{\scriptscriptstyle -1} \I \\
            0 & \alpha^{\scriptscriptstyle -1} \I 
        \end{bmatrix} =  \begin{bmatrix}
        0 & 0 \\
        0 & \alpha^{\scriptscriptstyle -2} \cov^{-1}
    \end{bmatrix}  .    \end{align*}
Substituting this, 
\begin{align*}
    \scal{\HZero}{\E \left[ \athetv{T} \otimes \athetv{T} \right]} &\leq 3 \sigma^2 T^2 \scal{ \begin{bmatrix}
        0 & 0 \\
        0 & \alpha^{\scriptscriptstyle -2} \cov^{-1}
    \end{bmatrix} }{\sum_{i=1}^{T} \left( \Id - \A^{\left(T-i+1\right)} \right) \noise \left( \Id - \A^{\left(T-i+1\right)} \right)^{\top}} ,\\
    &= 3 \sigma^2 T^2 \sum_{i=1}^{T} \scal{ \begin{bmatrix}
        0 & 0 \\
        0 & \alpha^{\scriptscriptstyle -2} \cov^{-1}
    \end{bmatrix} }{ \left( \Id - \A^{\left(i\right)} \right) \noise \left( \Id - \A^{\left(i\right)} \right)^{\top}}.
\end{align*}
From Cauchy Schwarz, we know
\begin{align*}
    \left( \Id - \A^{\left(i\right)} \right) \noise \left( \Id - \A^{\left(i\right)} \right)^{\top} &\preccurlyeq 2 \noise + 2 \A^{\left(i\right)} \noise  \left( \A^\top \right)^{\left(i\right)}
    = 2 \noise +  2 \opTil^{i} \mycirc \noise,
\end{align*} Using this,
\begin{align*}
    \scal{\HZero}{\E \left[ \athetv{T} \otimes \athetv{T} \right]} &\leq 3 \sigma^2 T^2 \sum_{i=1}^{T} \scal{ \begin{bmatrix}
        0 & 0 \\
        0 & \alpha^{\scriptscriptstyle -2} \cov^{-1}
    \end{bmatrix} }{ 2 \noise +  2 \opTil^{i} \mycirc \noise }.
\end{align*}
Using Lemma~\ref{lem:nesterov-potential-H-inverse} for the right part, we get the following, 
\begin{align*}
    \scal{\HZero}{\E \left[ \athetv{T} \otimes \athetv{T} \right]} &\leq 3 \sigma^2 T^2 \sum_{i=1}^{T} 6d = 18 \left( \sigma^2 d \right) T^3.
\end{align*}
\end{proof}

\subsection{Potentials for Nesterov Method} \label{app:nesterov-potential}

In this section, we will use the  potential functions used in the proof of nesterov accelerated method to bounding the terms in our recursion.  Consider the Algorithm~\ref{alg:nesterov} with exact gradients in that setting, 

\begin{subequations} \label{alg:nesterov:exac}
    \eqals{ \ytp{t+1} = \xtp{t} -  \beta \cov \left(  \xtp{t} - \xt{*} \right) \label{eq:yt:exact}, \\ \ztp{t+1} = \ztp{t} - \alpha (t+1) \cov \left(  \xtp{t} - \xt{*} \right)  \label{eq:nest_zt:exact} ,\\ ( t+2 )\xtp{t+1} = (t+1)\ytp{t+1} + \ztp{t+1}. \label{eq:xt:exact} }
\end{subequations}
By similar rescaling and the definition of the operator, it can be seen that 
\begin{align}\label{eq:optil-nesterov-equivalence}
    \begin{bmatrix}
        t ( \ytp{t} - \xt{*} ) \\
        \ztp{t} - \xt{*} 
    \end{bmatrix} = \A^{t} \mycirc      \begin{bmatrix}
        0 \\
        \xt{0} - \xt{*} .
    \end{bmatrix}
\end{align}

\begin{lemma} \label{lem:nesterov-potential-H}  For the step sizes satisfying $0 < \alpha \leq \beta  \leq 1/L $, 
    \begin{align*}
        \scal{\HZero}{\opTil^{~\left(t\right)} \mycirc \left[ \thet{0} \otimes \thet{0} \right] } \leq \min\left\lbrace \frac{1}{\alpha} , \frac{8 (t+1)}{\beta} \right\rbrace \nor{\xt{0} - \xt{*}}^2.
    \end{align*}
\end{lemma}
\begin{proof}
    From the above equivalence \eqref{eq:optil-nesterov-equivalence}, we can see that
    \begin{align*}
        \scal{\HZero}{\opTil^{~\left(t\right)} \mycirc \left[ \thet{0} \otimes \thet{0} \right] } &= (t+1)^2\nor{\xtp{t} - \xt{0}}^2.
    \end{align*}
    Now using the potential function, 
    $ V_{t} = (t)(t+1)\nor{\ytp{t} - \xt{*}}^2_{\cov} + \frac{1}{a}\nor{\ztp{t} - \xt{*}}^2 $, for $\alpha \leq \beta$ we can see that $ V_{t} \leq V_{t-1} \leq \ldots V_0$. Using this \begin{align*}
        t^2 \nor{\ytp{t} - \xt{*}}^2 \leq \frac{1}{a} \nor{\ztp{0} - \xt{*}}^2 = \frac{1}{a}  \nor{\xt{0} - \xt{*}}^2, \\
         \nor{\ztp{t} - \xt{*}}^2 \leq \nor{\ztp{0} - \xt{*}}^2 = \nor{\xt{0} - \xt{*}}^2.
    \end{align*}
    Noting that $ ( t+1 )( \xtp{t} - \xt{*} ) = t( \ytp{t} - \xt{*} ) + ( \ztp{t} - \xt{*} ) $ and using Cauchy-Schwarz inequality, 
    \begin{align*}
        (t+1)^2\nor{\xtp{t} - \xt{0}}^2 \leq 2  t^2 \nor{\ytp{t} - \xt{*}}^2 + 2  \nor{\ztp{t} - \xt{*}}^2 =  \left( \frac{2}{\alpha} + 2 \right) \nor{\xt{0} - \xt{*}}^2.
    \end{align*} But doing exact computations gives better bounds. The above algorithm is exactly equivalent to the algorithm considered in \citet{flammarion2015averaging} as seen below
    \begin{align*}
        \eta_{t+1} = ( \I - \alpha \cov ) \eta_{t} + ( \I - \beta \cov ) ( \eta_t - \eta_{t-1} ). 
    \end{align*}
    where $\eta_{t} = (t+1) ( \xtp{t} - \xt{*} ) $. Hence we can apply their results giving the bound $\min\left\lbrace \frac{1}{\alpha} , \frac{8 (t+1)}{\beta} \right\rbrace \nor{\xt{0} - \xt{*}}^2$.
\end{proof}

\begin{lemma} \label{lem:nesterov-potential-H-inverse} 
    \begin{align*}
        \scal{\begin{bmatrix}
            0 & 0 \\
            0 & \alpha^{\scriptscriptstyle -2} \cov^{-1}
        \end{bmatrix}}{\opTil^{~\left(t\right)} \mycirc \noise } \leq 2.
    \end{align*}
\end{lemma}
\begin{proof}
Both $\A$ and $\noise$ are diagonizable wrt to the eigen basis of $\cov$. We will now project these block  matrices onto their eigen basis and compute the summation of each component individually. Note that,
\begin{align}
    \A = \sum_{i=1}^d \A_i  \otimes_{k} e_ie_i^\top, \qquad
    \noise = \sum_{i=1}^d \noise_i \otimes_{k} e_ie_i^\top.
\end{align}
where $\A_i$ and $\noise_i$ are 
\begin{align}
    \A_i = \begin{bmatrix}
        1 - \beta \lambda_i & 1- \beta \lambda_i \\
        - \alpha \lambda_i & 1 - \alpha \lambda_i  
    \end{bmatrix}, \qquad 
    \noise_i = \begin{bmatrix}
        \beta^2 \lambda_i & \beta \alpha \lambda_i \\
        \beta \alpha \lambda_i & \alpha^2 \lambda_i  
    \end{bmatrix}.
\end{align}
Now, the scalar product using the properties of Kronecker product, 
\begin{align*}
    \scal{\begin{bmatrix}
        0 & 0 \\
        0 & \alpha^{\scriptscriptstyle -2} \cov^{-1}
    \end{bmatrix}}{\opTil^{~\left(t\right)} \mycirc \noise } = \sum_{i=1}^{d} \scal{\begin{bmatrix}
        0 & 0 \\
        0 & \alpha^{-2}\lambda_i^{-1}
    \end{bmatrix} }{\A_i^{t} \cdot \noise_{i} \cdot \left( \A_i^{t} \right)^{\top}}.
\end{align*}
To compute $\scal{\begin{bmatrix}
    0 & 0 \\
    0 & \alpha^{-2}\lambda_i^{-1}
\end{bmatrix} }{\A_i^{t} \cdot \noise_{i} \cdot \left( \A_i^{t} \right)^{\top}}$ we invoke Lemma~\ref{alg:nesterov:1D}, with \begin{align*}
    \Gamma = \A_i, &\qquad \aleph = \lambda_i \noise_{i},\\
    \bl = \beta \lambda_i, &\qquad
     \al = \alpha \lambda_i.
\end{align*} which gives, 
\begin{align*}
    \scal{\begin{bmatrix}
        0 & 0 \\
        0 & \alpha^{-2}\lambda_i^{-1}
    \end{bmatrix} }{\A_i^{t} \cdot \noise_{i} \cdot \left( \A_i^{t} \right)^{\top}} &\leq  1 + \frac{\alpha \lambda_i }{\left( 1 - \beta \lambda_i \right)^2}.
\end{align*}
Note from condition on step sizes \eqref{eq:asgd-step-size} that $ \alpha \leq \beta/2, \beta \lambda_i \leq \beta L \leq \beta R^2 \leq 1/2 $, 
\begin{align*}
    \scal{\begin{bmatrix}
        0 & 0 \\
        0 & \alpha^{-2}\lambda_i^{-1}
    \end{bmatrix} }{\A_i^{t} \cdot \noise_{i} \cdot \left( \A_i^{t} \right)^{\top}} &\leq  1 + \frac{\alpha \lambda_i }{\left( 1 - \beta \lambda_i \right)^2} \leq 2.
\end{align*}
Computing the sum across dimension, we get the desired bound.
\end{proof}

\section{Inverting operators}\label{app:inverse-of-operators}

In this section, we give proof for the almost eigenvalues of the operators $ \opM \mycirc ( 1 - \opTil)^{-1} $ ,   $\opM^{\top} \mycirc  \left( \Id - \opTil^{\top} \right)^{-1} $. As described earlier, although the calculations are a bit extensive, the underlying scheme remains the same. To compute  $( 1 - \opTil)^{-1}, \left( \Id - \opTil^{\top} \right)^{-1}  $, we formulate inverse as a summation of geometric series. Then we use the diagonalization of the $\cov$ and compute the geometric series.  In the last part, we use Property~\ref{prop:operator-M-upperbound}  and  Assumptions~\ref{asmp:fourth-moment},~\ref{asmp:stat-condition-number} on the data features to get the final bounds. 

\begin{property} \label{prop:operator-M-upperbound}
    Using $ \Ex{ aa^\top } = \cov  $, the following property holds for any PSD matrix $(\cdot)$,
    \begin{align*}
        \Ex{ \left( \cov - aa^{\top} \right) (\cdot) \left( \cov - aa^{\top} \right) } &=  \Ex{  aa^{\top} (\cdot) aa^{\top} } - \Ex{  \cov (\cdot) \cov }
        \preccurlyeq \Ex{  aa^{\top} (\cdot) aa^{\top} }.
    \end{align*}
\end{property}

\begin{lemma} \label{eq:lem:optil}
    With $ 0 < \alpha,\beta < 1/L$ and $ (\alpha +  2 \beta ) L < 1$,
    \begin{align}
        \left( \Id - \opTil \right)^{-1} \mycirc \noise &\preccurlyeq \frac{1}{3} \begin{bmatrix}
            2 \alpha  ( \beta \cov )^{-1}  +  (2\beta -3 \alpha) \I &  \alpha \beta^{-1}  (2\beta  - \alpha) \I   \\
            \alpha \beta^{-1}  (2\beta  - \alpha) \I  &  2  \alpha^2 \beta^{-1} \I 
          \end{bmatrix} .
    \end{align}
\end{lemma}

\begin{proof} We will compute the inverse by evaluating the summation of the following infinite series, 
    \begin{align*}
        \left( \Id - \opTil \right)^{-1} \mycirc \noise = \sum_{t=0}^{\infty}  \opTil^{t} \mycirc \noise
        &= \sum_{t=0}^{\infty} \A^{t} \cdot \noise \cdot (\A^{t})^{\top} .
    \end{align*} 
Both $\A$ and $\noise$ are diagonizable wrt to the eigen basis of $\cov$. We will now project these block  matrices onto their eigen basis and compute the summation of each component individually. Note that,
\begin{align}
    \A = \sum_{i} \A_i  \otimes_{k} e_ie_i^\top, \qquad
    \noise = \sum_{i} \noise_i \otimes_{k} e_ie_i^\top.
\end{align}
where $\A_i$ and $\noise_i$ are 
\begin{align}
    \A_i = \begin{bmatrix}
        1 - \beta \lambda_i & 1- \beta \lambda_i \\
        - \alpha \lambda_i & 1 - \alpha \lambda_i  
    \end{bmatrix}, \qquad 
    \noise_i = \begin{bmatrix}
        \beta^2 \lambda_i & \beta \alpha \lambda_i \\
        \beta \alpha \lambda_i & \alpha^2 \lambda_i  
    \end{bmatrix}.
\end{align}
Using these projections, 
\begin{align}
    \sum_{t=0}^{\infty} \A^{t} \cdot \noise \cdot (\A^{t})^{\top}  &= \sum_{t=0}^{\infty} \sum_i \left( \A_i^{t} \cdot \noise_{i} \cdot \left( \A_i^{t} \right)^{\top} \right) \otimes_{k} e_i e_i^{\top}, \\
    &= \sum_i \left[  \sum_{t=0}^{\infty} \A_i^{t} \cdot \noise_{i} \cdot \left( \A_i^{t} \right)^{\top}   \right] \otimes_{k} e_i e_i^{\top}. 
\end{align}
We invoke Lemma \ref{lem:A-infinite-series} with 
\begin{align*}
    \Gamma = \A_i, &\qquad \aleph = \lambda_i \noise_{i}, \\
    \bl = \beta \lambda_i, &\qquad
     \al = \alpha \lambda_i .
\end{align*}
\begin{align*}
    \sum_{t=0}^{\infty} \A_i^{t} \cdot \lambda_i \noise_{i} \cdot \left( \A_i^{t} \right)^{\top}  &= \frac{1}{ \beta \lambda_i  (4 - (\alpha+ 2 \beta) \lambda_i )} \begin{bmatrix}
        2 \alpha \lambda_i  + \beta \lambda_i (2\beta \lambda_i- 3 \alpha \lambda_i) &  \alpha \lambda_i  (2\beta \lambda_i - \alpha \lambda_i  )   \\
        \alpha \lambda_i (2\beta \lambda_i-\alpha \lambda_i) &  2 ( \alpha \lambda_i ) ^2  
   \end{bmatrix},  \\
   \sum_{t=0}^{\infty} \A_i^{t} \cdot \noise_{i} \cdot \left( \A_i^{t} \right)^{\top}  &= \frac{1}{   (4 - (\alpha+ 2 \beta) \lambda_i )} \begin{bmatrix}
     2 \alpha  ( \beta \lambda_i )^{-1}  +  (2\beta -3 \alpha) &  \alpha \beta^{-1}  (2\beta  - \alpha)   \\
     \alpha \beta^{-1}  (2\beta  - \alpha)  &  2  \alpha^2 \beta^{-1}
   \end{bmatrix}.
\end{align*}
We have \begin{align*}
    (\alpha+ 2 \beta) \lambda_i \leq (\alpha+ 2 \beta) L\leq 1, \\
    \text{Hence, } 4 -((\alpha+ 2 \beta) \lambda_i) \geq 3 .
\end{align*} 
Also, $\sum\limits_{t=0}^{\infty} \A_i^{t} \cdot \noise_{i} \cdot \left( \A_i^{t} \right)^{\top}$ is PSD as $ \opTil, \noise $ are positive. Hence, the following holds
\begin{align*}
    3 \sum_{t=0}^{\infty} \A_i^{t} \cdot \noise_{i} \cdot \left( \A_i^{t} \right)^{\top} \preccurlyeq 4 -((\alpha+ 2 \beta) \lambda_i) \sum_{t=0}^{\infty} \A_i^{t} \cdot \noise_{i} \cdot \left( \A_i^{t} \right)^{\top}, \\
    =  \begin{bmatrix}
        2 \alpha  ( \beta \lambda_i )^{-1}  +  (2\beta -3 \alpha) &  \alpha \beta^{-1}  (2\beta  - \alpha)   \\
        \alpha \beta^{-1}  (2\beta  - \alpha)  &  2  \alpha^2 \beta^{-1}
      \end{bmatrix}. \\ 
\end{align*}
This given the following 
\begin{align} \label{eq:lem:optil:eigen}
     \sum_{t=0}^{\infty} \A_i^{t} \cdot \noise_{i} \cdot \left( \A_i^{t} \right)^{\top} \preccurlyeq \frac{1}{3} \begin{bmatrix}
        2 \alpha  ( \beta \lambda_i )^{-1}  + (2\beta -3 \alpha) &  \alpha \beta^{-1}  (2\beta  - \alpha)   \\
        \alpha \beta^{-1}  (2\beta  - \alpha)  &  2  \alpha^2 \beta^{-1}
      \end{bmatrix}.
\end{align}
Using the fact that kronecker product of two PSD matrices is positive, 
\begin{align*} 
    \sum_{t=0}^{\infty} \A_i^{t} \cdot \noise_{i} \cdot \left( \A_i^{t} \right)^{\top} \otimes_{k} e_i e_i^{\top}  \preccurlyeq \frac{1}{3} \begin{bmatrix}
       2 \alpha  ( \beta \lambda_i )^{-1}  +  (2\beta -3 \alpha) &  \alpha \beta^{-1}  (2\beta  - \alpha)   \\
       \alpha \beta^{-1}  (2\beta  - \alpha)  &  2  \alpha^2 \beta^{-1}  
     \end{bmatrix} \otimes_{k}  e_i e_i^{\top}.
\end{align*}
Now adding this result along all directions we get 
\begin{align*}
    \sum_i \sum_{t=0}^{\infty} \A_i^{t} \cdot \noise_{i} \cdot \left( \A_i^{t} \right)^{\top} \otimes_{k} e_i e_i^{\top}  &\preccurlyeq \frac{1}{3} \sum_i  \begin{bmatrix}
       2 \alpha  ( \beta \lambda_i )^{-1}  +  (2\beta -3 \alpha) &  \alpha \beta^{-1}  (2\beta  - \alpha)   \\
       \alpha \beta^{-1}  (2\beta  - \alpha)  &  2  \alpha^2 \beta^{-1}  
     \end{bmatrix} \otimes_{k}  e_i e_i^{\top}, \\
     \left( \Id - \opTil \right)^{-1} \mycirc \noise &\preccurlyeq \frac{1}{3} \begin{bmatrix}
        2 \alpha  ( \beta \cov )^{-1}  +  (2\beta -3 \alpha) \I &  \alpha \beta^{-1}  (2\beta  - \alpha) \I   \\
        \alpha \beta^{-1}  (2\beta  - \alpha) \I  &  2  \alpha^2 \beta^{-1} \I 
      \end{bmatrix},
\end{align*}
where the last inequation comes from the facts \begin{align*}
    \I = \sum_i e_i e_i^{\top} \qquad H^{-1} = \sum_{i} \lambda_i^{-1} e_i e_i^{\top}.
\end{align*}
\end{proof}

\begin{lemma} \label{eq:lem:inv:opt} With $ ( \alpha + 2 \beta ) R^2 \leq 1 ,  \alpha \leq  \frac{\beta}{2 \ki} $, \begin{align}
        \pinf \defeq \left( \Id - \opT \right)^{-1} \mycirc \noise &\preccurlyeq 3 \cdot \left(\Id - \opTil\right)^{-1} \mycirc  \noise = \begin{bmatrix}
            2 \alpha  ( \beta \cov )^{-1}  +  (2\beta -3 \alpha) \I &  \alpha \beta^{-1}  (2\beta  - \alpha) \I   \\
            \alpha \beta^{-1}  (2\beta  - \alpha) \I  &  2  \alpha^2 \beta^{-1} \I 
          \end{bmatrix}. 
    \end{align}
\end{lemma}

\begin{proof}
    Writing the inverse as a sum of exponential series gives us 
    \begin{align*}
        \pinf \defeq \left( \Id - \opT \right)^{-1} \mycirc \noise = \sum_{t=0}^{\infty} \opT^{t} \mycirc \noise.
    \end{align*}
    Recursion for $\opT^{t} \mycirc \noise$ will be as follows,
    \begin{align*}
        \opT^{t} \mycirc \noise &= \opTil \mycirc \opT^{t-1} \mycirc \noise + \opM \mycirc  \opT^{t-1} \mycirc \noise, \\
        &= \opTil^2 \mycirc \opT^{t-2} \mycirc \noise +  \opTil \mycirc \opM \mycirc  \opT^{t-2} \mycirc \noise + \opM \mycirc  \opT^{t-1} \mycirc \noise ,\\
        &= \opTil^{t} \mycirc \noise + \sum_{k=0}^{t-1} \opTil^{t-k-1} \mycirc \opM \mycirc  \opT^{k} \mycirc \noise .
    \end{align*}
    Taking the sum of these terms from $0$ to $\infty$
    \begin{align*}
        \sum_{t=0}^{\infty} \opT^{t} \mycirc \noise
        &= \sum_{t=0}^{\infty}  \opTil^{t} \mycirc \noise + \sum_{t=0}^{\infty}  \sum_{k=0}^{t-1} \opTil^{t-k-1} \mycirc \opM \mycirc  \opT^{k} \mycirc \noise .
    \end{align*}
    Interchanging the summations in the second part, 
    \begin{align*}
        \sum_{t=0}^{\infty} \opT^{t} \mycirc \noise
        &=  \left( \sum_{t=0}^{\infty}  \opTil^{t} \right) \mycirc \noise + \sum_{k=0}^{\infty} \left(  \sum_{t=k+1}^{\infty} \opTil^{t-k-1} \right) \mycirc \opM \mycirc  \opT^{k} \mycirc \noise \\
        \text{Using  } \sum_{t=k+1}^{\infty} \opTil^{t-k-1} =  \sum_{t=0}^{\infty} \opTil^{t} &= (\Id - \opTil)^{-1} ,\\
        \sum_{t=0}^{\infty} \opT^{t} \mycirc \noise
        &=  \left(\Id - \opTil\right)^{-1} \mycirc \noise + \sum_{k=0}^{\infty} \left(\Id - \opTil\right)^{-1} \mycirc \opM \mycirc  \opT^{k} \mycirc \noise , \\
        &=  \left(\Id - \opTil\right)^{-1} \mycirc \noise +  \left(\Id - \opTil\right)^{-1} \mycirc \opM \mycirc  \sum_{k=0}^{\infty} \opT^{k} \mycirc \noise ,\\
        \pinf &=  \left(\Id - \opTil\right)^{-1} \mycirc \noise +  \left(\Id - \opTil\right)^{-1} \mycirc \opM \mycirc \pinf.
    \end{align*}
    From this we have, 
    \begin{align*}
        \pinf - \left(\Id - \opTil\right)^{-1} \mycirc \opM \mycirc \pinf  &=  \left(\Id - \opTil\right)^{-1} \mycirc \noise, \\
        \left( \Id - \left(\Id - \opTil\right)^{-1} \mycirc \opM \right) \mycirc \pinf &= \left(\Id - \opTil\right)^{-1} \mycirc \noise, \\
        \pinf &= \left( \Id - \left(\Id - \opTil\right)^{-1} \mycirc \opM \right)^{-1} \mycirc \left(\Id - \opTil\right)^{-1} \mycirc \noise. 
    \end{align*}
    Writing the inverse as a sum of exponential series gives us  
    \begin{align} \label{eq:lem:inv:opt:master}
        \pinf &= \sum_{t=0}^{\infty} \left( \left(\Id - \opTil\right)^{-1} \mycirc \opM \right)^{t} \mycirc \left(\Id - \opTil\right)^{-1} \mycirc \noise, 
    \end{align}
    Note $ ( \alpha +  2 \beta)R^2 \leq 1 \implies  ( \alpha +  2 \beta) \lambda_{\max} \leq 1 $. Hence we can invoke Lemma~\ref{lem:opm-optil-inv} here. 
    \begin{align*} 
        \opM \mycirc ( 1 - \opTil)^{-1} \mycirc \noise  &\preccurlyeq \left[ \frac{2 \alpha \ki }{3 \beta} + \frac{( \alpha + 2 \beta ) R^2}{3}  \right] \noise
    \end{align*}
    Using $$  ( \alpha + 2 \beta ) R^2 \leq 1 ,  \alpha \leq  \frac{\beta}{2 \ki}, $$
    \begin{align*} 
        \opM \mycirc ( 1 - \opTil)^{-1} \mycirc \noise  &\preccurlyeq \left[ \frac{1}{3} + \frac{1}{3}  \right] \noise \preccurlyeq \frac{2}{3} \noise.
    \end{align*}
    Using this in \eqref{eq:lem:inv:opt:master}, 
    \begin{align*} 
        \pinf &= \sum_{t=0}^{\infty} \left( \left(\Id - \opTil\right)^{-1} \mycirc \opM \right)^{t} \mycirc \left(\Id - \opTil\right)^{-1} \mycirc \noise, \\
        \text{Use } \left( \left(\Id - \opTil\right)^{-1} \mycirc \opM \right)^{t} & \mycirc \left(\Id - \opTil\right)^{-1}  = \left(\Id - \opTil\right)^{-1} \mycirc \left( \opM \mycirc \left(\Id - \opTil\right)^{-1} \right)^{t}, \\
        \pinf &= \left(\Id - \opTil\right)^{-1} \mycirc \sum_{t=0}^{\infty}  \left( \opM \mycirc \left(\Id - \opTil\right)^{-1} \right)^{t} \mycirc \noise ,\\
        \text{Using } \left( \opM \mycirc \left(\Id - \opTil\right)^{-1} \right)^{t} &\preccurlyeq \left[  \frac{2}{3} \right]^{t} \noise, \\
        \pinf &\preccurlyeq  \left(\Id - \opTil\right)^{-1} \mycirc \sum_{t=0}^{\infty} \left[  \frac{2}{3} \right]^{t} \cdot \noise, \\
        &\preccurlyeq \left(\Id - \opTil\right)^{-1} \mycirc 3\cdot \noise .
    \end{align*}
    This completes the proof. 
\end{proof}

\begin{lemma} \label{lem:pqrs}
    For any block matrix $\begin{bmatrix}
        P & Q \\
        R & S
    \end{bmatrix} $, 

    \begin{align*} 
        \opM \mycirc \begin{bmatrix}
            P & Q \\
            R & S
        \end{bmatrix} &= \begin{bmatrix}
            \beta^2  & \alpha \beta \\
            \alpha \beta &  \alpha^2
        \end{bmatrix} \otimes_{k} \Ex{ \left( \cov - aa^{\top} \right) ( P + Q + R + S) \left( \cov - aa^{\top} \right) }
    \end{align*}
    
\end{lemma}

\begin{proof}
\begin{align*}
        \opM \mycirc \begin{bmatrix}
            P & Q \\
            R & S
        \end{bmatrix} &= \Ex{ \begin{bmatrix}  \beta \Ha & \beta \Ha \\  \alpha \Ha & \alpha \Ha \end{bmatrix}  \begin{bmatrix}
        P & Q \\
        R & S
    \end{bmatrix} \begin{bmatrix}  \beta \Ha & \alpha \Ha \\  \beta \Ha & \alpha \Ha \end{bmatrix} }
    \end{align*}
    where $ \Ha = \left( \cov - a a^\top \right) $
    \begin{align*}
        \begin{bmatrix}  \beta \Ha & \beta \Ha \\  \alpha \Ha & \alpha \Ha \end{bmatrix}  \begin{bmatrix}
            P & Q \\
            R & S
        \end{bmatrix} &= \begin{bmatrix}
            \beta \Ha ( P + R ) & \beta \Ha ( Q + S ) \\
            \alpha \Ha ( P + R ) & \alpha \Ha ( Q + S )
        \end{bmatrix}, \\
        \begin{bmatrix}
            \beta \Ha ( P + R ) & \beta \Ha ( Q + S ) \\
            \alpha \Ha ( P + R ) & \alpha \Ha ( Q + S )
        \end{bmatrix} \begin{bmatrix}  \beta \Ha & \alpha \Ha \\  \beta \Ha & \alpha \Ha \end{bmatrix}  &= \begin{bmatrix}
            \beta^2 \Ha ( P + Q + R + S) \Ha &  \alpha \beta \Ha ( P + Q + R + S) \Ha \\
            \alpha \beta \Ha ( P + Q + R + S) \Ha & \alpha^2 \Ha ( P + Q + R + S) \Ha
        \end{bmatrix} ,\\
        &= \begin{bmatrix}
            \beta^2  & \alpha \beta \\
            \alpha \beta &  \alpha^2
        \end{bmatrix} \otimes_{k} \Ha ( P + Q + R + S) \Ha.
    \end{align*} 
Taking expectation, 
    \begin{align*}
        \opM \mycirc \begin{bmatrix}
            P & Q \\
            R & S
        \end{bmatrix} &=  \begin{bmatrix}
            \beta^2  & \alpha \beta \\
            \alpha \beta &  \alpha^2
        \end{bmatrix} \otimes_{k} \Ex{ \Ha ( P + Q + R + S) \Ha }.
    \end{align*} This completes the proof.
\end{proof}

\begin{lemma} \label{lem:kurtosis-ut}
    Under Assumption \ref{asmp:uniform-kurtosis}, for any $ t \geq 0 $, 
    \begin{align*}
        \opM \mycirc \Ex{ \thet{t} \otimes \thet{t} } \preccurlyeq \kappa \scal{\coff}{\Ex{\thet{t} \otimes \thet{t}}} \noise.
    \end{align*}
\end{lemma}

\begin{proof}
    \begin{align*}
        \opM \mycirc \Ex{ \thet{t} \otimes \thet{t} } &= \Ex{ \J \Ex{ \thet{t} \otimes \thet{t} } \J^{\top} }, \\
        \thet{t} \otimes \thet{t} &= \matvz{t} \otimes \matvz{t} = \begin{bmatrix}
            \vt{t} \vt{t}^\top & \vt{t} \wt{t}^\top \\ \wt{t} \vt{t}^\top & \wt{t} \wt{t}^\top
        \end{bmatrix}, \\
        \Ex{ \thet{t} \otimes \thet{t} } &= \begin{bmatrix}
            \Ex{ \vt{t} \vt{t}^\top } & \Ex{ \vt{t} \wt{t}^\top } \\ \Ex{\wt{t} \vt{t}^\top} & \Ex{ \wt{t} \wt{t}^\top }
        \end{bmatrix} .
    \end{align*}
    As $J$ and $\Ex{ \thet{t} \otimes \thet{t} }$ are independent, invoking Lemma~\ref{lem:pqrs} with 
    \begin{align*}
        \begin{bmatrix}
            P & Q \\ R & S
        \end{bmatrix} &= \begin{bmatrix}
            \Ex{ \vt{t} \vt{t}^\top } & \Ex{ \vt{t} \wt{t}^\top } \\ \Ex{\wt{t} \vt{t}^\top} & \Ex{ \wt{t} \wt{t}^\top }
        \end{bmatrix} .
    \end{align*}
    Now $P+Q+R+S$ in our case will be 
    \begin{align*}
        P+Q+R+S &= \Ex{ \vt{t} \vt{t}^\top + \vt{t} \wt{t}^\top + \wt{t} \vt{t}^\top + \wt{t} \wt{t}^\top }, \\
        &= \Ex{ \left(\vt{t} + \wt{t}\right) \left(\vt{t} + \wt{t}\right)^\top } = \Ex{ \ut{t} \ut{t} ^\top }, \text{ from \eqref{eq:xt} . }
    \end{align*}
    Using this we get 
    \begin{align*}
        \opM \mycirc \Ex{ \thet{t} \otimes \thet{t} } =  \begin{bmatrix}
            \beta^2  & \alpha \beta \\
            \alpha \beta &  \alpha^2
        \end{bmatrix} \otimes_{k} \Ex{ \left( \cov - aa^{\top} \right) \Ex{ \ut{t} \ut{t} ^\top } \left( \cov - aa^{\top} \right) }.
    \end{align*}
    Using Property~\ref{prop:operator-M-upperbound}, 
    \begin{align*}
        \Ex{ \left( \cov - aa^{\top} \right) \Ex{ \ut{t} \ut{t} ^\top } \left( \cov - aa^{\top} \right) } \preccurlyeq 
        \Ex{ aa^\top \Ex{ \ut{t} \ut{t} ^\top } aa^\top } = \Ex{ \scal{a}{\Ex{ \ut{t} \ut{t}^\top } a} aa^\top }.
    \end{align*}
    Using Assumption~\ref{asmp:uniform-kurtosis} with $M = \Ex{\ut{t} \ut{t}^\top} $, 
    \begin{align*}
        \Ex{ \scal{a}{\Ex{\ut{t} \ut{t}^\top} a} aa^\top } &\preccurlyeq \kappa \tr{\left( \Ex{ \cov \ut{t} \ut{t}^\top } \right)} \cov .
    \end{align*}
    As kronecker product of two PSD matrices is positive, 
    \begin{align*}
        \opM \mycirc \Ex{ \thet{t} \otimes \thet{t} } &\preccurlyeq \begin{bmatrix}
            \beta^2  & \alpha \beta \\
            \alpha \beta &  \alpha^2
        \end{bmatrix} \otimes_{k} \kappa \tr{\left( \Ex{ \cov \ut{t} \ut{t}^\top } \right)} \cov, \\
        &= \kappa \tr{\left( \Ex{ \cov \ut{t} \ut{t}^\top } \right)} \begin{bmatrix}
            \beta^2 \cov  & \alpha \beta \cov \\
            \alpha \beta \cov &  \alpha^2 \cov
        \end{bmatrix} = \kappa \tr{\left( \Ex{ \cov \ut{t} \ut{t}^\top } \right)} \noise.
    \end{align*}
Noting that $ \scal{\coff}{\Ex{\thet{t} \otimes \thet{t}}} = \tr{\left( \Ex{ \cov \ut{t} \ut{t}^\top } \right)} $ completes the proof. 
\end{proof}

\begin{lemma} \label{lem:opm-optil-inv} With 
$ ( \alpha + 2 \beta ) R^2 \leq 1 ,  \alpha \leq  \frac{\beta}{2 \ki}$ ,    \begin{align*} 
        \opM \mycirc ( 1 - \opTil)^{-1} \mycirc \noise  &\preccurlyeq \frac{2}{3} \noise
    \end{align*}
\end{lemma}

\begin{proof}
    Note $ ( \alpha +  2 \beta)R^2 \leq 1 \implies  ( \alpha +  2 \beta) \lambda_{\max} \leq 1 $. Hence we can invoke  Lemma~\ref{eq:lem:optil} here.  
    \begin{align*}
        \opM \mycirc ( 1 - \opTil)^{-1} \mycirc \noise \preccurlyeq \opM \mycirc \frac{1}{3} \begin{bmatrix}
            2 \alpha  ( \beta \cov )^{-1}  +  (2\beta -3 \alpha) \I &  \alpha \beta^{-1}  (2\beta  - \alpha) \I   \\
            \alpha \beta^{-1}  (2\beta  - \alpha) \I  &  2  \alpha^2 \beta^{-1} \I 
          \end{bmatrix} .
    \end{align*}
    Invoking the Lemma~\ref{lem:pqrs} for a block matrix, \begin{align*} 
        \begin{bmatrix}
            P & Q \\
            R & S
        \end{bmatrix} &= \frac{1}{3} \begin{bmatrix}
            2 \alpha  ( \beta \cov )^{-1}  +  (2\beta -3 \alpha) \I &  \alpha \beta^{-1}  (2\beta  - \alpha) \I   \\
            \alpha \beta^{-1}  (2\beta  - \alpha) \I  &  2  \alpha^2 \beta^{-1} \I 
          \end{bmatrix} .
    \end{align*}
Now $P+Q+R+S$ in our case is 
    \begin{align*}
       3*(P+Q+R+S) &=
        2 \alpha  ( \beta \cov )^{-1}  +  (2\beta -3 \alpha) \I + 2 \alpha \beta^{-1}  (2\beta  - \alpha) \I + 2  \alpha^2 \beta^{-1} \I  ,\\
        &= 2 \alpha  ( \beta \cov )^{-1} + (2\beta -3 \alpha) \I + 4 \alpha \I - 2 \alpha^2 \beta^{-1} \I +  + 2  \alpha^2 \beta^{-1} \I ,\\
        &=  2 \alpha  ( \beta \cov )^{-1} + ( \alpha + 2 \beta ) \I .
    \end{align*}
    \begin{align} \label{eq:opm-optil}
        \opM \mycirc ( 1 - \opTil)^{-1} \mycirc \noise  = \frac{1}{3} \begin{bmatrix}
            \beta^2  & \alpha \beta \\
            \alpha \beta &  \alpha^2
        \end{bmatrix}   \otimes_{k} \Ex{ \left( \cov - aa^{\top} \right) \left[ 2 \alpha  ( \beta \cov )^{-1} + ( \alpha + 2 \beta ) \I \right] \left( \cov - aa^{\top} \right) }. 
    \end{align}
Using Property~\ref{prop:operator-M-upperbound},
    \begin{align*}
        \Ex{ \left( \cov - aa^{\top} \right) \left[ 2 \alpha  ( \beta \cov )^{-1} + ( \alpha + 2 \beta ) \I \right] \left( \cov - aa^{\top} \right) } &\preccurlyeq \Ex{  aa^{\top} \left[ 2 \alpha  ( \beta \cov )^{-1} + ( \alpha + 2 \beta ) \I \right] aa^{\top} }, \\
        &= 2 \frac{\alpha}{\beta} \Ex{ \nor{a}_{\cov^{-1}}^2 ~ aa^\top } + ( \alpha + 2 \beta ) \Ex{ \nor{a}^2 ~ aa^\top }.
    \end{align*}
Using the Assumptions~\ref{asmp:fourth-moment},~\ref{asmp:stat-condition-number} of the feature distribution, we have
    \begin{align*}
        \Ex{ \left( \cov - aa^{\top} \right) \left[ 2 \alpha  ( \beta \cov )^{-1} + ( \alpha + 2 \beta ) \I \right] \left( \cov - aa^{\top} \right) } &\preccurlyeq \frac{2 \alpha \ki }{\beta} \cov + ( \alpha + 2 \beta ) R^2 \cov ,
    \end{align*}
    Using the above in~\eqref{eq:opm-optil} and that fact that kronecker product of two PSD matrices is positive we get, 
\begin{align*} 
        \opM \mycirc ( 1 - \opTil)^{-1} \mycirc \noise  &\preccurlyeq \frac{1}{3} \begin{bmatrix}
            \beta^2  & \alpha \beta \\
            \alpha \beta &  \alpha^2
        \end{bmatrix}   \otimes_{k} \left[ \frac{2 \alpha \ki }{\beta} \cov + ( \alpha + 2 \beta ) R^2 \cov \right] , \\
        &= \left[ \frac{2 \alpha \ki }{3 \beta} + \frac{( \alpha + 2 \beta ) R^2}{3}  \right] \left( \begin{bmatrix}
            \beta^2  & \alpha \beta \\
            \alpha \beta &  \alpha^2
        \end{bmatrix}   \otimes_{k} \cov \right), \\
        &= \left[ \frac{2 \alpha \ki }{3 \beta} + \frac{( \alpha + 2 \beta ) R^2}{3}  \right] \noise.
    \end{align*}
    where the last step is from the definition of $\noise$. 
    Using $$  ( \alpha + 2 \beta ) R^2 \leq 1 ,  \alpha \leq  \frac{\beta}{2 \ki}, $$ 
    \begin{align*} 
        \opM \mycirc ( 1 - \opTil)^{-1} \mycirc \noise  &\preccurlyeq \left[ \frac{1}{3} + \frac{1}{3}  \right] \noise \preccurlyeq \frac{2}{3} \noise.
    \end{align*}
\end{proof}


\begin{lemma} \label{lem:opM:opTil:transpose} With $  ( \alpha + 2 \beta ) R^2 \leq 1 ,  \alpha \leq  \frac{\beta}{2 \ki}, $ and $\coff$ from \eqref{eq:noise_and_coeff}, we have, 
    \begin{align}
        \opM^{\top} \mycirc  \left( \Id - \opTil^{\top} \right)^{-1} \mycirc \coff &\preccurlyeq \frac{2}{3} \coff.
    \end{align}
\end{lemma}
\begin{proof}
    Compute the inverse by evaluating the summation of the following infinite series, 
    \begin{align*}
        \left( \Id - \opTil^\top \right)^{-1} \mycirc \coff = \sum_{t=0}^{\infty} \left( \opTil^{t} \right)^{\top} \mycirc \noise
        &= \sum_{t=0}^{\infty} (\A^{t})^{\top}  \cdot \coff \cdot \A^{t}.
    \end{align*} 
    From this it follows that,
    \begin{align*}
        \opM^{\top} \mycirc  \left( \Id - \opTil^{\top} \right)^{-1} = \Ex{ \J^{\top} \cdot \left( (\A^{t})^{\top}  \cdot \coff \cdot \A^{t} \right) \cdot \J }.
    \end{align*}

\begin{itemize}
    \item Both $\A$ and $\coff$ are diagonizable wrt to the eigen basis of $\cov$. We will now project these block  matrices onto their eigen basis and compute the summation of each component individually. Note that,
\begin{align}
    \A = \sum_{i} \A_i  \otimes_{k} e_ie_i^\top, \qquad
    \coff = \sum_{i} \coff_i \otimes_{k} e_ie_i^\top .
\end{align}
where $\A_i$ and $\coff_i$ are 
\begin{align*}
    \A_i = \begin{bmatrix}
        1 - \beta \lambda_i & 1- \beta \lambda_i \\
        - \alpha \lambda_i & 1 - \alpha \lambda_i  
    \end{bmatrix} \qquad 
    \coff_i = \begin{bmatrix}
         \lambda_i &  \lambda_i \\
         \lambda_i & \lambda_i  
    \end{bmatrix}.
\end{align*} 
Using these projections, 
\begin{align*}
     (\A^{t})^{\top} \cdot \coff \cdot (\A^{t})  &= \sum_i \left( (\A_i^{t})^{\top}  \cdot \coff_i \cdot \A_i^{t} \right) \otimes_{k} e_i e_i^{\top} .
\end{align*}

\item Now the random matrix  \begin{align*} \J =  \begin{bmatrix}  \beta \left(H-aa^\top\right) & \beta \left(H-aa^\top\right) \\  \alpha \left(H-aa^\top\right) & \alpha \left(H-aa^\top\right) \end{bmatrix} =  \begin{bmatrix} \beta & \beta \\ \alpha & \alpha \end{bmatrix} \otimes_{k} \left(H-aa^\top\right).
 \end{align*}
From the mixed product property of kronecker product i.e for any matrices of appropriate dimension $P,Q,R,S$
\[  \left( P \otimes_{k} Q \right) \left( R \otimes_{k} S \right)  = PR \otimes_{k} QS.  \]\begin{align}    \label{eq:bias-eigen-decomposition}
    \begin{aligned}
        \Ex{\J^\top \cdot \left[ \left( (\A_i^{t})^{\top}  \cdot \coff_i \cdot \A_i^{t} \right) \otimes_{k} e_ie_i^\top \right]  \cdot \J}  &= \begin{bmatrix} \beta & \alpha \\ \beta & \alpha \end{bmatrix} \left( (\A_i^{t})^{\top}  \cdot \coff_i \cdot \A_i^{t} \right) \begin{bmatrix} \beta & \beta \\ \alpha & \alpha \end{bmatrix} \\ & \hspace{2.5cm} \otimes_{k} \Ex{\left(H-aa^\top\right) e_ie_i^\top \left(H-aa^\top\right)}.
    \end{aligned}
\end{align}
\end{itemize}
Using the above observations,  
\begin{align*}
    \sum_{t=0}^{\infty} (\A^{t})^{\top} \cdot \coff \cdot (\A^{t})   &= \sum_{t=0}^{\infty} \sum_i \left( \left( \A_i^{t} \right)^{\top}  \cdot \coff_i \cdot \left( \A_i^{t} \right) \right) \otimes_{k} e_i e_i^{\top} ,\\
    \Ex{ \sum_{t=0}^{\infty} \J^\top (\A^{t})^{\top} \cdot \coff \cdot (\A^{t}) \J } &= \sum_{t=0}^{\infty} \sum_{i} \Ex{\J^\top \cdot \left[ \left( (\A_i^{t})^{\top}  \cdot \coff_i \cdot \A_i^{t} \right) \otimes_{k} e_ie_i^\top \right]  \cdot \J } ,\\
    &= \sum_{i} \sum_{t=0}^{\infty} \begin{bmatrix} \beta & \alpha \\ \beta & \alpha \end{bmatrix} \left( (\A_i^{t})^{\top}  \cdot \coff_i \cdot \A_i^{t} \right) \begin{bmatrix} \beta & \beta \\ \alpha & \alpha \end{bmatrix}  \otimes_{k} \Ex{\left(H-aa^\top\right) e_ie_i^\top \left(H-aa^\top\right)}. 
\end{align*}
Hence, 
\begin{align} \label{eq:bias-eigen-upperbound}
\opM^{\top} \mycirc  \left( \Id - \opTil^{\top} \right)^{-1} \mycirc \Upsilon &= \sum_{i} \left( \sum_{t=0}^{\infty} \begin{bmatrix} \beta & \alpha \\ \beta & \alpha \end{bmatrix} \left( (\A_i^{t})^{\top}  \cdot \coff_i \cdot \A_i^{t} \right) \begin{bmatrix} \beta & \beta \\ \alpha & \alpha \end{bmatrix} \right)  \otimes_{k} \Ex{\left(H-aa^\top\right) e_ie_i^\top \left(H-aa^\top\right)}. 
\end{align}
From the definition of $\coff_i$, 
\begin{align*}
    \lambda_i \left( \sum_{t=0}^{\infty} \begin{bmatrix} \beta & \alpha \\ \beta & \alpha \end{bmatrix}  \left( (\A_i^{t})^{\top}  \cdot \coff_i \cdot \A_i^{t} \right) \begin{bmatrix} \beta & \alpha \\ \beta & \alpha \end{bmatrix}  \right) &=\sum_{t=0}^{\infty} \begin{bmatrix} \beta \lambda_i  & \alpha \lambda_i \\ \beta \lambda_i & \alpha \lambda_i \end{bmatrix} \left( (\A_i^{t})^{\top}  \cdot \begin{bmatrix}
        1 & 1 \\ 1 & 1
    \end{bmatrix} \cdot \A_i^{t} \right) \begin{bmatrix} \beta \lambda_i & \beta \lambda_i \\ \alpha \lambda_i & \alpha \lambda_i \end{bmatrix}.
\end{align*}
With $ \bl = \beta \lambda_i, \al = \alpha \lambda_i  $, 
\begin{align*}
    \A_i = \begin{bmatrix}
        1 - \bl & 1 - \bl \\ -\al & 1 - \al
    \end{bmatrix}.
\end{align*}
Hence to compute this series we can invoke Lemma~\ref{lem:bias-infinite-series} with $\Gamma = \A_i $, 
\begin{align*}
    \sum_{t=0}^{\infty} \begin{bmatrix} \beta \lambda_i  & \alpha \lambda_i \\ \beta \lambda_i & \alpha \lambda_i \end{bmatrix} \left( (\A_i^{t})^{\top}  \cdot \begin{bmatrix}
        1 & 1 \\ 1 & 1
    \end{bmatrix} \cdot \A_i^{t} \right) \begin{bmatrix} \beta \lambda_i & \beta \lambda_i \\ \alpha \lambda_i & \alpha \lambda_i \end{bmatrix} &= \left( \frac{2a}{b(4 - (a+2b))} + \frac{a+2b}{(4 - (a+2b))} \right) \begin{bmatrix}
        1 & 1 \\
        1 & 1 
    \end{bmatrix}, \\
    \lambda_i \left( \sum_{t=0}^{\infty} \begin{bmatrix} \beta & \alpha \\ \beta & \alpha \end{bmatrix}  \left( (\A_i^{t})^{\top}  \cdot \coff_i \cdot \A_i^{t} \right) \begin{bmatrix} \beta & \alpha \\ \beta & \alpha \end{bmatrix}  \right) &= \left( \frac{2\alpha \lambda_i }{\beta \lambda_i(4 - (\alpha \lambda_i+2\beta \lambda_i))} + \frac{\alpha \lambda_i+2\beta \lambda_i}{(4 - (\alpha \lambda_i+2\beta \lambda_i))} \right) \begin{bmatrix}
        1 & 1 \\
        1 & 1 
    \end{bmatrix}, \\
    \left( \sum_{t=0}^{\infty} \begin{bmatrix} \beta & \alpha \\ \beta & \alpha \end{bmatrix}  \left( (\A_i^{t})^{\top}  \cdot \coff_i \cdot \A_i^{t} \right) \begin{bmatrix} \beta & \alpha \\ \beta & \alpha \end{bmatrix}  \right) &= \left( \frac{2\alpha }{\beta \lambda_i(4 - (\alpha \lambda_i +2\beta \lambda_i))} + \frac{\alpha +2\beta }{(4 - (\alpha \lambda_i+2\beta \lambda_i))} \right) \begin{bmatrix}
        1 & 1 \\
        1 & 1 
    \end{bmatrix} .
\end{align*}
We have \begin{align*}
    (\alpha+ 2 \beta) \lambda_i \leq (\alpha+ 2 \beta) \lambda_{\max} \leq (\alpha+ 2 \beta) R^2 \leq 1, \\
    \text{Hence, } 4 -((\alpha+ 2 \beta) \lambda_i) \geq 3 ,\\
    \frac{1}{4 -((\alpha+ 2 \beta) \lambda_i)} \leq \frac{1}{3}, \\
    \frac{2\alpha }{\beta \lambda_i(4 - (\alpha \lambda_i +2\beta \lambda_i))} + \frac{\alpha+2\beta }{(4 - (\alpha \lambda_i+2\beta \lambda_i))} \leq \frac{1}{3} \left( \frac{2 \alpha}{\beta \lambda_i} + (\alpha + 2 \beta)  \right).
\end{align*} 
As the matrix $ \begin{bmatrix}
    1 & 1 \\
    1 & 1 
\end{bmatrix} $ is PSD,  
\begin{align*}
    \left( \sum_{t=0}^{\infty} \begin{bmatrix} \beta & \alpha \\ \beta & \alpha \end{bmatrix}  \left( (\A_i^{t})^{\top}  \cdot \coff_i \cdot \A_i^{t} \right) \begin{bmatrix} \beta & \alpha \\ \beta & \alpha \end{bmatrix}  \right) &\preccurlyeq  \frac{1}{3} \left( \frac{2 \alpha}{\beta \lambda_i} + (\alpha + 2 \beta)  \right)  \begin{bmatrix}
        1 & 1 \\
        1 & 1 
    \end{bmatrix} .
\end{align*}
Using the Property~\ref{prop:operator-M-upperbound}, 
\begin{align*}
    \Ex{\left(H-aa^\top\right) e_ie_i^\top \left(H-aa^\top\right)}  \preccurlyeq \Ex{ aa^\top \cdot  e_ie_i^\top \cdot aa^\top }. 
\end{align*}
Using the above two results and the fact that for any PSD matrices $P,Q,R,S$, $ P \preccurlyeq Q $ and $ R \preccurlyeq S $ then $ P \otimes_k R \preccurlyeq Q \otimes_k S $. Hence from \eqref{eq:bias-eigen-upperbound} we can get the bound as follows
\begin{align*} 
    \opM^{\top} \mycirc  \left( \Id - \opTil^{\top} \right)^{-1} \mycirc \Upsilon &= \sum_{i} \left( \sum_{t=0}^{\infty} \begin{bmatrix} \beta & \alpha \\ \beta & \alpha \end{bmatrix} \left( (\A_i^{t})^{\top}  \cdot \coff_i \cdot \A_i^{t} \right) \begin{bmatrix} \beta & \beta \\ \alpha & \alpha \end{bmatrix} \right)  \otimes_{k} \Ex{\left(H-aa^\top\right) e_ie_i^\top \left(H-aa^\top\right)} ,\\
    &\preccurlyeq \sum_i \frac{1}{3} \left( \frac{2 \alpha}{\beta \lambda_i} + (\alpha + 2 \beta)  \right)  \begin{bmatrix}
        1 & 1 \\
        1 & 1 
    \end{bmatrix}  \otimes_k \Ex{ aa^\top \cdot  e_ie_i^\top \cdot aa^\top },  \\
    &= \begin{bmatrix}
        1 & 1 \\
        1 & 1 
    \end{bmatrix} \otimes_k \Ex{ aa^\top \cdot \sum_i \left[ \left( \frac{2 \alpha}{3\beta \lambda_i} + \frac{(\alpha + 2 \beta)}{3} \right) e_ie_i^\top \right] \cdot aa^\top } ,\\
    &= \begin{bmatrix}
        1 & 1 \\
        1 & 1 
    \end{bmatrix} \otimes_k \Ex{ aa^\top \cdot \left[ \frac{2 \alpha}{3\beta} \cov^{-1} + \frac{(\alpha + 2 \beta)}{3} \I  \right] \cdot aa^\top } ,\\
    &= \begin{bmatrix}
        1 & 1 \\
        1 & 1 
    \end{bmatrix} \otimes_k \left[ \frac{2 \alpha}{3\beta} \Ex{ \nor{a}^2_{\cov^{-1}} aa^\top } + \frac{(\alpha + 2 \beta)}{3} \Ex{ \nor{a}^2 aa^\top }  \right] .
    \end{align*} 
    Using the Assumptions~\ref{asmp:fourth-moment},~\ref{asmp:stat-condition-number} of the feature distribution, we have
    \begin{align*}
        \frac{2 \alpha}{3\beta} \Ex{ \nor{a}^2_{\cov^{-1}} aa^\top } + \frac{(\alpha + 2 \beta)}{3} \Ex{ \nor{a}^2 aa^\top } \preccurlyeq \left( \frac{2 \alpha \ki}{3\beta}  + \frac{(\alpha + 2 \beta)R^2}{3} \right) \cov .
    \end{align*}
    Using $$  ( \alpha + 2 \beta ) R^2 \leq 1 ,  \alpha \leq  \frac{\beta}{2 \ki}, $$ 
    \begin{align*}
        \frac{2 \alpha}{3\beta} \Ex{ \nor{a}^2_{\cov^{-1}} aa^\top } + \frac{(\alpha + 2 \beta)}{3} \Ex{ \nor{a}^2 aa^\top } \preccurlyeq \left( \frac{1}{3} + \frac{1}{3} \right) \cov = \frac{2}{3} \cov .
    \end{align*}
    Using this bound, kronecker product of two PSD matrices is positive completes the proof. 
    \begin{align*}
        \opM^{\top} \mycirc  \left( \Id - \opTil^{\top} \right)^{-1} \mycirc \Upsilon \preccurlyeq \begin{bmatrix}
            1 & 1 \\
            1 & 1 
        \end{bmatrix} \otimes_k \frac{2}{3} \cov = \frac{2}{3} \begin{bmatrix}
            \cov & \cov \\
            \cov & \cov 
        \end{bmatrix}. 
    \end{align*}
\end{proof}

\section{Technical Lemmas}\label{app:technical-lemmas}
\begin{property}[Eigen Decomposition of $\Gamma$] \label{pro:eigen-decomposition-A}
    For the matrix 
    \begin{align}
        \Gamma = \begin{bmatrix}
            1 - \bl & 1- \bl \\
            - \al & 1 - \al 
        \end{bmatrix}
    \end{align}
    The eigen values of $~\Gamma$ are given by 
    \begin{align}
        \rp = 1 - \frac{(\al+\bl)}{2} + \sqrt{ \left(\frac{\al+\bl}{2}\right)^2 - \al }  \qquad \rmi = 1 - \frac{(\al+\bl)}{2} - \sqrt{ \left(\frac{\al+\bl}{2}\right)^2 - \al }
    \end{align}
    The eigen decomposition $ \Gamma = \U \Eigenab \U^{-1} $ where 
    \begin{align}
        \U = \frac{1}{\diff} \begin{bmatrix}
            \frac{\rmi}{(1-\rmi)} & 1 \\ 
            1 & \frac{1-\rp}{\rp}
        \end{bmatrix} &\qquad 
        \U^{-1} =  \begin{bmatrix}
            -\al & (1-\rmi)\rp \\ 
            (1-\rmi)\rp & -\rp\rmi
        \end{bmatrix} \\
        \Eigenab &=  \begin{bmatrix}
            \rp & 0 \\ 
            0 & \rmi
        \end{bmatrix} \\
        \diff &= \rp - \rmi 
    \end{align}
    The following observations hold
    \begin{itemize}
        \item $\U$ and $\U^{-1}$ are symmetric and $\rp,\rmi$ can be complex as $\Gamma$ is not symmetric.
        \item For $ 0 < a,b < 1 $ ,  $|\rp| , |\rmi| < 1 $
    \end{itemize}

\end{property}

\begin{lemma}\label{lem:nesterov:1D}
    For $\rp,\rmi$ given in the Property~\ref{pro:eigen-decomposition-A},  the following bound holds
    \begin{align}
        \left[ \frac{ \left(1 - \rp \right) \rp^{t} -  \left(1 - \rmi\right) \rmi^{t} } {\diff} \right]^2 \leq 1  + \frac{a}{\left( 1 - b \right)^2}.
    \end{align}
\end{lemma}

\begin{proof}
    To prove this conside the one-dimensional nesterov sequences starting form $x_0 = 1, z_0 = 0$  
        \begin{subequations} \label{alg:nesterov:1D}
            \eqals{ y_{t+1} = x_{t} -  b x_{t} \label{eq:yt:1D}, \\ 
            z_{t+1} = \zt{t} - a (t+1) x_{t}  \label{eq:zt:1D} ,\\ 
            ( t+2 )x_{t+1} = (t+1)y_{t+1} + z_{t+1} \label{eq:xt:1D} }
        \end{subequations}
We can easily check that, for $t \geq 0$, 
    \begin{align}
        \begin{bmatrix}
            (t+1) y_{t+1} \\
            z_{t+1}
        \end{bmatrix} = \Gamma^{t+1} \begin{bmatrix}
            1 \\ 0
        \end{bmatrix}.
    \end{align}
Using the eigen decomposition of $\Gamma = U \Lambda U^{-1}$ we can check that 
    \begin{align*}
        U^{-1}  \begin{bmatrix}
            1 \\ 0
        \end{bmatrix} &= \begin{bmatrix}
            -\al & (1-\rmi)\rp \\ 
            (1-\rmi)\rp & -\rp\rmi
        \end{bmatrix} \begin{bmatrix}
            1 \\ 0
        \end{bmatrix} = \begin{bmatrix}
            -\al \\ (1-\rmi)\rp
        \end{bmatrix} ,\\
        U \Lambda^{t+1} U^{-1}  \begin{bmatrix}
            1 \\ 0
        \end{bmatrix} &= \frac{1}{\diff} \begin{bmatrix}
            \frac{\rmi}{(1-\rmi)} & 1 \\ 
            1 & \frac{1-\rp}{\rp}
        \end{bmatrix} \begin{bmatrix}
            -\al \rp^{t+1} \\ (1-\rmi)\rp \rmi^{t+1}
        \end{bmatrix} = \frac{1}{\diff} \begin{bmatrix}
            - \rp \rmi \left[ (1 - \rp) \rp^{t} - (1-\rmi) \rmi^{t} \right] \\ -a \left[ \rp^{t+1} - \rmi^{t+1} \right]
        \end{bmatrix}, \\
        \Gamma^{t+1} \begin{bmatrix}
            1 \\ 0
        \end{bmatrix} &= \frac{1}{\diff}\begin{bmatrix}
            - \rp \rmi \left[ (1 - \rp) \rp^{t} - (1-\rmi) \rmi^{t} \right] \\ -a \left[ \rp^{t+1} - \rmi^{t+1} \right].
        \end{bmatrix}
    \end{align*}
    Hence,    
    \begin{align}
        \begin{bmatrix}
            (t+1) y_{t+1} \\
            z_{t+1}
        \end{bmatrix} = \frac{1}{\diff}\begin{bmatrix}
            - \rp \rmi \left[ (1 - \rp) \rp^{t} - (1-\rmi) \rmi^{t} \right] \\ -a \left[ \rp^{t+1} - \rmi^{t+1} \right].
        \end{bmatrix}
    \end{align}

    Now use the potential function defined by $ V_t = t^2 y_{t}^2 + \frac{1}{a} z_{t}^2$, for $t \geq 1$.  If  $ 0 < a \leq b < 1$ then for $t \geq 1$ we can show that $ V_{t+1} \leq V_{t}$. Hence, for any $t \geq 1$, $ V_{t} \leq V_{1} $. Note that $V_1 < V_0$ does not hold due to different initialization. From this, 
    \begin{align}
       (t+1)^2 y_{t+1}^2 \leq (t+1)^2 y_{t+1}^{2} + \frac{1}{a} z_{t+1}^2 \leq V_{1}, \\
       V_1 =  (1 - b)^2 + \frac{1}{a} a^2 =  ( 1 - b)^2 + a .  
    \end{align}
    Using the expression of $ (t+1)^2 y_{t+1}^2 $, we get 

    \begin{align*}
        \left[ \frac{ \left(1 - \rp \right) \rp^{t} -  \left(1 - \rmi\right) \rmi^{t} } {\diff} \right]^2 &\leq \frac{1}{(\rp\rmi)^2} \left[ ( 1 - b)^2 + a \right], \\
        &= 1 + \frac{a}{\left(1-b \right)^2}.
    \end{align*}
This proves the lemma. 
\end{proof}

\begin{lemma}
    For $ 0 < a \leq b < 1$, for $\Gamma, \aleph$ in Lemma~\ref{lem:A-infinite-series}, 
    \begin{align}
        \scal{\begin{bmatrix}
            0 & 0 \\ 0 & \displaystyle \frac{1}{a^2}
        \end{bmatrix}}{ \Gamma^{T} \aleph \left(\Gamma^{T}\right)^{\top}} \leq 1 + \frac{a}{\left(1-b \right)^2}. 
    \end{align}
\end{lemma}
\begin{proof}
    In the following Lemma~\ref{lem:A-infinite-series}, we compute the closed form for $\Gamma^{T} \aleph \left(\Gamma^{T}\right)^{\top}$ = $ \begin{bmatrix}
        \nu_{11}(t) & \nu_{12}(t) \\ \nu_{21}(t) & \nu_{22}(t)
    \end{bmatrix} $. Using this 
    \begin{align*}
        \scal{\begin{bmatrix}
            0 & 0 \\ 0 & \displaystyle \frac{1}{a^2}
        \end{bmatrix}}{ \Gamma^{T} \aleph \left(\Gamma^{T}\right)^{\top}} = \frac{\nu_{22}(t)}{a^2}.
    \end{align*}
    From \eqref{lem:inv:v22}, 
    \begin{align}
        \scal{\begin{bmatrix}
            0 & 0 \\ 0 & \displaystyle \frac{1}{a^2}
        \end{bmatrix}}{ \Gamma^{T} \aleph \left(\Gamma^{T}\right)^{\top}} =         \left[ \frac{ \left(1 - \rp \right) \rp^{t} -  \left(1 - \rmi\right) \rmi^{t} } {\diff} \right]^2. 
    \end{align}
    From Lemma~\ref{lem:nesterov:1D}, the lemma holds. 
\end{proof}

\begin{lemma} \label{lem:A-infinite-series}
For $ 0 < a,b < 1 $, with $\Gamma$ and $\aleph$ of form  
    \begin{align*}
        \Gamma = \begin{bmatrix}
            1 - \bl & 1- \bl \\
            - \al & 1 - \al 
        \end{bmatrix}, \qquad 
        \aleph = \begin{bmatrix}
            \bl^2  & \bl \al \\
            \bl \al & \al^2  
        \end{bmatrix}.
    \end{align*}
    The series
    \begin{align}
        \sum_{t=0}^{\infty} \Gamma^{t} \aleph \left(\Gamma^{t}\right)^{\top} &= \frac{1}{b (4 - (a+2b))} \begin{bmatrix}
            2a + b(2b-3a) &  a(2b-a)   \\
            a(2b-a) &  2 a^2  
       \end{bmatrix}. 
    \end{align}

\end{lemma}

\begin{proof}
    To calculate the exponents of $\Gamma$ we use the eigen decomposition in Property~\ref{pro:eigen-decomposition-A}, 
    \begin{align*}
        \Gamma &=  \U \Eigenab \U^{-1}, \\
        \Gamma^{t} &=  \U \Eigenab^{t} \U^{-1}, \\
        \Gamma^{t} \cdot \aleph \cdot \left(\Gamma^{t}\right)^{\top} & = \U \Eigenab^{t} \U^{-1} \cdot  \aleph \cdot  \left(\U \Eigenab^{t} \U^{-1}\right)^{\top}, 
    \end{align*}
    From Property~\ref{pro:eigen-decomposition-A} that $\U,\U^{-1}$ are symmetric.
    \begin{align*}
        \Gamma^{t} \cdot \aleph \cdot \left(\Gamma^{t}\right)^{\top} & = \U \Eigenab^{t} \left[ \U^{-1}   \aleph  \U^{-1} \right] \Eigenab^{t} \U .
    \end{align*}
    \paragraph{Computing ${U^{-1}\aleph\U^{-1}}$:} 
    From Property~\ref{pro:eigen-decomposition-A}, 
    \begin{align*}
        \U &= \frac{1}{\diff} \begin{bmatrix}
            \frac{\rmi}{(1-\rmi)} & 1 \\ 
            1 & \frac{1-\rp}{\rp}
        \end{bmatrix}, \qquad 
        \U^{-1} =  \begin{bmatrix}
            -\al & (1-\rmi)\rp \\ 
            (1-\rmi)\rp & -\rp\rmi
        \end{bmatrix}, \\
        \U^{-1} \aleph \U^{-1} &= \U^{-1} \begin{bmatrix}
            \bl^2 & \al\bl \\
            \al\bl & \al^2 
        \end{bmatrix} \U^{-1} =  \U^{-1}  \left( \begin{bmatrix}
            \bl \\ \al
        \end{bmatrix} \right) \otimes \left(  \begin{bmatrix}
            \bl \\ \al
        \end{bmatrix} \right)  \U^{-1}, \\ &= \left(  \U^{-1} \begin{bmatrix}
            \bl \\ \al
        \end{bmatrix} \right) \otimes \left(  \U^{-1} \begin{bmatrix}
            \bl \\ \al
        \end{bmatrix} \right) , \\
        \U \Eigenab^{t} \left[ \U^{-1}   \aleph  \U^{-1} \right] \Eigenab^{t} \U &= \left( \U \Eigenab^{t} \U^{-1} \begin{bmatrix}
            \bl \\ \al
        \end{bmatrix} \right) \otimes \left(  \U \Eigenab^{t} \U^{-1} \begin{bmatrix}
            \bl \\ \al
        \end{bmatrix} \right) , \\
        \U^{-1} \begin{bmatrix}
            \bl \\ \al
        \end{bmatrix} &=  \begin{bmatrix}
            -\al & (1-\rmi)\rp \\ 
            (1-\rmi)\rp & -\rp\rmi
        \end{bmatrix} \begin{bmatrix}
            \bl \\ \al
        \end{bmatrix} = \begin{bmatrix}
            - \al \bl + (1-\rmi)\rp \al \\ (1-\rmi)\rp \bl - \rp \rmi \al   
        \end{bmatrix}, \\
        \intertext{using $\al = \left( 1 - \rp \right) \left( 1 - \rmi \right) , \bl = \left(   1 - \rp \rmi \right) $ ,} 
        &= \begin{bmatrix}
            - \al (1 - \rp \rmi ) + (1-\rmi)\rp \al \\ (1-\rmi)\rp (1 - \rp \rmi ) - \rp \rmi \left( 1 - \rp \right) \left( 1 - \rmi \right)
        \end{bmatrix}, \\ &=  \begin{bmatrix}
            \al \left( - \left(1 - \rp \rmi \right) + \left(1-\rmi\right) \rp \right) \\ \rp (1-\rmi) \left( 1 - \rmi \rp  - \rmi \left( 1 - \rp \right) \right)  
        \end{bmatrix} = \begin{bmatrix}
            \al \left( \rp - 1 \right) \\ \rp (1-\rmi)^2   
        \end{bmatrix}. 
    \end{align*}
    \begin{align*}
        \Eigenab^{t} \U^{-1} \begin{bmatrix}
            \bl \\ \al
        \end{bmatrix}   &= \begin{bmatrix}
            \rp^{t} & 0 \\ 
            0 & \rmi^{t}
        \end{bmatrix} \begin{bmatrix}
            \al \left( \rp - 1 \right) \\ \rp (1-\rmi)^2   
        \end{bmatrix} = \begin{bmatrix}
            \al  \rp^{t} \left( \rp - 1 \right) \\ \rp  \rmi^{t} (1-\rmi)^2   
        \end{bmatrix}. 
    \end{align*}

\begin{align*}
        U \Eigenab^{t} \U^{-1} \begin{bmatrix}
            \bl \\ \al 
        \end{bmatrix} &= \frac{1}{\diff} \begin{bmatrix}
            \frac{\rmi}{(1-\rmi)} & 1 \\ 
            1 & \frac{1-\rp}{\rp}
        \end{bmatrix}  \begin{bmatrix}
            \al  \rp^{t} \left( \rp - 1 \right) \\ \rp  \rmi^{t} (1-\rmi)^2   
        \end{bmatrix} , \\
        &= \frac{1}{\diff}  \begin{bmatrix}
            \frac{\rmi}{(1-\rmi)} \al  \rp^{t} \left( \rp - 1 \right) +  \rp  \rmi^{t} (1-\rmi)^2   \\
            \al  \rp^{t} \left( \rp - 1 \right) +  \left(1-\rp \right)  \rmi^{t} (1-\rmi)^2  
        \end{bmatrix}. \\
        \intertext{Using $\al = \left( 1 - \rp \right) \left( 1 - \rmi \right) , \bl = \left(   1 - \rp \rmi \right) ,$ } 
        &= \frac{1}{\diff}  \begin{bmatrix}
             - \rmi \rp^{t} \left( \rp - 1 \right)^2 +  \rp \rmi^{t} (1-\rmi)^2   \\
            \al  \rp^{t} \left( \rp - 1 \right) + a \rmi^{t} (1-\rmi)  
        \end{bmatrix}, \\
        &= \frac{1}{\diff}  \begin{bmatrix}
            - \left[ \rmi \left( \rp - 1 \right)^2  \rp^{t}  - \rp \left(1-\rmi\right)^2   \rmi^{t} \right]  \\
           - \al \left[ \left(1 - \rp \right) \rp^{t} -  \left(1 - \rmi\right) \rmi^{t}  \right]
       \end{bmatrix}.
    \end{align*}
    \begin{align*}
        \Gamma^{t} \cdot \aleph \cdot \left(\Gamma^{t}\right)^{\top} &= \left(  U \Eigenab^{t} \U^{-1} \begin{bmatrix}
            \bl \\ \al 
        \end{bmatrix} \right) \otimes \left(  U \Eigenab^{t} \U^{-1} \begin{bmatrix}
            \bl \\ \al 
        \end{bmatrix} \right), \\
        &= \frac{1}{\diff}  \begin{bmatrix}
            - \left[ \rmi \left( 1 - \rp \right)^2  \rp^{t}  - \rp \left(1-\rmi\right)^2   \rmi^{t} \right]  \\
           - \al \left[ \left(1 - \rp \right) \rp^{t} -  \left(1 - \rmi\right) \rmi^{t}  \right]
       \end{bmatrix} \otimes \frac{1}{\diff}  \begin{bmatrix}
        - \left[ \rmi \left( \rp - 1 \right)^2  \rp^{t}  - \rp \left(1-\rmi\right)^2   \rmi^{t} \right]  \\
       - \al \left[ \left(1 - \rp \right) \rp^{t} -  \left(1 - \rmi\right) \rmi^{t}  \right]
   \end{bmatrix} ,\\
   &= \begin{bmatrix}
       \nu_{11}(t) & \nu_{12}(t) \\ \nu_{12}(t) & \nu_{22}(t)
   \end{bmatrix}.
    \end{align*}
where 
\begin{align}
    \diff^2 ~ \nu_{11}(t) &\defeq \left[ \rmi \left( 1 - \rp \right)^2  \rp^{t}  - \rp \left(1-\rmi\right)^2   \rmi^{t} \right]^2 \label{lem:inv:v11}, \\
    \diff^2 ~ \nu_{22}(t) &\defeq  \al^2 \left[ \left(1 - \rp \right) \rp^{t} -  \left(1 - \rmi\right) \rmi^{t}  \right]^2 \label{lem:inv:v22},  \\
    \diff^2 ~ \nu_{12}(t) &\defeq \al \left( \rmi \left( 1 - \rp \right)^2  \rp^{t}  - \rp \left(1-\rmi\right)^2   \rmi^{t} \right) \left( \left(1 - \rp \right) \rp^{t} -  \left(1 - \rmi\right) \rmi^{t}  \right) \label{lem:inv:v12} . 
\end{align}
Using these, 
\begin{align*}
   \sum\limits_{t=0}^{\infty}  \Gamma^{t} \cdot \aleph \cdot \left(\Gamma^{t}\right)^{\top} = \begin{bmatrix}
    \sum\limits_{t=0}^{\infty}  \nu_{11}(t) &  \sum\limits_{t=0}^{\infty}  \nu_{12}(t)  \\ \sum\limits_{t=0}^{\infty}  \nu_{12}(t) &  \sum\limits_{t=0}^{\infty}  \nu_{22}(t)
   \end{bmatrix}.  
\end{align*}
\paragraph{Evaluating  $ \sum\limits_{t=0}^{\infty}  \nu_{11}(t)$ :} 
\begin{align*}
    \diff^2 ~ \sum_{t=0}^{\infty} ~ \nu_{11}(t) &= \sum_{t=0}^{\infty} \left[ \rmi \left( 1 - \rp \right)^2  \rp^{t}  - \rp \left(1-\rmi\right)^2   \rmi^{t} \right]^2, \\
    &=  \left[ \rmi \left( 1 - \rp \right)^2  - \rp \left(1-\rmi\right)^2 \right]^2 \\ &\hspace{1cm}+ \sum_{t=1}^{\infty} \left[ \rmi^2 \left( 1 - \rp \right)^4  \rp^{2t}  + \rp^2 \left(1-\rmi\right)^4   \rmi^{2t} - 2  \rmi \left( 1 - \rp \right)^2 \rp \left(1-\rmi\right)^2  \rp^{t} \rmi^{t} \right]. 
\end{align*}
From Property~\ref{pro:eigen-decomposition-A}, when  $ 0 < a,b < 1$ then  $ |\rp| , |\rmi| < 1 $. Hence, the following holds, 
\begin{align*}
    \sum_{t=1}^{\infty} \rp^{2t} &= \frac{\rp^2}{1 - \rp^2}, \quad \sum_{t=1}^{\infty} \rmi^{2t} = \frac{\rmi^2}{1 - \rmi^2}, \quad \sum_{t=1}^{\infty} \rp^{t}\rmi^{t} = \frac{\rp\rmi}{1 - \rp\rmi}, \\
    \rmi \left( 1 - \rp \right)^2  - \rp \left(1-\rmi\right)^2 &= \left(1 - (1 - \rmi) \right) \left( 1 - \rp \right)^2  - \left(1 - (1 - \rp) \right) \left(1-\rmi\right)^2 \\&= \left( 1 - \rp \right)^2 - \left(1-\rmi\right)^2 - (1 - \rmi) \left( 1 - \rp \right)^2 + (1 - \rp) \left(1-\rmi\right)^2, \\
    &= \left(\rmi - \rp\right) \left(2 - \rmi - \rp\right) - (1 - \rmi) \left( 1 - \rp \right) \left(\rmi - \rp\right) ,\\
    &= \left(\rmi - \rp\right) \left[ \left(2 - \rmi - \rp\right) - (1 - \rmi) \left( 1 - \rp \right) \right] ,\\
    &=  \left(\rmi - \rp\right) \left[ 1 - \rp \rmi \right] = - \diff \bl.
\end{align*}
From here, the first term of the sumation is as follows,
\begin{align} \label{lem:inv:delta-b}
    \rmi \left( 1 - \rp \right)^2  - \rp \left(1-\rmi\right)^2  &= - \diff \bl.
\end{align}
To calculate the sum of the remaining terms,
\begin{align*}
    \diff^2 ~ \left(  \sum_{t=0}^{\infty} \nu_{11}(t) - b^2 \right) &= \sum_{t=1}^{\infty} \left[ \rmi^2 \left( 1 - \rp \right)^4  \rp^{2t}  + \rp^2 \left(1-\rmi\right)^4   \rmi^{2t} - 2  \rmi \left( 1 - \rp \right)^2 \rp \left(1-\rmi\right)^2  \rp^{t} \rmi^{t} \right], \\
    &= \rp^2\rmi^2 \sum_{t=0}^{\infty} \left[  \left( 1 - \rp \right)^4  \rp^{2t}  + \left(1-\rmi\right)^4   \rmi^{2t} - 2  \rmi \left( 1 - \rp \right)^2 \rp \left(1-\rmi\right)^2  \rp^{t} \rmi^{t} \right], \\
    &= \rp^2\rmi^2 \sum_{t=0}^{\infty} \left[ \rmi^2 \left( 1 - \rp \right)^2  \rp^{t}  - \left(1-\rmi\right)^2   \rmi^{t} \right]^2.
\end{align*} Invoking Lemma~\ref{lem:series-rplus-rminus}, \begin{align*}
    \diff^2 ~ \left(  \sum_{t=0}^{\infty} \nu_{11}(t) - b^2 \right) &=  \diff^2  \rp^2\rmi^2  \left[ \frac{  a (4 - (a+2b)) + (a+2b)^2 }{ 2 b (4 - (a+2b))  } \right], \\
    \left(  \sum_{t=0}^{\infty} \nu_{11}(t) - b^2 \right) &= (1-b)^2 \left[ \frac{  a (4 - (a+2b)) + (a+2b)^2 }{ 2 b (4 - (a+2b))  } \right] .
\end{align*}
Using simple algebraic manipulations summation of $\nu_{11}(t)$'s can be compactly written as follows,
\begin{align*}
    \sum_{t=0}^{\infty} \nu_{11}(t) &= b^2 + (1-b)^2 \left[ \frac{  a (4 - (a+2b)) + (a+2b)^2 }{ 2 b (4 - (a+2b))  } \right] ,\\
    &= b^2 + (1-b)^2 \left[ \frac{  4a + (a+2b) \left((a+2b) - a \right) }{ 2 b (4 - (a+2b))  }  \right] ,\\
    &= b^2 + (1-b)^2 \left[ \frac{  4a + 2b(a+2b) }{ 2b(4 - (a+2b))  }  \right] ,\\
    &= b^2  + \frac{ (a+2b) (1-b)^2  }{(4 - (a+2b))} +\left[ \frac{  4a (1-b)^2  }{ 2b(4 - (a+2b))  }    \right], \\
    &= \frac{ b^2 (4 - (a+2b)) + (a+2b) (1-b)^2  }{(4 - (a+2b))} +\left[ \frac{  4a (1-b)^2  }{ 2b(4 - (a+2b))  }    \right]. \\
    b^2 (4 - (a+2b)) + (a+2b) (1-b)^2  &= b^2 (4 - (a+2b)) + (a+2b) \left( 1 - 2b + b^2 \right) \\
    &= 4 b^2 + (a+2b) \left( 1 - 2b  \right) = a + 2b - 2ab,  \\
    \sum_{t=0}^{\infty} \nu_{11}(t) &= \frac{  a + 2b - 2ab  }{(4 - (a+2b))} +\left[ \frac{  2 a (1-b)^2  }{ b(4 - (a+2b))  }    \right], \\
    &= \frac{  b(a + 2b - 2ab) +   2 a (1-b)^2  }{ b (4 - (a+2b))} \\
    &= \frac{  ab + 2b^2 - 2ab^2 + 2a - 4ab + 2ab^2 }{ b(4 - (a+2b))} = \frac{2a + b(2b-3a) }{b(4 - (a+2b))}. \\
\end{align*}
Hence,\begin{align}
    \label{lem:inv:v11:series}
    \sum_{t=0}^{\infty} \nu_{11}(t) &= \frac{2a}{b(4 - (a+2b))} + \frac{2b - 3a}{4 - (a+2b)}.
\end{align}

\paragraph{Evaluating  $ \sum\limits_{t=0}^{\infty}  \nu_{22}(t)$ :} From~\eqref{lem:inv:v12},
\begin{align*}
    \diff^2 ~ \nu_{22}(t) &=  \al^2 \left[ \left(1 - \rp \right) \rp^{t} -  \left(1 - \rmi\right) \rmi^{t}  \right]^2, \\
    \diff^2 ~ \sum_{t=0}^{\infty}  \ \nu_{22}(t) &= \al^2 \sum_{t=0}^{\infty} \left(1 - \rp \right)^2 \rp^{2t} + \left(1 - \rmi\right)^2 \rmi^{2t} - 2 \left(1 - \rp \right) \left(1 - \rmi\right) \rp^{t} \rmi^{t} .
\end{align*}
From Property~\ref{pro:eigen-decomposition-A}, when  $ 0 < a,b < 1$ then  $ |\rp| , |\rmi| < 1 $. Hence, the following holds, 
\begin{align*}
    \sum_{t=0}^{\infty} \rp^{2t} &= \frac{1}{1 - \rp^2}, \quad \sum_{t=0}^{\infty} \rmi^{2t} = \frac{1}{1 - \rmi^2}, \quad \sum_{t=0}^{\infty} \rp^{t}\rmi^{t} = \frac{1}{1 - \rp\rmi}, \\
    \diff^2 ~ \sum_{t=0}^{\infty}  \ \nu_{22}(t) &= \al^2 \left[  \left(1 - \rp \right)^2 \frac{1}{1 - \rp^2}  + \left(1 - \rmi\right)^2 \frac{1}{1 - \rmi^2} - 2 \frac{\left(1 - \rp \right) \left(1 - \rmi\right)}{1 - \rp\rmi} \right] .\\
    &= \al^2 \left[  \frac{1 - \rp}{1 + \rp}  +  \frac{1 - \rmi}{1 + \rmi} - 2 \frac{\left(1 - \rp \right) \left(1 - \rmi\right)}{1 - \rp\rmi} \right] .
\end{align*}

\begin{align} \label{lem:inv:v22-p1}
    \diff^2 ~ \sum_{t=0}^{\infty}\nu_{22}(t) &= \al^2 \left[  \frac{1 - \rp}{1 + \rp}  +  \frac{1 - \rmi}{1 + \rmi} - 2 \frac{\left(1 - \rp \right) \left(1 - \rmi\right)}{1 - \rp\rmi} \right] .
\end{align}
Considering the computation in the right part, 
\begin{align*}
    \frac{1 - \rp}{1 + \rp}  +  \frac{1 - \rmi}{1 + \rmi} &= \frac{ \left(1 - \rp \right) \left(1 + \rmi \right) + \left(1 - \rmi \right) \left(1 + \rp \right) }{\left(1 + \rp \right) \left(1 + \rmi \right)} = \frac{2\left( 1 - \rp \rmi \right)}{\left(1 + \rp \right) \left(1 + \rmi \right)}, \\
    \frac{1 - \rp}{1 + \rp}  +  \frac{1 - \rmi}{1 + \rmi} - 2 \frac{\left(1 - \rp \right) \left(1 - \rmi\right)}{1 - \rp\rmi} &= \frac{2\left( 1 - \rp \rmi \right)}{\left(1 + \rp \right) \left(1 + \rmi \right)} - 2 \frac{\left(1 - \rp \right) \left(1 - \rmi\right)}{1 - \rp\rmi} \\
    &= 2 \frac{ \left( 1 - \rp \rmi \right)^2 - \left(1 - \rp^2 \right) \left(1 - \rmi^2 \right) }{ \left(1 + \rp \right) \left(1 + \rmi \right) \left( 1 - \rp \rmi \right) }.
\end{align*}
Computing the numerator, we get the following, 
\begin{align*}
    \left( 1 - \rp \rmi \right)^2 - \left(1 - \rp^2 \right) \left(1 - \rmi^2 \right) &=  1 - 2 \rp \rmi  + \rp^2 \rmi^2 - \left( 1 - \rp^2 - \rmi^2 + \rp^2 \rmi^2 \right),  \\
    &= \rp^2 + \rmi^2 - 2 \rp \rmi = \diff^2. \\
\end{align*}
The denominator from~\eqref{lem:inv:denominator},
\begin{align*}
    \left( 1 + \rmi \right) \left( 1 + \rp \right)  \left( 1 - \rp \rmi \right) &= b (4 - (a+2b)) . 
\end{align*}
Substituting these back in~\eqref{lem:inv:v22-p1}, we get 
\begin{align*}
    \diff^2 ~ \sum_{t=0}^{\infty} \nu_{22}(t) &= \frac{2 a^2 \diff^2 }{b (4 - (a+2b))  }.
\end{align*}
Hence,
\begin{align}
    \label{lem:inv:v22:series}
    \sum_{t=0}^{\infty} \nu_{22}(t) &= \frac{2 a^2}{b (4 - (a+2b))  }.
\end{align}

\paragraph{Evaluating  $ \sum\limits_{t=0}^{\infty}  \nu_{12}(t)$ :} From~\eqref{lem:inv:v12}, 
\begin{align*}
    \diff^2 ~ \nu_{12}(t) &= \al \left( \rmi \left( 1 - \rp \right)^2  \rp^{t}  - \rp \left(1-\rmi\right)^2   \rmi^{t} \right) \left( \left(1 - \rp \right) \rp^{t} -  \left(1 - \rmi\right) \rmi^{t}  \right), \\
    &=  a \left[  \rmi \left( 1 - \rp \right)^3  \rp^{2t} + \rp \left(1-\rmi\right)^3   \rmi^{2t}  - \left\lbrace \rmi \left( 1 - \rp \right) + \rp \left(1-\rmi\right) \right\rbrace \left( 1 - \rp \right) \left(1-\rmi\right) \rp^{t} \rmi^{t} \right], \\
    \diff^2 ~ \sum_{t=0}^{\infty} \nu_{12}(t) &= a  \sum_{t=0}^{\infty} \left[  \rmi \left( 1 - \rp \right)^3  \rp^{2t} + \rp \left(1-\rmi\right)^3   \rmi^{2t}  - \left\lbrace \rmi \left( 1 - \rp \right) + \rp \left(1-\rmi\right) \right\rbrace \left( 1 - \rp \right) \left(1-\rmi\right) \rp^{t} \rmi^{t} \right].
\end{align*}

For $ 0 < a,b < 1 $ , we have $ |\rp| , |\rmi| < 1 $ from Property~\ref{pro:eigen-decomposition-A}. Hence, the following holds, 
\begin{align*}
    \sum_{t=0}^{\infty} \rp^{2t}  &= \frac{1}{1 - \rp^2}, \quad \sum_{t=0}^{\infty} \rmi^{2t} = \frac{1}{1 - \rmi^2}, \quad \sum_{t=0}^{\infty} \rp^{t}\rmi^{t} = \frac{1}{1 - \rp\rmi}. \\
    &= a \left[  \frac{\rmi \left( 1 - \rp \right)^3}{1 - \rp^2} + \frac{\rp \left(1-\rmi\right)^3}{1 - \rmi^2}  - \left\lbrace \rmi \left( 1 - \rp \right) + \rp \left(1-\rmi\right) \right\rbrace \frac{\left( 1 - \rp \right) \left(1-\rmi\right)}{1 - \rp\rmi} \right]. \\
    &= a \left[  \frac{\rmi \left( 1 - \rp \right)^2}{1 + \rp} + \frac{\rp \left(1-\rmi\right)^2}{1 + \rmi}  - \left\lbrace \rmi \left( 1 - \rp \right) + \rp \left(1-\rmi\right) \right\rbrace \frac{\left( 1 - \rp \right) \left(1-\rmi\right)}{1 - \rp\rmi} \right].
\end{align*}
Using 
\begin{align*}
    2 \left( 1 - \rp\rmi \right) &= \left( 1 - \rmi\right) \left( 1 + \rp\right)  + \left( 1 - \rp\right) \left( 1 + \rmi\right), \\
\end{align*}
\begin{align} \label{lem:inv:12-p1}
    2 \rmi \left( 1 + \rmi\right) \left( 1 - \rp \right)^2 \left( 1 - \rp\rmi \right) &= \rmi \left( 1 - \rp\right) \left( 1 - \rmi^2 \right) \left( 1 - \rp^2 \right)    + \rmi \left( 1 + \rmi\right)^2 \left( 1 - \rp \right)^3.  
\end{align}
Similarly by symmetry
\begin{align}  \label{lem:inv:12-p2}
    2 \rp \left( 1 + \rp\right) \left( 1 - \rmi \right)^2 \left( 1 - \rp\rmi \right) &= \rp \left( 1 - \rmi \right) \left( 1 - \rmi^2 \right) \left( 1 - \rp^2 \right)    + \rp \left( 1 + \rp\right)^2 \left( 1 - \rmi \right)^3.  
\end{align}
\begin{align}  \label{lem:inv:12-p3}
    \begin{aligned}
        \left[ \rmi \left( 1 - \rp \right) + \rp \left(1-\rmi\right) \right]& \left( 1 - \rp^2 \right) \left(1-\rmi^2\right) ,\\ 
        &\hspace{-1.5cm}=  \rmi \left( 1 - \rp \right) \left( 1 - \rp^2 \right) \left(1-\rmi^2\right) +  \rp \left(1-\rmi\right) \left( 1 - \rp^2 \right) \left(1-\rmi^2\right).
    \end{aligned}
\end{align}
Combining them, 
\begin{align*}
    \eqref{lem:inv:12-p1} + \eqref{lem:inv:12-p2} - 2*\eqref{lem:inv:12-p3} &= \rmi \left( 1 + \rmi\right)^2 \left( 1 - \rp \right)^3 + \rp \left( 1 + \rp\right)^2 \left( 1 - \rmi \right)^3 \\ &\hspace{1cm} - \rmi \left( 1 - \rp \right) \left( 1 - \rp^2 \right) \left(1-\rmi^2\right) -  \rp \left(1-\rmi\right) \left( 1 - \rp^2 \right) \left(1-\rmi^2\right) , \\
    &=  \rmi \left( 1 + \rmi\right)^2 \left( 1 - \rp \right)^3 - \rmi \left( 1 - \rp \right) \left( 1 - \rp^2 \right) \left(1-\rmi^2\right)   ,\\
    &\hspace{1cm} +  \rp \left( 1 + \rp\right)^2 \left( 1 - \rmi \right)^3  - \rp \left(1-\rmi\right) \left( 1 - \rp^2 \right) \left(1-\rmi^2\right) ,\\
    &= \rmi \left( 1 + \rmi\right)  \left( 1 - \rp \right)^2 \left[  \left( 1 + \rmi\right)  \left( 1 - \rp \right) -  \left( 1 + \rp \right) \left(1-\rmi\right) \right]   ,\\
    &\hspace{1cm} +  \rp \left( 1 + \rp\right) \left( 1 - \rmi \right)^2 \left[ \left( 1 + \rp\right) \left( 1 - \rmi \right) -  \left( 1 - \rp \right) \left(1+\rmi \right) \right] ,\\
    &= 2 \rmi \left( 1 + \rmi\right)  \left( 1 - \rp \right)^2 \left[  \rmi - \rp \right]   + 2 \rp \left( 1 + \rp\right) \left( 1 - \rmi \right)^2 \left[ \rp - \rmi \right] ,\\
    &= - 2 \diff \left[ \rmi \left( 1 + \rmi\right)  \left( 1 - \rp \right)^2 - \rp \left( 1 + \rp\right) \left( 1 - \rmi \right)^2  \right].
\end{align*}
Evaluating
\begin{align*}
    \rmi \left( 1 + \rmi\right)  \left( 1 - \rp \right)^2 - \rp \left( 1 + \rp\right) \left( 1 - \rmi \right)^2 &= \rmi  \left( 1 - \rp \right)^2 - \rp \left( 1 - \rmi \right)^2 + \rmi^2  \left( 1 - \rp \right)^2 - \rp^2 \left( 1 - \rmi \right)^2.
\end{align*} 
From~\eqref{lem:inv:delta-b}, we have the following,
\begin{align*}
    \rmi \left( 1 - \rp \right)^2  - \rp \left(1-\rmi\right)^2  &= - \diff \bl ,\\
    \rmi^2  \left( 1 - \rp \right)^2 - \rp^2 \left( 1 - \rmi \right)^2 &= \left[ \rmi  \left( 1 - \rp \right) - \rp \left( 1 - \rmi \right) \right] \left[ \rmi  \left( 1 - \rp \right) + \rp \left( 1 - \rmi \right) \right] , \\
    &= - \diff \left[ \rp + \rmi - 2 \rp \rmi \right]  =  - \diff \left[ 2 - (a+b) - 2 (1-b) \right], \\
    &=  - \diff \left[ b - a \right]. 
\end{align*}
\begin{align*}
    \rmi \left( 1 + \rmi\right)  \left( 1 - \rp \right)^2 - \rp \left( 1 + \rp\right) \left( 1 - \rmi \right)^2 &= - \diff \bl  - \diff \left[ \bl - \al \right] = - \diff \left(2\bl - \al\right).
\end{align*}
So the numerator of $ \diff^2 \sum\limits_{t=0}^{\infty} \nu_{12}(t)$
\begin{align*}
    \frac{a}{2} \left( \eqref{lem:inv:12-p1} + \eqref{lem:inv:12-p2} - 2*\eqref{lem:inv:12-p3} \right) &= \frac{a}{2} \left( - 2 \diff (  - \diff \left(2\bl - \al\right) ) \right), \\
    &=  \diff^2 a \left(2\bl - \al\right).
\end{align*}
the denominator is 
\begin{align*}
    \left( 1 + \rmi \right) \left( 1 + \rp \right)  \left( 1 - \rp \rmi \right) &= b (4 - (a+2b)).  
\end{align*}
Finally, we have 
\begin{align} \label{lem:inv:v12:series}
    \sum\limits_{t=0}^{\infty} \nu_{12}(t) = \frac{a(2b-a)}{b(4-(a+2b))}.
\end{align}
From \eqref{lem:inv:v12:series}, \eqref{lem:inv:v22:series}, \eqref{lem:inv:v11:series}
\begin{align*}
    \sum_{t=0}^{\infty} \Gamma^{t} \aleph \left(\Gamma^{t}\right)^{\top} &= \begin{bmatrix}
        \displaystyle \frac{2a}{b(4 - (a+2b))} + \frac{2b-3a}{4 - (a+2b)} & \displaystyle \frac{a(2b-a)}{b(4-(a+2b))}   \\
        \displaystyle \frac{a(2b-a)}{b(4-(a+2b))} & \displaystyle \frac{2 a^2}{b (4 - (a+2b))  } 
    \end{bmatrix} \\
    &= \frac{1}{b (4 - (a+2b))} \begin{bmatrix}
         2a + b(2b-3a) &  a(2b-a)   \\
         a(2b-a) &  2 a^2  
    \end{bmatrix}. 
\end{align*}
This proves the lemma. 
\end{proof}

\begin{lemma} \label{lem:bias-infinite-series} For $ 0 < a,b < 1 $, with $\Gamma$ and $\aleph$ of form  \begin{align*}
        \Gamma = \begin{bmatrix}
            1 - \bl & 1- \bl \\
            - \al & 1 - \al 
        \end{bmatrix}, \qquad 
        \aleph = \begin{bmatrix}
            1  & 1 \\
            1 & 1  
        \end{bmatrix}.
    \end{align*}
    The series
    \begin{align}
        \begin{bmatrix} \bl & \al \\ \bl & \al \end{bmatrix} \sum_{t=0}^{\infty} \left(\Gamma^{t}\right)^{\top}  \aleph \Gamma^{t} \begin{bmatrix} \bl & \bl \\ \al & \al \end{bmatrix}  &= \left( \frac{2a}{b(4 - (a+2b))} + \frac{a+2b}{(4 - (a+2b))} \right) \begin{bmatrix}
            1 & 1 \\
            1 & 1 
        \end{bmatrix}.
    \end{align}

\end{lemma}

\begin{proof}
    To calculate the exponents of $\Gamma$ we use the eigendecomposition from Property~\ref{pro:eigen-decomposition-A}, 
    \begin{align*}
        \Gamma &=  \U \Eigenab \U^{-1}, \\
        \Gamma^{t} &=  \U \Eigenab^{t} \U^{-1}, \\
        \left(\Gamma^{t}\right)^{\top}  \cdot \aleph \cdot \Gamma^{t} & = \U^{-1} \Eigenab^{t} \U \cdot  \aleph \cdot  \left(\U^{-1} \Eigenab^{t} \U\right)^{\top}. 
    \end{align*}
    Using the fact in Property~\ref{pro:eigen-decomposition-A}, that $\U,\U^{-1}$ are symmetric. 
    \begin{align*}
        \begin{bmatrix} \bl & \al \\ \bl & \al \end{bmatrix} \left(\Gamma^{t}\right)^{\top}  \cdot \aleph \cdot \Gamma^{t} \begin{bmatrix} \bl & \al \\ \bl & \al \end{bmatrix}^{\top} & = \begin{bmatrix} \bl & \al \\ \bl & \al \end{bmatrix} \U^{-1} \Eigenab^{t} \left[ \U   \aleph  \U \right] \Eigenab^{t} \U^{-1} \begin{bmatrix} \bl & \al \\ \bl & \al \end{bmatrix}^{\top}.
    \end{align*}

    \begin{align*}
        \U \aleph \U &= \U \begin{bmatrix}
            1 & 1 \\
            1 & 1 
        \end{bmatrix} \U =  \U \left( \begin{bmatrix}
            1 \\ 1
        \end{bmatrix} \right) \otimes \left(  \begin{bmatrix}
            1 \\ 1
        \end{bmatrix} \right)  \U ,\\ 
        &= \left(  \U \begin{bmatrix}
            1 \\ 1
        \end{bmatrix} \right) \otimes \left(  \U \begin{bmatrix}
            1 \\ 1
        \end{bmatrix} \right) , \\
        \U^{-1} \Eigenab^{t} \left[ \U   \aleph  \U \right] \Eigenab^{t} \U^{-1} &= \left( \U^{-1} \Eigenab^{t} \U \begin{bmatrix}
            1 \\ 1
        \end{bmatrix} \right) \otimes \left( \U^{-1} \Eigenab^{t} \U \begin{bmatrix}
            1 \\ 1
        \end{bmatrix} \right), \\
        \begin{bmatrix} \bl & \al \\ \bl & \al \end{bmatrix} \left(\Gamma^{t}\right)^{\top}  \cdot \aleph \cdot \Gamma^{t} \begin{bmatrix} \bl & \al \\ \bl & \al \end{bmatrix}^{\top} &=   \left( \begin{bmatrix} \bl & \al \\ \bl & \al \end{bmatrix} \U^{-1} \Eigenab^{t} \U \begin{bmatrix}
            1 \\ 1
        \end{bmatrix} \right) \otimes \left( \begin{bmatrix} \bl & \al \\ \bl & \al \end{bmatrix} \U^{-1} \Eigenab^{t} \U \begin{bmatrix}
            1 \\ 1
        \end{bmatrix} \right).
    \end{align*}
    From eigendecomposition given in Property~\ref{pro:eigen-decomposition-A}, 
    \begin{align*}
        \diff \U &=  \begin{bmatrix}
            \frac{\rmi}{(1-\rmi)} & 1 \\ 
            1 & \frac{1-\rp}{\rp}
        \end{bmatrix}, \\
        \diff \U \begin{bmatrix}
            1 \\ 1
        \end{bmatrix} &=  \begin{bmatrix}
            \frac{\rmi}{(1-\rmi)} & 1 \\ 
            1 & \frac{1-\rp}{\rp}
        \end{bmatrix} \begin{bmatrix}
            \bl \\ \al
        \end{bmatrix} = \begin{bmatrix}
            \displaystyle \frac{1}{(1-\rmi)} \\ \displaystyle \frac{1}{\rp}
        \end{bmatrix} ,\\
        \diff \Eigenab^{t} \U \begin{bmatrix}
            1 \\ 1
        \end{bmatrix}   &= \begin{bmatrix}
            \rp^{t} & 0 \\ 
            0 & \rmi^{t}
        \end{bmatrix} \begin{bmatrix}
            \displaystyle \frac{1}{(1-\rmi)} \\ \displaystyle \frac{1}{\rp}
        \end{bmatrix}  = \begin{bmatrix}
            \displaystyle \frac{\rp^{t}}{(1-\rmi)} \\ \displaystyle \frac{\rmi^{t}}{\rp}
        \end{bmatrix} .
    \end{align*}
    Again from Property \ref{pro:eigen-decomposition-A} using $\U^{-1}$
    \begin{align*}
        \U^{-1} &=  \begin{bmatrix}
            -\al & (1-\rmi)\rp \\ 
            (1-\rmi)\rp & -\rp\rmi
        \end{bmatrix},\\
        \diff \U^{-1} \Eigenab^{t} \U \begin{bmatrix} 1 \\ 1 \end{bmatrix} &= \begin{bmatrix}
            -\al & (1-\rmi)\rp \\ 
            (1-\rmi)\rp & -\rp\rmi
        \end{bmatrix}  \begin{bmatrix}
             \frac{\rp^{t}}{(1-\rmi)} \\  \frac{\rmi^{t}}{\rp}
        \end{bmatrix} , \\
        &= \begin{bmatrix}
            \frac{-a \rp^{t}}{(1-\rmi)} + (1-\rmi) \rmi^{t} \\
            \rp^{t+1}  - \rmi^{t+1}
       \end{bmatrix}.\\
        \intertext{Using $\al = \left( 1 - \rp \right) \left( 1 - \rmi \right)$ , } 
        \diff \U^{-1} \Eigenab^{t} \U \begin{bmatrix} 1 \\ 1 \end{bmatrix} &= \begin{bmatrix}
            -(1-\rp) \rp^{t} + (1-\rmi) \rmi^{t} \\
            \rp^{t+1}  - \rmi^{t+1}
       \end{bmatrix} = \begin{bmatrix}
        -\left[ (1-\rp) \rp^{t} - (1-\rmi) \rmi^{t} \right] \\
        \rp^{t+1}  - \rmi^{t+1}
   \end{bmatrix} . 
\end{align*}
\begin{align*}
   \diff \begin{bmatrix} \bl & \al \\ \bl & \al \end{bmatrix} \U^{-1} \Eigenab^{t} \U \begin{bmatrix} 1 \\ 1 \end{bmatrix}   &=  \begin{bmatrix} \bl & \al \\ \bl & \al \end{bmatrix} \begin{bmatrix}
    -\left[ (1-\rp) \rp^{t} - (1-\rmi) \rmi^{t} \right] \\
    \rp^{t+1}  - \rmi^{t+1}
\end{bmatrix},
\\& = \begin{bmatrix}
    - \bl \left( (1-\rp) \rp^{t} - (1-\rmi) \rmi^{t} \right) +  \al \left( \rp^{t+1}  - \rmi^{t+1} \right) \\
    - \bl \left( (1-\rp) \rp^{t} - (1-\rmi) \rmi^{t} \right) +  \al \left( \rp^{t+1}  - \rmi^{t+1} \right) 
\end{bmatrix}, \\ &= \begin{bmatrix}
    - \left[ \left( \bl (1-\rp) - \al (\rp)  \right) \rp^{t} - \left( \bl (1-\rmi) - \al (\rmi) \right) \rmi^{t} \right] \\
    - \left[ \left( \bl (1-\rp) - \al (\rp)  \right) \rp^{t} - \left( \bl (1-\rmi) - \al (\rmi) \right) \rmi^{t} \right].
\end{bmatrix}
    \end{align*} 
    Using $\al = \left( 1 - \rp \right) \left( 1 - \rmi \right),  \bl  = 1 - \rmi \rp $, 
    \begin{align*}
        \bl (1-\rp) - \al (\rp) &= ( 1 - \rp \rmi ) (1-\rp) - \left( 1 - \rp \right) \left( 1 - \rmi \right) \rp, \\
        &=  \left( 1 - \rp \right) \left( ( 1 - \rp \rmi ) - \rp \left( 1 - \rmi \right) \right), \\
        &= \left( 1 - \rp \right) \left( 1 - \rp \rmi - \rp + \rp \rmi  \right) = \left(1-\rp\right)^2.
    \end{align*} 
    By symmetry, 
    \begin{align*}
        \bl (1-\rmi) - \al (\rmi) &= \left(1-\rmi\right)^2.
    \end{align*}
    Substituting this back we get 
    \begin{align*}
        \diff \begin{bmatrix} \bl & \al \\ \bl & \al \end{bmatrix} \U^{-1} \Eigenab^{t} \U \begin{bmatrix} 1 \\ 1 \end{bmatrix}   &= \begin{bmatrix}
            - \left( \left( 1 - \rp  \right)^2 \rp^{t} -  \left( 1 - \rmi  \right)^2 \rmi^{t} \right) \\
            - \left( \left( 1 - \rp  \right)^2 \rp^{t} -  \left( 1 - \rmi  \right)^2 \rmi^{t} \right) 
        \end{bmatrix}
        &= - \left( \left( 1 - \rp  \right)^2 \rp^{t} -  \left( 1 - \rmi  \right)^2 \rmi^{t} \right) \begin{bmatrix}
            1 \\ 1
        \end{bmatrix}.
    \end{align*} 
    
    \begin{align*}
        \diff \begin{bmatrix} \bl & \al \\ \bl & \al \end{bmatrix} \U^{-1} \Eigenab^{t} \U \begin{bmatrix} 1 \\ 1 \end{bmatrix} \otimes \diff \begin{bmatrix} \bl & \al \\ \bl & \al \end{bmatrix} \U^{-1} \Eigenab^{t} \U \begin{bmatrix} 1 \\ 1 \end{bmatrix} &=  \left( \left( 1 - \rp  \right)^2 \rp^{t} -  \left( 1 - \rmi  \right)^2 \rmi^{t} \right)^2 \begin{bmatrix}
            1 & 1 \\
            1 & 1 
        \end{bmatrix} ,\\
        \diff^2 \begin{bmatrix} \bl & \al \\ \bl & \al \end{bmatrix} \left(\Gamma^{t}\right)^{\top}  \cdot \aleph \cdot \Gamma^{t} \begin{bmatrix} \bl & \al \\ \bl & \al \end{bmatrix}^{\top} &=  \left( \left( 1 - \rp  \right)^2 \rp^{t} -  \left( 1 - \rmi  \right)^2 \rmi^{t} \right)^2 \begin{bmatrix}
            1 & 1 \\
            1 & 1 
        \end{bmatrix} . \\
         \begin{bmatrix} \bl & \al \\ \bl & \al \end{bmatrix} \sum_{t=0}^{\infty} \left(\Gamma^{t}\right)^{\top}  \cdot \aleph \cdot \Gamma^{t} \begin{bmatrix} \bl & \al \\ \bl & \al \end{bmatrix}^{\top} &=  \left[ \sum_{t=0}^{\infty} \left( \frac{ \left( 1 - \rp  \right)^2 \rp^{t} -  \left( 1 - \rmi  \right)^2 \rmi^{t} } { \diff } \right)^2 \right] \begin{bmatrix}
            1 & 1 \\
            1 & 1 
        \end{bmatrix}.
    \end{align*}
    Using Lemma~\ref{lem:series-rplus-rminus}, 
    \begin{align*}
        \begin{bmatrix} \bl & \al \\ \bl & \al \end{bmatrix} \sum_{t=0}^{\infty} \left(\Gamma^{t}\right)^{\top}  \cdot \aleph \cdot \Gamma^{t} \begin{bmatrix} \bl & \al \\ \bl & \al \end{bmatrix}^{\top} = \left( \frac{2a}{b(4 - (a+2b))} + \frac{a+2b}{(4 - (a+2b))} \right) \begin{bmatrix}
            1 & 1 \\
            1 & 1 
        \end{bmatrix}.
    \end{align*}

\end{proof}

\begin{lemma} \label{lem:series-rplus-rminus}
    With $\rp, \rmi$ defined by Property \ref{pro:eigen-decomposition-A}, and $\nu(t)$ defined by 
    \begin{align*}
       \nu(t) \defeq \left[ \frac{ \left( 1 - \rp  \right)^2 \rp^{t} -  \left( 1 - \rmi  \right)^2 \rmi^{t} } {\diff} \right]^2. 
    \end{align*}
    the series 
    \begin{align*}
         \sum_{t=0}^{\infty} \nu(t)  =  \frac{2a}{b(4 - (a+2b))} + \frac{a+2b}{(4 - (a+2b))}.
    \end{align*}
\end{lemma}

\begin{proof}
    \begin{align*}
        \diff^2 ~ \left(  \sum_{t=0}^{\infty} \nu(t)  \right) &= \sum_{t=0}^{\infty} \left[ \left( 1 - \rp \right)^4  \rp^{2t}  +  \left(1-\rmi\right)^4   \rmi^{2t} - 2 \left( 1 - \rp \right)^2 \left(1-\rmi\right)^2  \rp^{t} \rmi^{t} \right]. 
    \end{align*}
    From Property~\ref{pro:eigen-decomposition-A} when $ 0 < a,b < 1$ then $ |\rp| , |\rmi| < 1 $ . Hence, the following holds, 
    \begin{align*}
        \sum_{t=0}^{\infty} \rp^{2t}  &= \frac{1}{1 - \rp^2}, \quad \sum_{t=0}^{\infty} \rmi^{2t} = \frac{1}{1 - \rmi^2}, \quad \sum_{t=0}^{\infty} \rp^{t}\rmi^{t} = \frac{1}{1 - \rp\rmi}. \\
        \diff^2 ~ \left(  \sum_{t=0}^{\infty} \nu(t)  \right)  &=  \left( 1 - \rp \right)^4 \frac{1}{1 - \rp^2} +  \left(1-\rmi\right)^4 \frac{1}{1 - \rmi^2} -  2 \left( 1 - \rp \right)^2 \left(1-\rmi\right)^2  \frac{1}{1 - \rp\rmi}, \\
        &= \left[ \frac{\left( 1 - \rp \right)^3}{1 + \rp} + \frac{\left(1-\rmi\right)^3}{1+ \rmi} - 2 \frac{\left( 1 - \rp \right)^2 \left(1-\rmi\right)^2}{1 - \rp \rmi } \right].
    \end{align*}
    
    \begin{align} \label{lem:series:p7}
        \diff^2 ~ \left(  \sum_{t=0}^{\infty} \nu(t)  \right) = \left[ \frac{\left( 1 - \rp \right)^3}{1 + \rp} + \frac{\left(1-\rmi\right)^3}{1+ \rmi} - 2 \frac{\left( 1 - \rp \right)^2 \left(1-\rmi\right)^2}{1 - \rp \rmi } \right].
    \end{align}
    
    \begin{align} \label{lem:series:p6}
        \begin{aligned}
        \diff^2 ~ \left(  \sum_{t=0}^{\infty} \nu(t)  \right)
        &= \frac{ \left[ \left( 1 - \rp \right)^3 \left(1+ \rmi \right)  + \left(1-\rmi\right)^3 \left(1+ \rp \right) \right] \left[ 1 - \rp \rmi  \right] }{ \left(1+ \rp \right) \left(1+ \rmi \right) \left[ 1 - \rp \rmi  \right] } \\ &\hspace{2cm}- \frac{2 \left( 1 - \rp \right)^2 \left(1-\rmi\right)^2 \left(1+ \rmi \right) \left(1+ \rp \right)  }{ \left(1+ \rp \right) \left(1+ \rmi \right) \left[ 1 - \rp \rmi  \right] }.
        \end{aligned}
    \end{align}
    Note, 
        \[ \left( 1 + \rmi \right) \left( 1 + \rp \right) = 1 + \rmi + \rp + \rmi \rp. \] 
        Using $\rmi + \rp = 2 - (a+b)$, $\rmi \rp = 1-b$,
    \begin{align} \label{lem:series:p5}
        \left( 1 + \rmi \right) \left( 1 + \rp \right) = 4 - (a + 2b).
    \end{align}
    Using  
    \begin{align*} 
         2 \left[ 1 - \rp \rmi  \right] &= \left( 1 - \rp \right) \left( 1 + \rmi \right) + \left( 1 + \rp \right) \left( 1 - \rmi \right), \\ 
        2 \left[ 1 - \rp \rmi  \right] \left( 1 - \rp \right)^3 \left(1+ \rmi \right)   &=  \left[ \left( 1 - \rp \right) \left( 1 + \rmi \right) + \left( 1 + \rp \right) \left( 1 - \rmi \right) \right] \left( 1 - \rp \right)^3 \left(1+ \rmi \right). 
    \end{align*} \begin{align} \label{lem:series:p1}
        2 \left[ 1 - \rp \rmi  \right] \left( 1 - \rp \right)^3 \left(1+ \rmi \right)  &= \left( 1 - \rp \right)^{4} \left(1+ \rmi \right)^2 + \left( 1 - \rmi^2 \right) \left( 1 - \rp^2 \right) \left( 1 - \rp \right)^2,  
    \end{align}
    Symetrically,
    \begin{align} \label{lem:series:p2}
        2 \left[ 1 - \rp \rmi  \right] \left( 1 - \rmi \right)^3 \left(1+ \rp \right)   &= 
        \left( 1 - \rmi \right)^{4} \left(1+ \rp \right)^2 + \left( 1 - \rp^2 \right) \left( 1 - \rmi^2 \right) \left( 1 - \rmi \right)^2.
    \end{align} 
    \begin{align} \label{lem:series:p3}
        \begin{aligned}
            4 \left( 1 - \rp \right)^2 \left(1-\rmi\right)^2 \left(1+ \rmi \right) \left(1+ \rp \right) &= 2 \left(1+ \rmi \right)  \left( 1 - \rp \right)^2  \left(1+ \rp \right) \left(1-\rmi\right)^2  \\ & \hspace{1 cm} + 2  \left( 1 - \rmi^2 \right) \left( 1 - \rp^2 \right) \left( 1 - \rmi \right) \left( 1 - \rp \right).
        \end{aligned}
    \end{align}
    Combining the above calculations,
    \begin{align} \label{lem:series:p4}
     \eqref{lem:series:p1} + \eqref{lem:series:p2} - \eqref{lem:series:p3} &= \left[ \left( 1 - \rp \right)^{2} \left(1+ \rmi \right) -  \left( 1 - \rmi \right)^{2} \left(1+ \rp \right)  \right]^2 + \left( 1 - \rmi^2 \right) \left( 1 - \rp^2 \right) \left[ (1-\rp) - (1 - \rmi) \right]^2 ,
    \end{align}
    Computing the two terms,
    \begin{align*}
        \left( 1 - \rp \right)^{2} \left(1+ \rmi \right) -  \left( 1 - \rmi \right)^{2} \left(1+ \rp \right) &= \left( 1 - \rp \right)^{2} \left(2 - ( 1 - \rmi)  \right) -  \left( 1 - \rmi \right)^{2} \left(2 - ( 1 - \rp ) \right), \\
        & = 2 \left[ \left( 1 - \rp \right)^{2} - \left( 1 - \rmi \right)^{2} \right]    - \left[ \left( 1 - \rp \right)^{2}  \left( 1 - \rmi \right) - \left( 1 - \rmi \right)^{2} \left( 1 - \rp \right) \right] ,\\
        &= 2 \left( \rmi - \rp \right) \left( 2 - \rmi - \rp \right) - \left( 1 - \rmi \right) \left( 1 - \rp \right) \left( \rmi - \rp \right), \\
        &= \left( \rmi - \rp \right) \left( 4 - 2 \rmi - 2 \rp - 1 + \rmi + \rp - \rmi \rp \right) \\
        &= - \diff  \left( 4 - \left( 1 + \rmi \right) \left( 1 + \rp \right) \right) = - \diff (a +2b), \quad \textrm{ from \eqref{lem:series:p5} } , \\
    \left( 1 - \rmi^2 \right) \left( 1 - \rp^2 \right) \left[ (1-\rp) - (1 - \rmi) \right]^2 &= \diff^2 \left( 1 - \rmi \right) \left( 1 - \rp \right) \left( 1 + \rmi \right) \left( 1 + \rp \right) ,\\
    &=  \diff^2 a (4 - (a+2b)).
    \end{align*}
    The numerator of \eqref{lem:series:p6} as per \eqref{lem:series:p4} is $ \frac{1}{2} \left( \diff^2 a (4 - (a+2b)) + \diff^2 (a+2b)^2 \right) $. From \eqref{lem:series:p5} , the denominator is 
    \begin{align} \label{lem:inv:denominator}
        \left( 1 + \rmi \right) \left( 1 + \rp \right)  \left( 1 - \rp \rmi \right) &= b (4 - (a+2b)).  
    \end{align}
    Now from \eqref{lem:series:p6}, we have 
    \begin{align*} 
        \diff^2 ~ \left(  \sum_{t=0}^{\infty} \nu(t)  \right)
        &= \diff^2 \left[ \frac{  a (4 - (a+2b)) + (a+2b)^2 }{ 2 b (4 - (a+2b))  } \right].
    \end{align*}
    Hence, 
    \begin{align*}
        \left(  \sum_{t=0}^{\infty} \nu(t)  \right) &=  \left[ \frac{  a (4 - (a+2b)) + (a+2b)^2 }{ 2 b (4 - (a+2b))  } \right], \\
 &= \frac{ 4a + (a+2b) ( a + 2b - a) }{ 2 b (4 - (a+2b)) } = \frac{4a}{2 b (4 - (a+2b))} + \frac{ 2b(a+2b)}{2 b (4 - (a+2b))}, \\
&= \frac{2a}{b(4 - (a+2b))} + \frac{a+2b}{(4 - (a+2b))}.
    \end{align*}
    This completes the proof. 
\end{proof}








\end{document}